\documentclass[twoside,11pt]{article}

\usepackage{amssymb}
\usepackage{amsmath}

\newcommand{\xv}{\mathbf{x}}

\newcommand{\zv}{\mathbf{z}}
\newcommand{\Zv}{\mathbf{Z}}

\newcommand{\gv}{\mathbf{g}}

\newcommand{\Av}{\mathbf{A}}

\newcommand{\wv}{\mathbf{w}}
\newcommand{\Wv}{\mathbf{W}}

\newcommand{\ev}{\mathbf{e}}

\newcommand{\iTrain}{\mathcal{I}_{\textrm{tr}}}
\newcommand{\iTest}{\mathcal{I}_{\textrm{tst}}}

\newcommand{\uv}{\mathbf{u}}
\newcommand{\Uv}{\mathbf{U}}
\newcommand{\defEq}{\stackrel{\textrm{def}}{=}}
\newcommand{\argmax}{\operatornamewithlimits{argmax}}

\newcommand{\ud}{\mathrm{d}}

\newcommand{\sign}{\mathrm{sign}}
\newcommand{\etav}{\boldsymbol \eta}

\newcommand{\muv}{\boldsymbol \mu}
\newcommand{\nuv}{\boldsymbol \nu}
\newcommand{\psiv}{\boldsymbol \psi}
\newcommand{\phiv}{\boldsymbol \phi}
\newcommand{\xiv}{\boldsymbol \xi}
\newcommand{\thetav}{\boldsymbol \theta}
\newcommand{\alphav}{\boldsymbol \alpha}
\newcommand{\betav}{\boldsymbol \beta}
\newcommand{\piv}{\boldsymbol \pi}
\newcommand{\omegav}{\boldsymbol \omega}
\newcommand{\gammav}{\boldsymbol \gamma}
\newcommand{\ep}{\mathbb{E}}
\newcommand{\st}{\mathrm{s.t.}}
\newcommand{\KL}{\mathrm{KL}}
\newcommand{\B}{\mathrm{Beta}}

\newcommand{\Ber}{\mathrm{Bernoulli}}

\newcommand{\model}{\mathbf{M}}
\newcommand{\ms}{\mathcal{M}}
\newcommand{\data}{\mathcal{D}}
\def\indicator{{\mathbb I}}

\usepackage{jmlr2e}
\usepackage{enumerate}
\usepackage{algorithm}
\usepackage{algorithmic}
\usepackage{epsf}
\usepackage{multirow}
\usepackage{wrapfig}
\usepackage{graphicx}
\usepackage{subfigure} 
\usepackage[english]{babel}

\usepackage{color}
\newcommand{\ericx}[1]{{\color{red}{\bf\sf [Eric: #1]}}}

\jmlrheading{15}{2014}{?-?}{3/13; Revised 9/13}{?}{Jun Zhu, Ning Chen and Eric P. Xing}

\ShortHeadings{Regularized Bayesian Inference and Infinite Latent SVMs}{Zhu, Chen and Xing} \firstpageno{1}

\begin{document}

\title{Bayesian Inference with Posterior Regularization \\ and applications to Infinite Latent SVMs}

\author{\name Jun Zhu \email dcszj@mail.tsinghua.edu.cn \\
        \name Ning Chen \email ningchen@mail.tsinghua.edu.cn\\
        \addr State Key Laboratory of Intelligent Technology and Systems \\
        \addr Tsinghua National Laboratory for Information Science and Technology \\
        \addr Department of Computer Science and Technology \\
       Tsinghua University \\
        \name Eric P. Xing  \email epxing@cs.cmu.edu\\
        \addr School of Computer Science \\
       Carnegie Mellon University
}

\editor{Tony Jebara}

\maketitle

\begin{abstract}%
Existing Bayesian models, especially nonparametric Bayesian methods, rely on specially conceived priors to incorporate domain knowledge for discovering improved latent representations. While priors can affect posterior distributions through Bayes' rule, imposing posterior regularization is arguably more direct and in some cases more natural and general. In this paper, we present {\it regularized Bayesian inference} (RegBayes), a novel computational framework that performs posterior inference with a regularization term on the desired post-data posterior distribution under an information theoretical formulation. RegBayes is more flexible than the procedure that elicits expert knowledge via priors, and it covers both directed Bayesian networks and undirected Markov networks whose Bayesian formulation results in hybrid chain graph models. When the regularization is induced from a linear operator on the posterior distributions, such as the expectation operator, we present a general convex-analysis theorem to characterize the solution of RegBayes. Furthermore, we present two concrete examples of RegBayes, {\it infinite latent support vector machines} (iLSVM) and {\it multi-task infinite latent support vector machines} (MT-iLSVM), which explore the large-margin idea in combination with a nonparametric Bayesian model for discovering predictive latent features for classification and multi-task learning, respectively. We present efficient inference methods and report empirical studies on several benchmark datasets, which appear to demonstrate the merits inherited from both large-margin learning and Bayesian nonparametrics. Such results were not available until now, and contribute to push forward the interface between these two important subfields, which have been largely treated as isolated in the community.


\end{abstract}

\section{Introduction}

Over the past decade, nonparametric Bayesian models have gained remarkable popularity in machine learning and other fields, partly owing to their desirable utility as a ``nonparametric" prior distribution for a wide variety of probabilistic models, thereby turning the largely heuristic model selection practice, such as determining the unknown number of components in a mixture model~\citep{Antoniak:74} or the unknown dimensionality of latent features in a factor analysis model~\citep{Griffiths:tr05}, as a Bayesian inference problem in an unbounded model space. Popular examples include Gaussian process (GP)~\citep{Rasmussen:02}, Dirichlet process (DP)~\citep{Ferguson:73,Antoniak:74}, and Beta process (BP)~\citep{Thibaux:beta07}. DP is often described with a Chinese restaurant process (CRP) metaphor, and similarly BP is often described with an Indian buffet process (IBP) metaphor~\citep{Griffiths:tr05}. 
Such nonparametric Bayesian approaches allow the model complexity to grow as more data are observed, which is a key factor differing them from other traditional ``parametric" Bayesian models.

One recent development in practicing Bayesian nonparametrics is to relax some unrealistic assumptions on data, such as homogeneity and exchangeability. For example, to handle heterogenous observations, predictor-dependent processes~\citep{MacEachern99,Williamson:10} have been proposed; and to relax the exchangeability assumption, stochastic processes with various correlation structures, such as hierarchical structures~\citep{YWTeh:jasa06}, temporal or spatial dependencies~\citep{Beal:iHMM07,Blei:icml10}, and stochastic ordering dependencies~\citep{Hoff:03,Dunson:07}, have been introduced. A common principle shared by these approaches is that they rely on defining, or in some unusual cases learning~\citep{Welling:ITA12} a nonparametric Bayesian prior\footnote{Although likelihood is another dimension that can incorporate domain knowledge, existing work on Bayesian nonparametrics has been mainly focusing on the priors. Following this convention, this paper assumes that a common likelihood model (e.g., Gaussian likelihood for continuous data) is given.} encoding some special structures, which {\it indirectly}\footnote{A hard constraint on the prior (e.g., a truncated Gaussian) can directly affect the support of the posterior. RegBayes covers this as a special case as shown in Remark~\ref{remarkPriorConstraint}.} influences the posterior distribution of interest through an interplay with a likelihood model according to the Bayes' rule (also known as Bayes' theorem). In this paper, we explore a different principle known as {\it posterior regularization}, which offers an additional and arguably richer and more flexible set of means to augment a posterior distribution under rich side information, such as predictive margin, structural bias, etc., which can be harder, if possible, to be captured by a Bayesian prior.

Let $\Theta$ denote model parameters and $H$ denote hidden variables. Then given a set of observed data $\data$, posterior regularization~\citep{Taskar:postreg10} is generally defined as solving a regularized maximum likelihood estimation (MLE) problem:
\begin{eqnarray}\label{eq:posterior-regularization}
{\bf Posterior~Regularization}:~~\max_\Theta \mathcal{L}(\Theta; \data) + \Omega( p(H | \data, \Theta) ),
\end{eqnarray}
where $\mathcal{L}(\Theta; \data)$ is the marginal likelihood of $\data$, and $\Omega(\cdot)$ is a regularization function of the model posterior over latent variables (note that here we view posterior as a generic post-data distribution on hidden variables in the sense of \citep[pp.15]{Ghosh:book2003}, not necessarily corresponding to a Bayesian posterior that must be induced by the Bayes' rule). 
The regularizer can be defined as a KL-divergence between a desired distribution with certain properties over latent variables and the model posterior in question, or other constraints on the model posterior, such as those used in generalized expectation~\citep{McCallum:jmlr10} or constraint-driven semi-supervised learning~\citep{Chang:07}.
An EM-type procedure can be applied to solve Eq.~(\ref{eq:posterior-regularization}) approximately, and obtain an augmented MLE of the hidden variable model: $p(H | \data, \Theta_{\rm MLE}$). When a distribution over the model parameter is of interest, going beyond the classical Bayesian theory, recent attempts toward learning a regularized posterior distribution of model parameters (and latent variables as well if present) include the ``learning from measurements"~\citep{Liang:ICML09}, maximum entropy discrimination (MED)~\citep{Jaakkola:99,Zhu:jmlr09} and maximum entropy discrimination latent Dirichlet allocation (MedLDA)~\citep{Zhu:MedLDA09}.
All these methods are parametric in that they give rise to distributions over a fixed and finite-dimensional parameter space. To the best of our knowledge, very few attempts have been made to impose posterior regularization in a nonparametric setting where model complexity depends on data, such as the case for nonparametric Bayesian latent variable models. A general formalism for (parametric and nonparametric) Bayesian inference with posterior regularization seems to be not yet available or apparent. In this paper, we present such a formalism, which we call {\it regularized Bayesian inference}, or RegBayes, built on the convex duality theory over distribution function spaces; and we apply this formalism to learn regularized posteriors under the Indian buffet process (IBP), conjoining two powerful machine learning paradigms, nonparametric Bayesian inference and SVM-style max-margin constrained optimization.

Unlike the regularized MLE formulation in Eq.~(\ref{eq:posterior-regularization}), under the traditional formulation of Bayesian inference one is not directly optimizing an objective with respect to the posterior. To enable a regularized optimization formulation of  RegBayes, we begin with a variational reformulation of the Bayes' theorem, and define $\mathcal{L}( q(\model | \data) )$ as the KL-divergence between a desired post-data posterior $q(\model|\data)$ over model $\model$ and the standard Bayesian posterior $p(\model|\data)$ (see Section~\ref{sec:Bayes-As-Optimization} for a recapitulation of the connection between KL-minimization and Bayes' theorem). RegBayes solves the following optimization problem:
\begin{eqnarray}\label{eq:RegBayes-Generic}
{\bf RegBayes}:~~\inf_{q(\model|\data) \in \mathcal{P}_{\textrm{prob}} } \mathcal{L}(q(\model|\data)) + \Omega( q(\model|\data) ),
\end{eqnarray}
where the regularization $\Omega(\cdot)$ is a function of the post-data posterior $q(\model|\data)$, and $\mathcal{P}_{\textrm{prob}}$ is the feasible space of well-defined distributions. By appropriately defining the model and its prior distribution, RegBayes can be instantiated to perform either parametric and nonparametric regularized Bayesian inference.

One particularly interesting way to derive the posterior regularization is to impose posterior constraints. Let $\xiv$ denote slack variables and $\mathcal{P}_{\textrm{post}}(\xiv)$ denote the general soft posterior constraints (see Section 3.2 for a formal description), then, we can express the regularization term variationally:
\setlength\arraycolsep{1pt}\begin{eqnarray}\label{eq:RegBayes-VarReg}
\Omega( q(\model|\data) )  = \inf_{\xiv}~ U(\xiv) ,~~\textrm{s.t.:}~ q( \model|\data ) \in \mathcal{P}_{\textrm{post}}(\xiv),
\end{eqnarray}
{where $U(\xiv)$ is normally defined as a convex penalty function}. The RegBayes formalism defined in Eq.~(\ref{eq:RegBayes-Generic}) applies to a wide spectrum of models, including directed graphical models (i.e., Bayesian networks) and undirected Markov networks. For undirected models, when performing Bayesian inference the resulting posterior takes the form of a hybrid chain graphical model~\citep{Frydenberg:90}~\citep{Murray:UAI04,AlanQi:05,Welling:UAI06}, which is usually much more challenging to regularize than for Bayesian inference with directed GMs. When the regularization term is convex and induced from a linear operator (e.g., expectation) of the posterior distributions, RegBayes can be solved with convex analysis theory.

By allowing direct regularization over posterior distributions, RegBayes provides a significant source of extra flexibility for post-data posterior inference, which applies to both parametric and nonparametric Bayesian learning (see the remarks after the main Theorem~\ref{lemma:RegBayes}). In this paper, we focus on applying this technique to the later case, and illustrate how to use RegBayes to facilitate integration of Bayesian nonparametrics and large-margin learning, which have complementary advantages but have been largely treated as two disjoint subfields. Previously, it has been shown that, the core ideas of support vector machines~\citep{Vapnik:95} and maximum entropy discrimination~\citep{Jaakkola:99}, as well as their structured extensions to the max-margin Markov networks~\citep{Taskar:03} and maximum entropy discrimination Markov networks~\citep{Zhu:jmlr09}, have led to successful outcomes in many scenarios. But a large-margin model rarely has the flexibility of nonparametric Bayesian models to automatically handle  model complexity from data, especially when latent variables are present~\citep{Jebara:thesis,Zhu:MedLDA09}. In this paper, we intend to bridge this gap using the RegBayes principle.

Specifically, we develop the {\it infinite latent support vector machines} (iLSVM) and {\it multi-task infinite latent support vector machines} (MT-iLSVM), which explore the discriminative large-margin idea to learn infinite latent feature models for classification and multi-task learning~\citep{Argyriou:nips07,Bakker:JMLR03}, respectively.
We show that both models can be readily instantiated from the RegBayes master equation~(\ref{eq:RegBayes-Generic}) by defining appropriate posterior regularization using the large-margin principle, and by employing an appropriate prior. For iLSVM, we use the IBP prior to allow the model to have an unbounded number of latent features {\it a priori}. For MT-iLSVM, we use a similar IBP prior to infer a latent projection matrix to capture the correlations among multiple predictive tasks while avoiding pre-specifying the dimensionality of the projection matrix. The regularized inference problems can be efficiently solved with an iterative procedure, which leverages existing high-performance convex optimization techniques.

The rest of the paper is organized as follows. Section 2 discusses related work. Section 3 presents regularized Bayesian inference (RegBayes), together with the convex duality results that will be needed in latter sections. Section 4 concretizes the ideas of RegBayes and presents two infinite latent feature models with large-margin constraints for both classification and multi-task learning. Section 5 presents some preliminary experimental results. Finally, Section 6 concludes and discusses future research directions.

\section{Related Work}

Expectation regularization or expectation constraints have been considered to regularize model parameter estimation in the context of semi-supervised learning or learning with weakly labeled data. Mann and McCallum~\citep{McCallum:jmlr10} summarized the recent developments of the generalized expectation (GE) criteria for training a discriminative probabilistic model (e.g., maximum entropy models or conditional random fields~\citep{Lafferty:01}) with unlabeled data. By providing appropriate side information, such as labeled features or estimates of label distributions, a GE-based penalty function is defined to regularize the model distribution, e.g., the distribution of class labels. One commonly used GE function is the KL-divergence between empirical expectation and model expectation of some feature functions if the expectations are normalized or the general Bregman divergence for unnormalized expectations. 
Although the GE criteria can be used alone as a scoring function to estimate the unknown parameters of a discriminative model, it is more usually used as a regularization term to an estimation method, such as maximum (conditional) likelihood estimation. Bellare et al.~\citep{McCallum:uai09} presented a different formulation of using expectation constraints in semi-supervised learning by introducing an auxiliary distribution to GE, 
together with an alternating projection algorithm, which can be more efficient. Liang et al.~\citep{Liang:ICML09} proposed to use the general notion of ``measurements" to encapsulate the variety of weakly labeled data for learning exponential family models. The measurements can be labels, partial labels or other constraints on model predictions. Under the EM framework, posterior constraints were used in~\citep{Taskar:nips07} to modify the E-step of an EM algorithm to project model posterior distributions onto the subspace of distributions that satisfy a set of auxiliary constraints.

Dudik et al.~\citep{Dudik:07} studied the generalized maximum entropy principle with a rich form of expectation constraints using convex duality theory, where the standard moment matching constraints of maximum entropy are relaxed to inequality constraints. But their analysis was restricted to KL-divergence minimization (maximum entropy is a special case) and the finite dimensional space of observations. Later on, Altun and Smola~\citep{Altun:COLT06} presented a more general duality theory for a family of divergence functions on Banach spaces. We have drawn inspiration from both papers to develop the regularized Bayesian inference framework using convex duality theory.

When using large-margin posterior regularization, RegBayes generalizes the previous work on maximum entropy discrimination~\citep{Jaakkola:99,Zhu:jmlr09}. 
The present paper provides a full extension of our preliminary work on max-margin nonparametric Bayesian models~\citep{Zhu:iSVM11,Zhu:iLSVM11}. For example, the infinite SVM (iSVM)~\citep{Zhu:iSVM11} is a latent class model, where each data example is assigned to a single mixture component (i.e., an 1-dimensional space), and both iLSVM and MT-iLSVM extend the ideas to infinite latent feature models. For multi-task learning, nonparametric Bayesian models have been developed in \citep{XueYa:icml07,HalDaume:10} for learning features shared by multiple tasks. However, these methods are based on standard Bayesian inference without a posterior regularization using, for example, the large-margin constraints. Finally, MT-iLSVM can be also regarded as a nonparametric Bayesian formulation of the popular multi-task learning methods~\citep{AndoTong:05,Jebara:jmlr11}.

\section{Regularized Bayesian Inference}\label{sec:RegBayes}

We begin by laying out a general formulation of regularized Bayesian inference, using an optimization framework built on convex duality theory.

\subsection{Variational formulation of Bayes' theorem}\label{sec:Bayes-As-Optimization}

We first derive an optimization-theoretic reformulation of the Bayes' theorem. Let $\ms$ denote the space of feasible models, and $\model \in \ms$ represents an atom in this space. We assume that $\ms$ is a complete separable metric space endowed with its Borel $\sigma$-algebra $\mathcal{B}(\ms)$. Let $\Pi$ be a distribution (i.e., a probability measure) on the measurable space $(\ms, \mathcal{B}(\ms))$. We assume that $\Pi$ is absolutely continuous with respect to some background measure $\mu$, so that there exists a density $\pi$ such that $\ud \Pi = \pi \ud \mu$.
{Let $\data=\{\xv_n \}_{n=1}^N$ be a collection of observed data, which we assume to be i.i.d. given a model.
Let $P( \cdot  |\model)$ be the likelihood distribution, which is assumed to be dominated by a $\sigma$-finite measure $\lambda$ for all $\model$ with positive density, so that there exists a density $p(\cdot|\model)$ such that $\ud P(\cdot | \model) = p(\cdot | \model) \ud \lambda$. Then, the Bayes' conditionalization rule gives a posterior distribution with the density~\citep[Chap.1.3]{Ghosh:book2003}:
\begin{eqnarray}\label{eq:bayesthrm}
p(\model|\data) = \frac{\pi(\model) p(\data|\model)}{p(\data)} = \frac{\pi(\model) \prod_{n=1}^N p(\xv_n|\model)}{p(\xv_1, \cdots, \xv_N)},
\end{eqnarray}
a density over $\model$ with respect to the base measure $\mu$, where $p(\data)$ is the marginal likelihood of the observed data.}

For reasons to be clear shortly, we now introduce a variational formulation of the Bayes' theorem.
Let $Q$ are an arbitrary distribution on the measurable space $(\ms, \mathcal{B}(\ms))$. We assume that $Q$ is absolutely continuous with respect to $\Pi$ and denote by $q$ its density with respect to the background measure $\mu$.\footnote{This assumption is necessary to make the $\KL$-divergence between the two distributions $Q$ and $\Pi$ well-defined. This assumption (or constraint) will be implicitly included in $\mathcal{P}_{\textrm{prob}}$ for clarity.} 
It can be shown that the posterior distribution of $\model$ due to the Bayes' theorem is equivalent to the optimum solution of the following convex optimization problem:
\begin{eqnarray}\label{eq:BasicMaxEnt}
\inf_{q(\model)} && \KL(q(\model) \Vert \pi(\model)) - \int_{\ms}  \log p(\data|\model) q(\model) \ud \mu(\model) \\
\mathrm{s.t.:} && q(\model) \in \mathcal{P}_{\textrm{prob}}, \nonumber
\end{eqnarray}
where $\KL(q(\model)\Vert \pi(\model)) = \int_\ms q(\model) \log (q(\model) / \pi(\model)) \ud \mu(\model)$ is the Kullback-Leibler (KL) divergence from $q(\cdot)$ to $\pi(\cdot)$, and $\mathcal{P}_{\textrm{prob}}$ represents the feasible space of all density functions over $\model$ {with respect to the measure $\mu$}. 
The proof is straightforward by noticing that the objective will become $\KL(q(\model) \Vert p(\model|\data))$ by adding the constant $\log p(\data)$. It is noteworthy that $q(\model)$ here represents the density of a general post-data posterior distribution in the sense of \citep[pp.15]{Ghosh:book2003}, not necessarily corresponding to a Bayesian posterior that is induced by the Bayes' rule. {As we shall see soon later, when we introduce additional constraints, the post-data posterior $q(\model)$ is different from the Bayesian posterior $p(\model|\data)$, and moreover, it could even not be obtainable from any Bayesian conditionalization in a different model}. In the sequel, in order to distinguish $q(\cdot)$ from the Bayesian posterior, we will call it post-data distribution\footnote{Rigorously, $q(\cdot)$ is the density of the post-data posterior distribution $Q(\cdot)$. We simply call $q$ a distribution if no confusion arises.} in short or post-data posterior distribution in full. For notation simplicity, we have omitted the condition $\data$ in the post-data posterior distribution $q(\model)$.

\begin{remark}
The optimization formulation in~(\ref{eq:BasicMaxEnt}) implies that Bayes' rule is an information projection procedure that projects a prior density to a post-data posterior by taking account of the observed data. In general, Bayes's rule is a special case of the principle of minimum information~\citep{Williams:BayesCond1980}.
\end{remark}

\subsection{Regularized Bayesian Inference with Expectation Constraints}\label{sec:regBayes}

{In the variational formulation of Bayes' rule in Eq. (\ref{eq:BasicMaxEnt}), the constraints on $q(\model)$ ensure that $q$ is well-normalized and the objective is well-defined, i.e., $q(\model) \in \mathcal{P}_{\textrm{prob}}$, which do not capture any domain knowledge or structures of the model or data. Some previous efforts have been devoted to eliciting domain knowledge by constraining the prior or the base measure $\mu$~\citep{Christian:1995,Garthwaite:jasa05}. As we shall see, such constraints without considering data are special cases of RegBayes to be presented.}

Specifically, the optimization-based formulation of Bayes' rule makes it straightforward to generalize Bayesian inference to a richer type of posterior inference, by replacing the standard normality constraint on $q$ with a wide spectrum of knowledge-driven and/or data-driven constraints or regularization. (To contrast, we will refer to the problem in Eq. (\ref{eq:BasicMaxEnt}) as ``unconstrained" or ``unregularized".) Formally, we define {\it regularized Bayesian inference} (RegBayes) as a generalized posterior inference  procedure that solves a constrained optimization problem due to such additional regularization imposed on $q$:
\begin{eqnarray}\label{eq:constraindBayes}
\inf_{q(\model), \xiv} && \KL(q(\model) \Vert \pi(\model)) - \int_{\ms}  \log p(\data|\model) q(\model) \ud \mu(\model) + U(\xiv) \\
\mathrm{s.t.:} && q(\model) \in \mathcal{P}_{\textrm{post}}(\xiv), \nonumber
\end{eqnarray}
\noindent where $\mathcal{P}_{\textrm{post}}(\xiv)$ is a subspace of distributions that satisfy a set of additional constraints besides the standard normality constraint of a probability distribution. Using the variational formulation in Eq. (\ref{eq:RegBayes-VarReg}), problem~(\ref{eq:constraindBayes}) can be rewritten in the form of the master equation~(\ref{eq:RegBayes-Generic}), of which the objective is: $\mathcal{L}(q(\model)) = \KL(q(\model) \Vert \pi(\model)) - \int_{\ms}  \log p(\data|\model) q(\model) \ud \mu(\model) = \KL( q(\model) \Vert p(\model, \data))$ and the posterior regularization is $\Omega(q(\model)) = \inf_{\xiv}~ U(\xiv) ,~\textrm{s.t.:}~ q( \model|\data ) \in \mathcal{P}_{\textrm{post}}(\xiv)$. {Note that when $\data$ is given, the distribution $p(\model, \data)$ is unnormalized for $\model$; and we have abused the KL notation for unnormalized distributions in $\KL( q(\model) \Vert p(\model, \data))$, but with the same formula.}


\begin{figure*}\vspace{-.2cm}
\begin{center}
{\hfill\subfigure[]{\includegraphics[height=0.25\columnwidth]{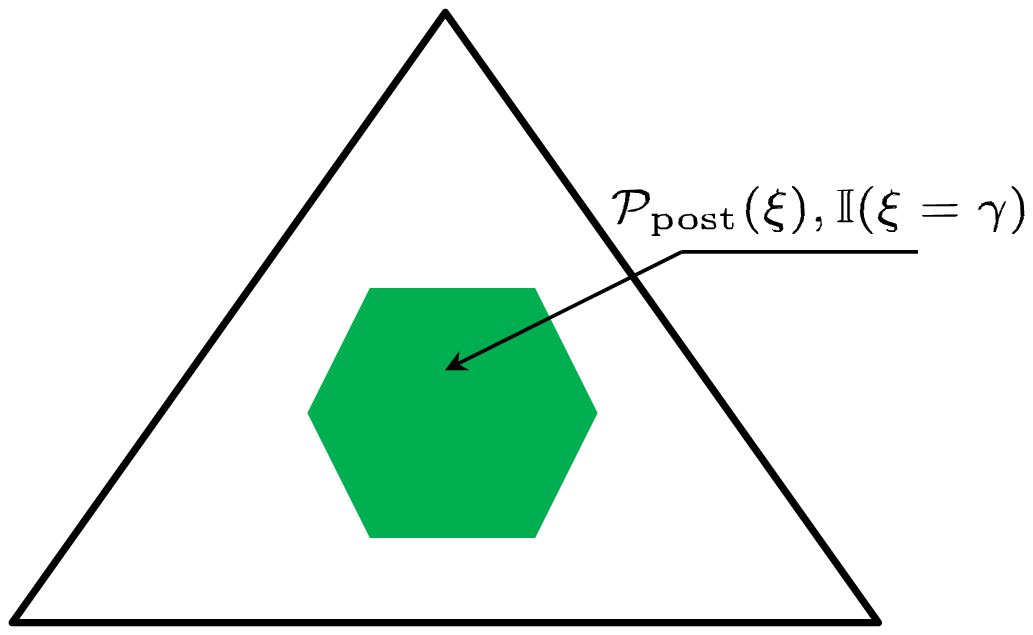}\label{fig:HardConstraints}}\hfill
\subfigure[]{\includegraphics[height=.25\columnwidth]{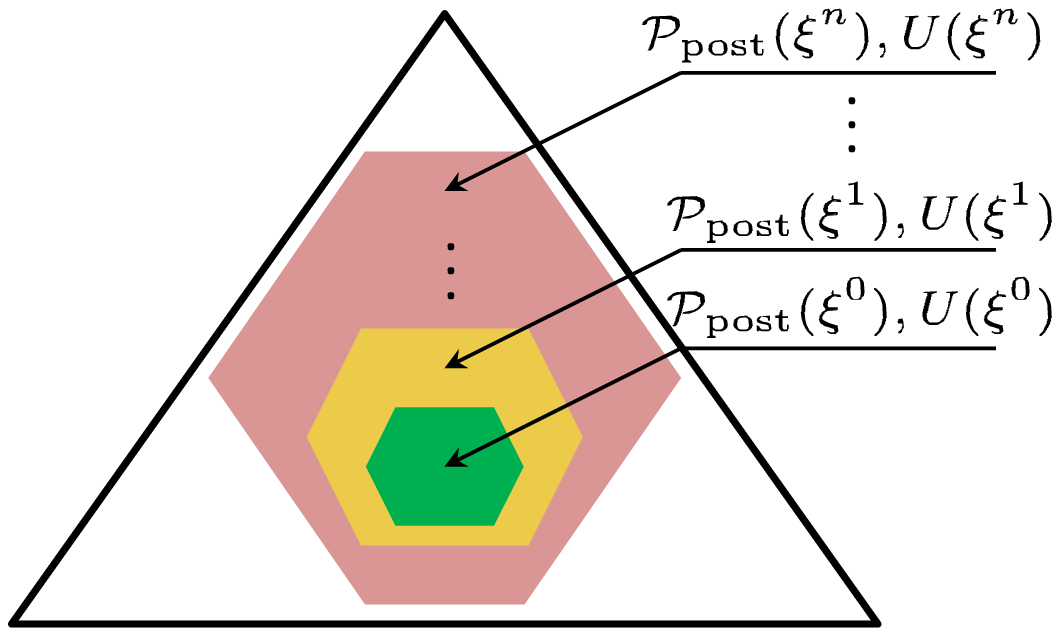}\label{fig:SoftConstraints}}\hfill}\vspace{-.2cm}
\caption{Illustration for the (a) hard and (b) soft constraints in the simple setting which has only three possible models. For hard constraints, we have only one feasible subspace. In contrast, we have many (normally infinite for continuous $\xiv$) feasible subspaces for soft constraints and each of them is associated with a different complexity or penalty, measured by the $U$ function.}
\end{center}\vspace{-.3cm}
\end{figure*}

Obviously this formulation enables different types of constraints to be employed in practice. In this paper, we focus on the {\it expectation constraints}, of which each one is a function of $q(\model)$ through an expectation operator. For instance, let $\psiv = (\psi_1, \cdots, \psi_T)$ be a vector of feature functions, each of which is $\psi_t(\model; \data)$ defined on $\model$ and possibly data dependent. Then a subspace of feasible post-data distributions can be defined in the following form:
\begin{eqnarray}\label{eq:ExpConstraint}
\mathcal{P}_{\textrm{post}}(\xiv) \defEq  \Big\{ q(\model) | ~ \forall t=1,\cdots, T, ~h\big(E q( \psi_t; \data ) \big) \leq \xi_t \Big\} ,
\end{eqnarray}
where $E$ is the expectation operator that maps $q(\model)$ to a point in the space $\mathbb{R}^T$, and for each feature function $\psi_t$:~$E q (\psi_t; \data) \defEq \ep_{ q(\model)}[ \psi_t(\model; \data)]$. The function $h$ can be of any form in theory, though a simple $h$ function will make the optimization problem easy to solve.
The auxiliary parameters $\xiv$ are usually nonnegative and interpreted as slack variables. The constraints with non-trivial $\xiv$ are soft constraints as illustrated in Figure~\ref{fig:SoftConstraints}. But we emphasize that by defining $U$ as an indicator function, the formulation (\ref{eq:constraindBayes}) covers the case where hard constraints are imposed. For instance, if we define $$U(\xiv) = \sum_{t=1}^T \indicator(\xi_t = \gamma_t) = \indicator(\xiv = \gammav),$$ where $\indicator(c)$ is an indicator function that equals to $0$ if the condition $c$ is satisfied; otherwise $\infty$, then all the expectation constraints (\ref{eq:ExpConstraint}) are hard constraints. As illustrated in Figure~\ref{fig:HardConstraints}, hard constraints define one single feasible subspace (assuming to be non-empty). In general, we assume that $U(\xiv)$ is a convex function, which represents a penalty on the size of the feasible subspaces, as illustrated in Figure~\ref{fig:SoftConstraints}. A larger subspace typically leads to models with a higher complexity. In the classification models to be presented, $U$ corresponds to a surrogate loss, e.g., hinge loss of a prediction rule, as we shall see.

Similarly, the formulation of RegBayes with expectation constraints~(\ref{eq:ExpConstraint}) can be equivalently written in an ``unconstrained" form by using the rule in (\ref{eq:RegBayes-VarReg}). Specifically, let $g(E q( \psiv; \data  ) ) \defEq  \inf_{\xiv} U(\xiv),~\textrm{s.t.}:~ h(Eq (\psi_t; \data) ) \leq \xi_t,~\forall t $, we have the equivalent optimization problem:
\begin{eqnarray}\label{eq:constraindBayes2}
\inf_{q(\model) \in \mathcal{P}_{\textrm{prob}}} && \KL(q(\model) \Vert \pi(\model)) - \int_{\ms} \log p(\data|\model) q(\model) \ud \mu(\model) + g( E q(\psiv; \data ) ), 
\end{eqnarray}
where $E q(\psiv; \data )$ is a point in $\mathbb{R}^T$ and the $t$-th coordinate is $E q(\psi_t; \data)$, a function of $q(\model)$ as defined before.
We assume that the real-valued function $g:~\mathbb{R}^T \to \mathbb{R}$ is convex and lower semi-continuous. For each $U$, we can induce a $g$ function by taking the infimum of $U(\xiv)$ over $\xiv$ with the posterior constraints; vice versa.
If we use hard constraints, similar as in regularized maximum entropy density estimation~\citep{Altun:COLT06,Dudik:07}, we have
\begin{eqnarray}
g(Eq) = \sum_{t=1}^T \indicator(h(E q( \psi_t; \data )) \leq \gamma_t).
\end{eqnarray}

For the regularization function $g$, as well as $U$, we can have many choices, besides the above mentioned indicator function. For example, if the feature function $\psi_t$ is an indicator function and we could obtain `prior' expectations $\ep_{\tilde{p}}[\psi_t]$ from domain/expert knowledge about $\model$. {If we further normalize the empirical expectations of $T$ functions and denote the discrete distribution by $\tilde{p}(\model)$, one natural regularization function would be the KL-divergence between prior expectations and the expectations computed from the normalized model posterior $q(\model)$,
i.e., $g(Eq) = \sum_t s(\ep_{\tilde{p}}[\psi_t], Eq(\psi_t)) = \KL(\tilde{p}(\model) \Vert q(\model) )$, where $s(x,y) = x \log(x/y)$ for $x,y \in (0,1)$. The general Bregman divergence can be used for unnormalized expectations.} This kind of regularization function has been used in~\citep{McCallum:jmlr10} for label regularization, in the context of semi-supervised learning. Other choices of the regularization function include the $\ell_2^2$ penalty or indicator function with equality constraints (Please see Table 1 in~\citep{Dudik:07} for a summary).

\begin{remark}
So far, we have focused on RegBayes in the context of full Bayesian inference. Indeed, RegBayes can be generalized to apply to empirical Bayesian inference, where some model parameters need to be estimated. More generally, RegBayes applies to both directed Bayesian networks (of which the hierarchical Bayesian models we have discussed are an example) and undirected Markov random fields. But for undirected models, a RegBayes treatment will have to deal with a chain graph resultant from Bayesian inference, which is more challenging due to existence of normalization factors. We will discuss some details and examples in Appendix A.
\end{remark}

\subsection{Optimization with Convex Duality Theory}

Depending on several factors, including the size of the model space, the data likelihood model, the prior distribution, and the regularization function, a RegBayes problem in general can be highly non-trivial to solve, either in the constrained or unconstrained form, as can be seen from several concrete examples of RegBayes models we will present in the next section and in the Appendix B. 
In this section, we present a representation theorem to characterize the solution the convex RegBayes problem~(\ref{eq:constraindBayes2}) with expectation regularization. {These theoretical results will be used later in developing concrete RegBayes models}. 

To make the subsequent statements general, we consider the following problem: 
\begin{eqnarray}\label{eq:GeneralProblem}
\inf_{x \in \mathcal{X}} f(x) + g(Ax)
\end{eqnarray}
where $f: \mathcal{X} \to \mathbb{R}$ is a convex function; $A: \mathcal{X} \to \mathcal{B}$ is a bounded linear operator; and $g: \mathcal{B} \to \mathbb{R}$ is also convex.
Below we introduce some tools in convex analysis theory to study this problem. We begin by formulating the primal-dual space relationships of convex optimization problems in the general settings, where we assume both $\mathcal{X}$ and $\mathcal{B}$ are Banach spaces\footnote{A Banach space is a vector space with a metric that allows the computation of vector length and distance between vectors. Moreover, a Cauchy sequence of vectors always converges to a well defined limit in the space.}. An important result we build on is the Fenchel duality theorem.
\begin{definition}[Convex Conjugate]
Let $\mathcal{X}$ be a Banach space and $\mathcal{X}^\ast$ be its dual space. The convex conjugate or the Legendre-Frenchel transformation of a function $f:~\mathcal{X} \to [-\infty, +\infty]$ is $f^\ast:~\mathcal{X}^\ast \to [-\infty, +\infty]$, where
\begin{eqnarray}
f^\ast( x^\ast) = \sup_{x \in \mathcal{X}}\{ \langle x^\ast, x \rangle - f(x) \}.
\end{eqnarray}
\end{definition}

\begin{theorem}[Fenchel Duality~\citep{Borwein:05}]\label{thrm:FenchelDuality}
Let $\mathcal{X}$ and $\mathcal{B}$ be Banach spaces, $f: \mathcal{X} \to \mathbb{R}\cup \{+\infty\}$ and $g:~\mathcal{B} \to \mathbb{R}\cup \{+\infty\}$ be convex functions and $A:~\mathcal{X} \to \mathcal{B}$ be a bounded linear map. Define the primal and dual values $t, d$ by the Fenchel problems
\begin{eqnarray}
t = \inf_{x \in \mathcal{X}} \{f(x) + g(Ax) \}~and~d = \sup_{x^\ast \in \mathcal{B}^\ast} \{ -f^\ast(A^\ast x^\ast) - g^\ast(-x^\ast) \}.\nonumber
\end{eqnarray}
Then these values satisfy the weak duality inequality $t \geq d$. If $f$, $g$ and $A$ satisfy either
\begin{eqnarray}
0 \in \mathrm{core}(\mathrm{dom}g - A\mathrm{dom}f) ~\textrm{and both}~f~\textrm{and}~g~are~lower~semicontinuous~(lsc),
\end{eqnarray}
or
\begin{eqnarray}
A \mathrm{dom}f \cap \mathrm{cont}g \neq \emptyset,
\end{eqnarray}
then $t = d$ and the supremum to the dual problem is attainable if finite.
\end{theorem}
Let $\mathcal{S}$ be a subset of a Banach space $\mathcal{B}$. In the above theorem, we say $s$ is in the {\it core} of $\mathcal{S}$, denoted by $s \in \textrm{core}( \mathcal{S} )$, provided that $\cup_{\lambda > 0} \lambda(\mathcal{S} - s) = \mathcal{B}$.

The Fenchel duality theorem has been applied to solve divergence minimization problems for density estimation~\citep{Altun:COLT06,Dudik:07}. 
Let $\psiv \defEq (\psi_1, \cdots, \psi_T)$ be a vector of feature functions. {Each feature function is a mapping, $\psi_t: \ms \to \mathbb{R}$. Therefore, $\mathcal{B}$ is the product space $\mathbb{R}^T$, a simple Banach space. Let $\mathcal{X}$ be the Banach space of finite signed measures (with total variation as the norm) that are absolutely continuous with respect to the measure $\mu$, and let $A$ be the expectation operator of the feature functions with respect to the distribution $q$ on $\ms$, that is, $Aq \defEq \ep_{\model \sim q}[\psiv(\model)]$, where $\psiv(\model) = (\psi_1(\model), \cdots, \psi_T(\model))$. Let $\tilde{\psiv}$ be a reference point in $\mathbb{R}^T$. As for density estimation, we have some observations of $\model$ here, and $\tilde{\psiv} = A p_{\textrm{emp}}[ \psiv(\model) ]$, where $p_{\textrm{emp}}$ is the empirical distribution.}
Then, when the $f$ function is a KL-divergence and the constraints are relaxed moment matching constraints, the following result can be proven.
\begin{lemma}[KL-divergence with Constraints \citep{Altun:COLT06}]\label{lemma:KLConstraints}
\begin{eqnarray}
&& \inf_q \Big\{ \KL(q \Vert p)~\mathrm{s.t.}: \Vert \ep_q[\psiv] - \tilde{\psiv} \Vert_{\mathcal{B}} \leq \epsilon~and~ q \in \mathcal{P}_{\mathrm{prob}} \Big\} \\
= ~ && \sup_{\phiv}\Big\{ \langle \phiv, \tilde{\psiv} \rangle - \log \int_{\ms} p(\model) \exp(\langle \phiv, \psiv(\model) \rangle ) \ud \mu(\model) - \epsilon \Vert \phiv \Vert_{\mathcal{B}^\ast}  \Big\}, \nonumber
\end{eqnarray}
where the unique solution is given by $\hat{q}_{\hat{\phiv}}(\model) = p(\model)\exp(\langle \hat{\phiv}, \psiv(\model) \rangle - \Lambda_{\hat{\phiv}} )$; $\hat{\phiv}$ is the solution of the dual problem; and $\Lambda_{\hat{\phiv}}$ is the log-partition function.
\end{lemma}
Note that for this lemma and the ones to be presented below to hold, the problems need to meet some regularity conditions (or constraint qualifications), such as those in Theorem~\ref{thrm:FenchelDuality}. In practice it can be difficult to check whether the constraint qualifications hold. One solution is to solve the dual optimization problem and examine if the conditions hold depending on whether the solution diverge or not~\citep{Altun:COLT06}.

The problem in the above lemma is subject to hard constraints, therefore the corresponding $g$ is the indicator function $\indicator(\Vert \ep_q[\psiv] - \tilde{\psiv} \Vert_{\mathcal{B}} \leq \epsilon)$ when applying the Fenchel duality theorem. Other examples of the posterior constraints can be found in~\citep{Dudik:07,McCallum:jmlr10,Taskar:postreg10}, as we have discussed in Section~\ref{sec:regBayes}. In this paper, we consider the general soft constraints as defined in the RegBayes problem~(Eq. (\ref{eq:constraindBayes})). Furthermore, we do not assume the existence of a fully observed dataset to compute the empirical expectation $\tilde{\phiv}$. Specifically, following a similar line of reasoning as in~\citep{Altun:COLT06}, though this time with an un-normalized $p$ in $\KL(q\Vert p)$, we have the following result. The detailed proof is deferred to Appendix C.1.
\begin{theorem}[Representation theorem of RegBayes]\label{lemma:RegBayes}
Let $E$ be the expectation operator with feature functions $\psiv(\model;\data)$, and assume $g$ is convex and lower semicontinuous (lsc). We have
\begin{eqnarray}
&& \inf_{q(\model)} \Big\{ \KL(q(\model) \Vert p(\model, \data)) + g(Eq)~\mathrm{s.t.}: q(\model) \in \mathcal{P}_{\mathrm{prob}} \Big\} \\
= ~ && \sup_{\phiv}\Big\{ - \log \int_{\ms}  p(\model, \data) \exp(\langle \phiv, \psiv(\model;\data) \rangle ) \ud \mu(\model) - g^\ast( - \phiv) \Big\}, \nonumber
\end{eqnarray}
where the unique solution is given by $\hat{q}_{\hat{\phiv}}(\model) = p(\model, \data)\exp(\langle \hat{\phiv}, \psiv(\model;\data) \rangle - \Lambda_{\hat{\phiv}} )$; and $\hat{\phiv}$ is the solution of the dual problem; and $\Lambda_{\hat{\phiv}}$ is the log-partition function.
\end{theorem}
From the optimum solution $\hat{q}_{\hat{\phiv}}(\model)$, we can see that the form of the RegBayes posterior is symbolically similar to that of the Bayesian posterior; but instead of multiplying the likelihood term with a prior distribution, RegBayes introduces an extra term, $\exp(\langle \hat{\phiv}, \psiv(\model;\data) \rangle - \Lambda_{\hat{\phiv}} )$, whose coefficients are derived from an constrained optimization problem resultant from the constraints on the posterior. We make the following remarks.
{
\begin{remark}[Putting constraints on priors is a special case of RegBayes]\label{remarkPriorConstraint}
If both the feature function $\psiv(\model;\data)$ and $\hat{\phi}$ depend on the model $\model$ only, this extra term contributes to define a new prior $\pi^\prime(\model) \propto \pi(\model) \exp(\langle \hat{\phiv}, \psiv(\model;\data) \rangle - \Lambda_{\hat{\phiv}} )$. For example, if we constrain the model space to a subset $\ms_0 \subset \ms$ a priori, this constraint can be incorporated in RegBayes by defining the expectation constraint on $\model$ only. Specifically, define the single feature function $\psi(\model)$: $\psi(\model) = 0$ if $\model \in \ms_0$, otherwise $1$; and define the simple posterior regularization $g(Eq) = \indicator(\ep_q[ \psi(\model) ] = 0 )$. Then, by Theorem~\ref{lemma:RegBayes},\footnote{We also used the fact that if $f(x) = \indicator(x = c)$ is an indicator function, its conjugate is $f^\ast(\mu) = c \cdot \mu$.} we have $\hat{\phi} = -\infty$ and $\hat{q}_{\hat{\phi}}(\model)  \propto \pi^\prime(\model) p(\data|\model) $, where $\pi^\prime(\model) \propto \pi(\model) \indicator( \model \in \ms_0)$ is the constrained prior. Therefore, such a constraint lets RegBayes cover the widely used truncated priors, such as truncated Gaussian~\citep{Christian:1995}.
\end{remark}
\begin{remark}[RegBayes is more flexible than Bayes' rule] For the more general case where $\psiv(\model;\data)$ depends on both $\model$ and $\data$, the term $p(\model, \data)\exp(\langle \hat{\phiv}, \psiv(\model;\data) \rangle)$ implicitly defines a joint distribution on $(\model,\data)$ if it has a finite measure. In this case, RegBayes is doing implicit Bayesian conditionalization, that is, the posterior $\hat{q}_{\hat{\phiv}}(\model)$ can be obtained through Bayes' rule with some well-defined prior and likelihood. However, it could be that the integral of $p(\model, \data)\exp(\langle \hat{\phiv}, \psiv(\model;\data) \rangle)$ with respect to $(\model, \data)$ is not finite because of the way $\hat{\phiv}$ varies with $\data$,\footnote{Note: this does not affect the well-normalization of the posterior $\hat{q}_{\hat{\phiv}}(\model)$ because its integral is taken over $\model$ only, with $\data$ fixed.} in which case there is no implicit prior and likelihood that give back $\hat{q}_{\hat{\phiv}}(\model)$ through Bayesian conditionalization. Therefore, RegBayes is more flexible than the standard Bayesian inference, where the prior and likelihood model are explicitly defined, but no additional constraints or regularization can be systematically incorporated. The recent work~\citep{Zhu:RegBayes-icml14} presents an example. Specifically, we show that incorporating domain knowledge via posterior regularization can lead to a flexible framework that automatically learns the importance of each piece of knowledge, thereby allowing for a robust incorporation, which is important in the scenarios where noisy knowledge is collected from crowds. In contrast, eliciting expert knowledge via fitting some priors is generally hard, especially in high-dimensional spaces, as experts are normally good at perceiving low-dimensional and well-behaved distributions but can be very bad in perceiving high-dimensional or skewed distributions~\citep{Garthwaite:jasa05}.
\end{remark}}

It is worth mentioning that although the above theorem provides a generic representation of the solution to RegBayes, in practice we usually need to make additional assumptions in order to make either the primal or dual problem tractable to solve. Since such assumptions could make the feasible space non-convex, additional cautions need to be paid. For instance, the mean-field assumptions will lead to a non-convex feasible space~\citep{Wainwright:book08}, and we can only apply the convex analysis theory to deal with convex sub-problems within an EM-type procedure. More concrete examples will be provided later along the developments of various models. We should also note that the modeling flexibility of RegBayes comes with risks. For example, it might lead to inconsistent posteriors~\citep{Barron:99,Choi:08}. This paper focuses on presenting several practical instances of RegBayes and we leave a systematic analysis of the Bayesian asymptotic properties (e.g., posterior consistency and convergence rates) for future work.

Now, we derive the conjugate functions of three examples which will be used shortly for developing the infinite latent SVM models we have intended. We defer the proof to Appendix C.
Specifically, the first one is the conjugate of a simple function, which will be used in a binary latent SVM classification model.
\begin{lemma}\label{proposition:ConjugateBinaryFunc}
Let $g_0: \mathbb{R} \to \mathbb{R}$ be defined as $g_0(x) = C \max(0, x)$. Then, we have $$g_0^\ast(\mu) = \indicator(0 \leq \mu \leq C).$$
\end{lemma}

The second function is slightly more complex, which will be used for defining a multi-way latent SVM classifier. Specifically, we define the function $g_1:~\mathbb{R}^L \to \mathbb{R}$ as
\begin{eqnarray}\label{eq:gFunc}
g_1(\xv) = C \max(\xv),
\end{eqnarray}
where $\max(\xv) \defEq \max(x_1, \cdots, x_L)$. Apparently, $g_1$ is convex because it is a point-wise maximum~\citep{Boyd:04} of the simple linear functions $\phi_i(\xv) = x_i$. Then, we have the following results.
\begin{lemma}\label{proposition:ConjugateMaxFunc}
The convex conjugate of $g_1(\xv)$ as defined above is
\begin{eqnarray}
g_1^\ast(\muv) = \indicator\Big( \forall i, \mu_i \geq 0;~and~\sum_i \mu_i = C\Big) .\nonumber
\end{eqnarray}
\end{lemma}

Let $y \in \mathbb{R}$ and $\epsilon \in \mathbb{R}_+$ are fixed parameters. The last function that we are interested in is $g_2: \mathbb{R} \to \mathbb{R}$, where
\begin{eqnarray}\label{eq:gFunc2}
g_2(x; y, \epsilon ) = C \max(0, |x - y| - \epsilon).
\end{eqnarray}
Finally, we have the following lemma, which will be used in developing large-margin regression models.
\begin{lemma}\label{lemma:ConjugateMaxFunc2}
The convex conjugate of $g_2(x)$ as defined above is
\begin{eqnarray}
g_2^\ast(\mu; y, \epsilon) = \mu y + \epsilon |\mu| + \indicator\big( |\mu| \leq C \big) .\nonumber
\end{eqnarray}
\end{lemma}

\section{Infinite Latent Support Vector Machines}\label{sec:ilsvm}



Given the general theoretical framework of RegBayes introduced in Section~\ref{sec:RegBayes}, now we are ready to present its application to the development of two interesting nonparametric RegBayes models. In these two models we conjoin the ideas behind the nonparametric Bayesian infinite feature model known as the Indian buffet process (IBP), and the large margin classifier known as support vector machines (SVM) to build a new class of models for simultaneous single-task (or multi-task) classification and feature learning. A parametric Bayesian model is presented in Appendix B.

Specifically, to illustrate how to develop latent large-margin classifiers and automatically resolve the unknown dimensionality of latent features from data, we demonstrate how to choose/define the three key elements of RegBayes, that is, {\it prior distribution}, {\it likelihood model}, and {\it posterior regularization}.
We first present the single-task classification model. The basic setup is that we project each data example $\xv \in \mathcal{X} \subset \mathbb{R}^D$ to a latent feature vector $\zv$. Here, we consider binary features. Real-valued features can be easily considered by elementwisely multiplying $\zv$ by a Guassian vector~\citep{Griffiths:tr05}. Given a set of $N$ data examples, let $\Zv$ be the matrix, of which each row is a binary vector $\zv_n$ associated with data sample $n$. Instead of pre-specifying a fixed dimension of $\zv$, we resort to the nonparametric Bayesian methods and let $\zv$ have an infinite number of dimensions. To make the expected number of active latent features finite, we employ an IBP as prior for the binary feature matrix $\Zv$, as reviewed below.


\subsection{Indian Buffet Process}

Indian buffet process (IBP) was proposed in~\cite{Griffiths:tr05} and has been successfully applied in various fields, such as link prediction~\citep{Miller:nips09} and multi-task learning~\citep{HalDaume:10}. We will make use of its stick-breaking construction~\citep{YWTeh:aistats07}, which is good for developing efficient inference methods. Let $\pi_k \in (0, 1)$ be a parameter associated with each column of the binary matrix $\Zv$. Given $\pi_k$, each $z_{nk}$ in column $k$ is sampled independently from $\Ber(\pi_k)$. The parameter $\piv$ are generated by a stick-breaking process
\begin{eqnarray}
\pi_1 = \nu_1, ~\textrm{and}~ \pi_k = \nu_k \pi_{k-1} = \prod_{i=1}^k \nu_i,
\end{eqnarray}
where $\nu_i \sim \B(\alpha, 1)$. Since each $\nu_i$ is less than 1, this process generates a decreasing sequence of $\pi_k$. Specifically, given a finite dataset, the probability of seeing feature $k$ decreases exponentially with $k$.

{IBP has several properties. For a finite number of rows, $N$, the prior of the IBP gives zero mass on matrices with an infinite number of ones, as the total number of columns with non-zero entries is $\textrm{Poisson}(\alpha H_N)$, where $H_N$ is the $N$th harmonic number, $H_N = \sum_{j=1}^N \frac{1}{j}$. Thus, $\Zv$ has almost surely only a finite number of non-zero entries, though this number is unbounded. A second property of IBP is that the number of features possessed by each data point follows a $\textrm{Poisson}(\alpha)$ distribution. Therefore, the expected number of non-zero entries in $\Zv$ is $N\alpha$.}

\subsection{Infinite Latent Support Vector Machines}

Consider a single-task, but multi-way classification, where each training data is provided with a categorical label $y \in \mathcal{Y} \defEq \{1, \cdots, L\}$. 
Suppose that the latent features $\zv_n$ for document $n$ are given, then we can define the {\it latent discriminant function} as linear
\begin{eqnarray}\label{eq:latent-func-ilsvm}
f(y, \xv_n, \zv_n; \etav) \defEq \etav^\top \gv(y, \xv_n, \zv_n),
\end{eqnarray}
where $\gv(y, \xv_n, \zv_n)$ is a vector stacking $L$ subvectors\footnote{We can consider the input features $\xv_n$ or its certain statistics in combination with the latent features $\zv_n$ to define a classifier boundary, by simply concatenating them in the subvectors.} of which the $y$th is $\zv_n^\top$ and all the others are zero; $\etav$ is the corresponding infinite-dimensional 
vector of feature weights. Since we are doing Bayesian inference, we need to maintain the entire distribution profile of the latent feature matrix $\Zv$. However, in order to make a prediction on the observed data $\xv$, we need to remove the uncertainty of $\Zv$. Here, we define the {\it effective discriminant function} as an expectation\footnote{Although other choices such as taking the mode are possible, our choice could lead to a computationally easy problem because expectation is a linear functional of the distribution under which the expectation is taken. Moreover, expectation can be more robust than taking the mode~\citep{Murphy:nips10}, and it has been widely used in~\citep{Zhu:MedLDA09,Zhu:iSVM11}.} (i.e., a weighted average considering all possible values of $\Zv$) of the latent discriminant function. To fully explore the flexibility offered by Bayesian inference, we also treat $\etav$ as random and aim to infer its posterior distribution from given data. For the prior, we assume all the dimensions of $\etav$ are independent and each dimension $\eta_k$ follows the standard normal distribution. This is in fact a Gaussian process (GP) prior as $\etav$ is infinite dimensional. More formally, the effective discriminant function $f: \mathcal{X}\times\mathcal{Y} \mapsto \mathbb{R}$ is
\begin{eqnarray}\label{eq:effective-disc-func}
f\big(y, \xv_n; q(\Zv, \etav, \Wv)\big) &\defEq& \ep_{q(\Zv, \etav, \Wv)} \big\lbrack f(y, \xv_n, \zv_n; \etav) \big\rbrack \\
&=& \ep_{q(\Zv,\etav, \Wv)}\big\lbrack \etav^\top \gv(y, \xv_n, \zv_n)\big\rbrack, \nonumber
\end{eqnarray}
where $q(\Zv, \etav, \Wv)$ is the post-data posterior distribution we want to infer. We have included $\Wv$ as a place holder for any other variables we may define, e.g., the variables arising from a data likelihood model. Since we are taking the expectation, the variables which do not appear in the feature map $\gv$ (i.e., $\Wv$) will be marginalized out.

{Before moving on, we should note that since we require $q$ to be absolutely continuous with respect to the prior to make the KL-divergence term well defined in the RegBayes problem, $q(\Zv)$ will also put zero mass on $\Zv$'s with an infinite number of non-zero entries, because of the properties of the IBP prior. The sparsity of $\Zv$ is essential to ensure that the dot-product in Eq.~(\ref{eq:latent-func-ilsvm}) and the expectation in Eq.~(\ref{eq:effective-disc-func}) are well defined, i.e., with finite values\footnote{A more rigorous derivation of finiteness of these quantities is beyond the scope of this work and could require additional technical conditions~\citep{Orbanz:2012}. We refer the readers to~\citep{Stummer:it2012} for a generic definition of Bregman divergence (or KL divergence in particular) on Banach spaces and in the case where the second measure is unnormalized.}. Moreover, in practice, to make the problem computationally feasible, we usually set a finite upper bound $K$ to the number of possible features, where $K$ is sufficiently large and known as the truncation level (See Section~\ref{sec:inference} and Appendix D.2 for details). As shown in~\citep{Doshi-Velez:09}, the $\ell_1$-distance truncation error of marginal distributions decreases exponentially as $K$ increases. For a finite truncation level, all the expectations are definitely finite.}

Let $\iTrain$ denote the set of training data. Then, with the above definitions, we define the $\mathcal{P}_{\mathrm{post}}(\xiv)$ in problem~(\ref{eq:constraindBayes}) using soft\footnote{Hard constraints for the separable cases are covered by simply setting $\xiv=0$.} large-margin constraints as
\setlength\arraycolsep{1pt} \begin{eqnarray}\label{eq:svmConstraints}
\mathcal{P}_{\mathrm{post}}^{c}(\xiv) \defEq \left\{ q(\Zv, \etav, \Wv)~ \begin{array}{|cl}
\forall n \in \iTrain: &  \Delta f(y, \xv_n; q(\Zv, \etav, \Wv))  \geq  \ell_n^\Delta(y) - \xi_n, \forall y \\
{} & \xi_n \geq 0
\end{array}\right\}, \nonumber
\end{eqnarray}
where $ \Delta f(y, \xv_n; q(\Zv, \etav, \Wv)) \defEq f(y_n, \xv_n; q(\Zv, \etav, \Wv)) - f(y, \xv_n; q(\Zv, \etav, \Wv))$ is the margin favored by the true label $y_n$ over an arbitrary label $y$ and the superscript is used to distinguish from the posterior constraints for multi-task iLSVM to be presented. We define the penalty function for classification as $$U^c(\xiv) \defEq C\sum_{n \in \iTrain} \xi_n^\kappa,$$ where $\kappa \geq 1$. If $\kappa$ is 1, minimizing $U^c(\xiv)$ is equivalent to minimizing the hinge-loss (or $\ell_1$-loss) $\mathcal{R}^c_h$ of the averaging prediction rule~(\ref{eq:predrule}), where $$\mathcal{R}^c_h(q(\Zv, \etav, \Wv)) = C\sum_{n \in \iTrain} \max_y\big( \ell_n^\Delta(y) - \Delta f(y_n, \xv_n; q(\Zv, \etav, \Wv)) \big);$$ if $\kappa$ is 2, the surrogate loss is the squared $\ell_2$-loss. For clarity, we consider the hinge loss. The non-negative cost function $\ell_n^\Delta(y)$ (e.g., 0/1-cost) measures the cost of predicting $\xv_n$ to be $y$ when its true label is $y_n$. $\iTrain$ is the index set of training data.

Besides performing the prediction task, we may also be interested in explaining observed data $\xv$ using the latent factors $\Zv$. This can be done by defining a likelihood model $p(\xv|\Zv)$. 
Here, we define the most common linear-Gaussian likelihood model for real-valued data
\begin{eqnarray}\label{eq:iLSVM-GaussLikelihood}
p\big( \xv_n | \zv_n, \Wv, \sigma_{n0}^2\big) = \mathcal{N}\big(\xv_n | \Wv\zv_n^\top, \sigma_{n0}^2 I\big),
\end{eqnarray}
where $\Wv$ is a $D \times \infty$ random loading matrix. We assume $\Wv$ follows an independent Gaussian prior {and each entry has the prior distribution $\pi(w_{dk}) = \mathcal{N}(w_{dk}|0, \sigma_{0}^2)$}. The hyperparameters $\sigma_0^2$ and $\sigma_{n0}^2$ can be set a priori or estimated from observed data (See Appendix D.2 for details). Figure~\ref{fig:multitaskInfLSVM} (a) shows the graphical structure of iLSVM as defined above, where the plate means $N$ replicates. 



{\bf Training}: Putting the above definitions together, we get the RegBayes problem for iLSVM in the following two equivalent forms
\begin{eqnarray}\label{eq:iLSVMconstrained}
\inf_{q(\Zv, \etav, \Wv), \xiv} && \KL(q(\Zv, \etav, \Wv) \Vert p(\Zv, \etav, \Wv, \data)) + U^c(\xiv)  \\
\mathrm{s.t.}:~&& q(\Zv, \etav, \Wv) \in \mathcal{P}_{\mathrm{post}}^c(\xiv) \nonumber
\end{eqnarray}
\begin{eqnarray} \label{eq:iLSVMunconstrained}
\iff ~~\inf_{q(\Zv, \etav, \Wv) \in \mathcal{P}_{\mathrm{prob}}} && \KL(q(\Zv, \etav, \Wv) \Vert p(\Zv, \etav, \Wv, \data)) + \mathcal{R}^c_h(q(\Zv, \etav, \Wv)),
\end{eqnarray}
where $p(\Zv,\etav, \Wv, \data) = \pi(\etav) \pi(\Zv) \pi(\Wv) \prod_{n=1}^N p(\xv_n | \zv_n, \Wv, \sigma_{n0}^2)$ is the joint distribution of the model; $\pi(\Zv)$ is an IBP prior; and $\pi(\etav)$ and $\pi(\Wv)$ are Gaussian process priors with identity covariance functions.

Directly solving the iLSVM problems is not easy because either the posterior constraints or the non-smooth regularization function $\mathcal{R}^c$ is hard to deal with. Thus, we resort to convex duality theory, which will be useful for developing approximate inference algorithms. 
We can either solve the constrained form (E.q. (\ref{eq:iLSVMconstrained})) using Lagrangian duality theory~\citep{Ito:08} or solve the unconstrained form (E.q. (\ref{eq:iLSVMunconstrained})) using Fenchel duality theory. Here, we take the second approach. In this case, the linear operator is the expectation operator, denoted by $E:~\mathcal{P}_{\textrm{prob}} \to \mathbb{R}^{|\iTrain| \times L}$ and the element of $Eq$ evaluated at $y$ for the $n$th example is
\begin{eqnarray}
Eq(n, y) \defEq  \Delta f\big(y, \xv_n; q(\Zv, \etav, \Wv)\big) = \ep_{q(\Zv, \etav, \Wv)}\big[ \etav^\top \Delta \gv_n(y, \Zv)\big],
\end{eqnarray}
where $\Delta \gv_n(y, \Zv) \defEq \gv(y_n, \xv_n, \zv) - \gv(y, \xv_n, \zv)$. 
Then, let $g_1:~\mathbb{R}^{L} \to \mathbb{R}$ be a function defined in the same form as in Eq.~(\ref{eq:gFunc}). We have $$\mathcal{R}_h^c\big(q(\Zv, \etav, \Wv)\big) = \sum_{n \in \iTrain} g_1\big( \ell_n^\Delta - Eq(n) \big),$$ where $Eq(n) \defEq (Eq(n, 1), \cdots, Eq(n, L))$ and $\ell^\Delta_n \defEq (\ell_n^\Delta(1), \cdots, \ell_n^\Delta(L))$ are the vectors of elements evaluated for $n$th data. By the Fenchel's duality theorem and the results in Lemma~\ref{proposition:ConjugateMaxFunc}, we can derive the conjugate of the problem~(\ref{eq:iLSVMunconstrained}). The proof is deferred to Appendix C.4.
\begin{lemma}[Conjugate of iLSVM]\label{lemma:conjugateiLSVM}
For the iLSVM problem, we have that
\begin{eqnarray}
 \inf_{q(\Zv, \etav, \Wv) \in \mathcal{P}_{\mathrm{prob}}} && \KL\big(q(\Zv, \etav, \Wv) \Vert p(\Zv, \etav, \Wv, \data)\big) + \mathcal{R}^c_h\big(q(\Zv, \etav, \Wv)\big) \\
 = ~~~~~~~~ \sup_{\omegav}~~~~~ && - \log Z(\omegav | \data) + \sum_{n \in \iTrain} \sum_y \omega_n^y \ell_n^\Delta (y) - \sum_n g_1^\ast(\omegav_n), \nonumber
\end{eqnarray}
where $\omegav_n = (\omega_n^1, \cdots, \omega_n^L)$ is the subvector associated with data $n$. Moreover, The optimum distribution is the posterior distribution
\begin{eqnarray}
\hat{q}(\Zv, \etav, \Wv) = \frac{1}{Z(\hat{\omegav}  | \data )}p(\Zv, \etav, \Wv, \data)\exp\Big\{ \sum_{n \in \iTrain} \sum_y \hat{\omega}_n^y \etav^\top \Delta \gv_n(y, Z) \Big\},
\end{eqnarray}
where $Z(\hat{\omegav}  | \data )$ is the normalization factor and $\hat{\omegav}$ is the solution of the dual problem.
\end{lemma}

{\bf Testing}: to make prediction on test examples, we put both training and test data together to do regularized Bayesian inference.
For training data, we impose the above large-margin constraints because of the awareness of their true labels, while for test data, we do the inference without the large-margin constraints since we do not know their true labels. Therefore, the classifier (i.e., $q(\etav)$) is learned from the training data only, while both training and testing data influence the posterior distributions of the likelihood model $\Wv$. After inference, we make the prediction via the rule
\begin{eqnarray}\label{eq:predrule}
y^\ast \defEq \argmax_y f\big(y, \xv; q(\Zv, \etav, \Wv)\big).
\end{eqnarray}
Note that the ability to generalize to test data relies on the fact that all the data examples share $\etav$ and the IBP prior. We can also cast the problem as a transductive inference problem by imposing additional large-margin constraints on test data~\citep{Joachims:tsvm99}. However, the resulting problem will be generally harder to solve because it needs to resolve the unknown labels of testing examples. We also note that the testing is different from the standard inductive setting~\citep{Zhu:iSVM11}, where the latent features of a new data example can be approximately inferred given the training data. Our empirical study shows little difference on performance between our setting and the standard inductive setting.

\begin{figure*}
\begin{center}
\centering{
\hfill\subfigure[]{\includegraphics[width=0.23\columnwidth]{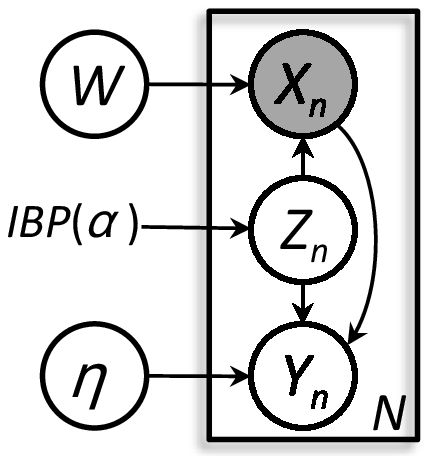}}\hfill
\subfigure[]{\includegraphics[width=.33\columnwidth]{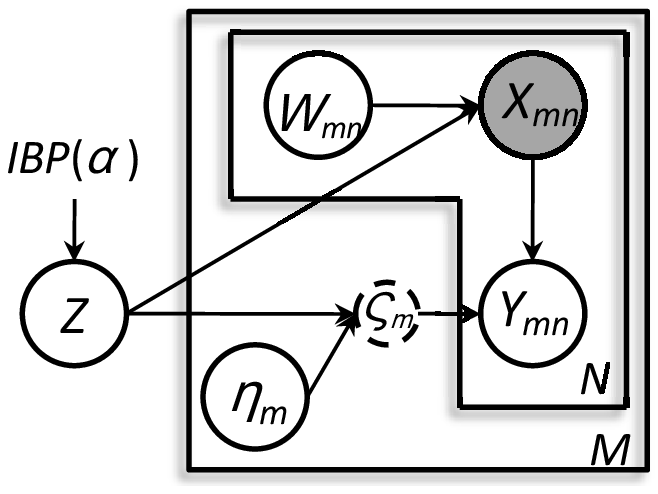}}\hfill}\vspace{-.2cm}
\caption{Graphical structures of (a) infinite latent SVM (iLSVM); and (b) multi-task infinite latent SVM (MT-iLSVM). For MT-iLSVM, the dashed nodes (i.e., $\varsigma_m$) illustrate the task relatedness but do not exist.}
\label{fig:multitaskInfLSVM}\vspace{-.3cm}
\end{center}
\end{figure*}


\subsection{Multi-Task Infinite Latent Support Vector Machines}

Different from classification, which is typically formulated as a single learning task, multi-task learning aims to improve a set of related tasks through sharing statistical strength among these tasks, which are performed jointly. Many different approaches have been developed for multi-task learning (See~\citep{Jebara:jmlr11} for a review). In particular, learning a common latent representation shared by all the related tasks has proven to be an effective way to capture task relationships~\citep{AndoTong:05,Argyriou:nips07,HalDaume:10}. Below, we present the multi-task infinite latent SVM (MT-iLSVM) for learning a common binary projection matrix $\Zv$ to capture the relationships among multiple tasks. Similar as in iLSVM, we also put the IBP prior on $\Zv$ to allow it to have an unbounded number of columns.

Suppose we have $M$ related tasks. Let $\data_m \!=\! \{ (\xv_{mn}, y_{mn}) \}_{n\in \mathcal{I}_{\textrm{tr}}^m}$ be the training data for task $m$. We consider binary classification tasks, where $\mathcal{Y}_m =\{+1, -1\}$. Extension to multi-way classification or regression can be easily done. A na\" ive way to solve this learning problem with multiple tasks is to perform the multiple tasks independently. In order to make the multiple tasks coupled and share statistical strength, MT-iLSVM introduces a latent projection matrix $\Zv$. If the latent matrix $\Zv$ is given, we define the {\it latent discriminant function} for task $m$ as
\begin{eqnarray}
f_m(\xv_{mn}, \Zv; \etav_m) \defEq ( \Zv \etav_m )^\top \xv_{mn} = \etav_m^\top (\Zv^\top \xv_{mn}),
\end{eqnarray}
where $\xv_{mn}$ is one data example in $\data_m$ and $\etav_m$ is the vector of parameters for task $m$. The dimension of $\etav_m$ is the number of columns of the latent projection matrix $\Zv$, which is unbounded in the nonparametric setting. This definition provides two views of how the $M$ tasks get related.
\begin{enumerate}[(1)]
\item If we let $\varsigma_m = \Zv \etav_m$, then $\varsigma_m$ is the actual parameter of task $m$ and all $\varsigma_m$ in different tasks are coupled by sharing the same latent matrix $\Zv$; 
\item Another view is that each task $m$ has its own parameters $\etav_m$, but all the tasks share the same latent projection matrix $\Zv$ to extract latent features $\Zv^\top \xv_{mn}$, which is a projection of the input features $\xv_{mn}$. 
\end{enumerate}
As such, our method can be viewed as a nonparametric Bayesian treatment of alternating structure optimization (ASO)~\citep{AndoTong:05}, which learns a single projection matrix with a pre-specified latent dimension. Moreover, different from~\citep{Jebara:jmlr11}, which learns a binary vector with known dimensionality to select features or kernels on $\xv$, we learn an unbounded projection matrix $\Zv$ using nonparametric Bayesian techniques.

As in iLSVM, we employ a Bayesian treatment of $\etav_m$, and view it as random variables. We assume that $\etav_m$ has a fully-factorized Gaussian prior, i.e., $\eta_{mk} \sim \mathcal{N}(0,1)$. Then, we define the effective discriminant function for task $m$ as the expectation
\begin{eqnarray}
f_m\big(\xv; q(\Zv, \etav, \Wv)\big) \defEq \ep_{q(\Zv, \etav, \Wv)}\big[ f_m(\xv, \Zv; \etav_m)\big] = \ep_{q(\Zv, \etav, \Wv)}[\Zv \etav_m ]^\top \xv,
\end{eqnarray}
where $\Wv$ is a place holder for the variables that possibly arise from other parts of the model. As in iLSVM, since we are taking expectation, the variables which do not appear in the feature map (i.e., $\Wv$) will be marginalized out. Then, the prediction rule for task $m$ is naturally $y_m^\ast \defEq \sign f_m(\xv)$.
Similarly, we perform regularized Bayesian inference by defining: $$U^{MT}(\xiv) \defEq C  \sum_{m, n \in \mathcal{I}^m_{\textrm{tr}}} \xi_{mn}$$ and imposing the following constraints:
\begin{eqnarray}\label{eq:MTconstraints}
\mathcal{P}_{\mathrm{post}}^{MT}(\xiv) \defEq \left\{ q(\Zv, \etav, \Wv)~ \begin{array}{|cl}
\forall m,~\forall n \in \mathcal{I}^m_{\textrm{tr}}: & ~ y_{mn} \ep_{q(\Zv, \etav, \Wv)}[ \Zv \etav_m ]^\top \xv_{mn} \geq 1 - \xi_{mn} \\
{} &~ \xi_{mn} \geq 0
\end{array}\right\}.
\end{eqnarray}
Finally, as in iLSVM we may also be interested in explaining observed data $\xv$. Therefore, we relate $\Zv$ to the observed data $\xv$ by defining a likelihood model:
\begin{eqnarray}
p\big(\xv_{mn} | \wv_{mn}, \Zv, \lambda_{mn}^2\big) = \mathcal{N}\big(\xv_{mn}| \Zv \wv_{mn}, \lambda_{mn}^2 I\big),
\end{eqnarray}
where $\wv_{mn}$ is a vector. We assume $\Wv$ has an independent prior $\pi(\Wv) = \prod_{mn} \mathcal{N}(\wv_{mn}|0, \sigma_{m0}^2 I)$. Fig.~\ref{fig:multitaskInfLSVM} (b) illustrates the graphical structure of MT-iLSVM.

For training, we can derive the similar convex conjugate as in the case of iLSVM. Similar as in iLSVM, minimizing $U^{MT}(\xiv)$ is equivalent to minimizing the hinge-loss $\mathcal{R}_h^{MT}$ of the multiple binary prediction rules, where
\begin{eqnarray}
\mathcal{R}_h^{MT}\big(q(\Zv, \etav, \Wv)\big) = C \sum_{m, n \in \mathcal{I}^m_{\textrm{tr}}} \max\big(0, 1 - y_{mn} \ep_{q(\Zv, \etav, \Wv)}[ \Zv \etav_m ]^\top \xv_{mn}\big).
\end{eqnarray}
Thus, the RegBayes problem of MT-iLSVM can be equivalently written as
\begin{eqnarray}\label{eq:MTconstrained}
\inf_{q(\Zv, \etav, \Wv)} \KL\big(q(\Zv, \etav, \Wv) \Vert p(\Zv, \etav, \Wv, \data)\big) + \mathcal{R}^{MT}_h\big(q(\Zv, \etav, \Wv)\big).
\end{eqnarray}
Then, by the Fenchel's duality theorem and Lemma~\ref{proposition:ConjugateBinaryFunc}, we can derive the conjugate of MT-iLSVM. The proof is deferred to Appendix C.5.
\begin{lemma}[Conjugate of MT-iLSVM]\label{lemma:conjugateMTiLSVM}
For the MT-iLSVM problem, we have that
\begin{eqnarray}
 \inf_{q(\Zv, \etav, \Wv) \in \mathcal{P}_{\mathrm{prob}}} && \KL(q(\Zv, \etav, \Wv) \Vert p(\Zv, \etav, \Wv, \data)) + \mathcal{R}^{MT}_h(q(\Zv, \etav, \Wv)) \\
 =~~~~~~~~ \sup_{\omegav} ~~~~~&& - \log Z^\prime(\omegav  | \data ) + \sum_{m,n} \omega_{mn} - \sum_{m,n} g_0^\ast(\omega_{mn}). \nonumber
\end{eqnarray}
Moreover, The optimum distribution is the posterior distribution
\begin{eqnarray}
\hat{q}(\Zv, \etav, \Wv) = \frac{1}{Z^\prime(\hat{\omegav}  | \data )}p(\Zv, \etav, \Wv, \data)\exp\Big\{ \sum_{m,n} y_{mn} \hat{\omega}_{mn} (\Zv\etav_m)^\top \xv_{mn} \Big\},
\end{eqnarray}
where $Z^\prime(\hat{\omegav}  | \data )$ is the normalization factor and $\hat{\omegav}$ is the solution of the dual problem.
\end{lemma}

For testing, we use the same strategy as in iLSVM to do Bayesian inference on both training and test data. The difference is that training data are subject to large-margin constraints, while test data are not. Similarly, the hyper-parameters $\sigma_{m0}^2$ and $\lambda_{mn}^2$ can be set a priori or estimated from data (See Appendix D.1 for details).

\subsection{Inference with Truncated Mean-Field Constraints}\label{sec:inference}

Now we discuss how to perform regularized Bayesian inference with the large-margin constraints for both iLSVM and MT-iLSVM. From the primal-dual formulations, it is obvious that there are basically two methods to perform the regularized Bayesian inference. One is to directly solve the primal problem for the posterior distribution $q(\Zv, \etav, \Wv)$, and the other is to first solve the dual problem for the optimum $\hat{\omegav}$ and then infer the posterior distribution. However, both the primal and dual problems are intractable for iLSVM and MT-iLSVM. The intrinsic hardness is due to the mutual dependency among the latent variables in the desired posterior distribution. Therefore, a natural approximation method is the mean field~\citep{Jordan:99}, which breaks the mutual dependency by assuming that $q$ is of some factorization form. This method approximates the original problems by imposing additional constraints. An alternative method is to apply approximate methods (e.g., MCMC sampling) to infer the true posterior distributions derived via convex conjugates as above, and iteratively estimate the dual parameters using approximate statistics (e.g., feature expectations estimated using samples)~\citep{schofield2006fitting}. Below, we use MT-iLSVM as an example to illustrate the idea of the first strategy. A full discussion on the second strategy is beyond the scope of this paper.
For iLSVM, the similar procedure applies and we defer its details to Appendix D.2.

To make the problem easier to solve, we use the stick-breaking representation of IBP, which includes the auxiliary variable $\nuv$, and infer the augmented posterior $q(\nuv, \Wv, \Zv, \etav)$. The joint model distribution is now $q(\nuv, \Wv, \Zv, \etav, \data)$. Furthermore, we impose the truncated mean-field constraint that
\begin{eqnarray}\label{eq:MF-MTiLSVM}
q(\nuv, \Wv, \Zv, \etav) = q(\etav) \prod_{k=1}^K \Big( q(\nu_k | \gammav_k) \prod_{d=1}^D q(z_{dk}|\psi_{dk}) \Big) \prod_{mn} q\Big(\wv_{mn}|\Phi_{mn}, \sigma_{mn}^2 I\Big),
\end{eqnarray}
where $K$ is the truncation level, and we assume that $$q(\nu_k | \gammav_k) = \B(\gamma_{k1}, \gamma_{k2}),$$
$$q(z_{dk}|\psi_{dk}) = \Ber(\psi_{dk}),$$
$$q(\wv_{mn}|\Phi_{mn}, \sigma_{mn}^2 I)=\mathcal{N}(\wv_{mn}|\Phi_{mn}, \sigma_{mn}^2 I).$$
Then, we can use the duality theory\footnote{Lagrangian duality~\citep{Ito:08} was used in~\citep{Zhu:iLSVM11} to solve the constrained variational formulations, which is closely related to Fenchel duality~\citep{Magnanti:74} and leads to the same solutions for iLSVM and MT-iLSVM.} to solve the RegBayes problem by alternating between two substeps, as outlined in Algorithm 1 and detailed below.

\begin{algorithm}[tb]
   \caption{Inference Algorithm for Infinite Latent SVMs}
   \label{alg:coord_descent_regress}
\begin{algorithmic}[1]
   \STATE {\bfseries Input:} corpus $\mathcal{D}$ and constants $(\alpha, C)$.
   \STATE {\bfseries Output:} posterior distribution $q(\nuv, \Zv, \etav, \Wv)$.
   \REPEAT
        \STATE infer $q(\nuv), q(\Wv)$ and $q(\Zv)$ with $q(\etav)$ and $\omegav$ given;
        \STATE infer $q(\etav)$ and solve for $\omegav$ with $q(\Zv)$ given.
   \UNTIL{convergence}
\end{algorithmic}
\end{algorithm}

{\bf Infer $q(\nuv)$, $q(\Wv)$ and $q(\Zv)$:} Since $q(\nuv)$ and $q(\Wv)$ are not directly involved in the posterior constraints, we can solve for them by using standard Bayesian inference, i.e., minimizing a KL-divergence. Specifically, for $q(\Wv)$, since the prior is also normal, we can easily derive the update rules for $\Phi_{mn}$ and $\sigma_{mn}^2$. For $q(\nuv)$, we have the same update rules as in~\citep{Doshi-Velez:09}. We defer the details to Appendix D.1.

For $q(\Zv)$, it is directly involved in the posterior constraints. So, we need to solve it together with $q(\etav)$ using conjugate theory. However, this is intractable. Here, we adopt an alternating strategy that first infers $q(\Zv)$ with $q(\etav)$ and dual parameters $\omegav$ fixed, and then infers $q(\etav)$ and solves for $\omegav$. Specifically, since the large-margin constraints are linear of $q(\Zv)$, we can get the mean-field update equation as $$\psi_{dk} = \frac{1}{1 + e^{-\vartheta_{dk}}},$$ where
\setlength\arraycolsep{1pt}\begin{eqnarray}\label{eq:updateZ_mainPaper}
\vartheta_{dk} = && \sum_{j=1}^k \ep_q[ \log v_j ] - \mathcal{L}_k^\nu - \sum_{mn} \frac{1}{2\lambda_{mn}^2}\Big( (K \sigma_{mn}^2 + (\phi_{mn}^k)^2)   \\
&&  - 2 x_{mn}^d \phi_{mn}^k + 2 \sum_{j\neq k} \phi_{mn}^j\phi_{mn}^k \psi_{dj} \Big) + \sum_{m, n\in \mathcal{I}_{\textrm{tr}}^m} y_{mn} \ep_q[\eta_{mk}] x_{mn}^d, \nonumber
\end{eqnarray}
and $\mathcal{L}_k^\nu$ is an lower bound of $\ep_q[ \log (1 - \prod_{j=1}^k v_j) ]$ (See Appendix D.1 for details). The last term of $\vartheta_{dk}$ is due to the large-margin posterior constraints as defined in Eq.~(\ref{eq:MTconstraints}). Therefore, from this equation we can see how the large-margin constraints regularize the procedure of inferring the latent matrix $\Zv$.

{\bf Infer $q(\etav)$ and solve for $\omegav$:} Now, we can apply the convex conjugate theory and show that the optimum posterior distribution of $\etav$ is
$$q(\etav) = \prod_m q(\etav_m),~\textrm{where}~q(\etav_m) \propto \pi(\etav_m) \exp\lbrace \etav_{m}^\top \muv_m \rbrace,$$
and $\muv_m = \sum_{n \in \mathcal{I}^m_{\textrm{tr}}} y_{mn} \omega_{mn} (\psiv^\top \xv_{mn})$. Here, we assume $\pi(\etav_m)$ is standard normal. Then, we have $q(\etav_m) = \mathcal{N}(\etav_m|\muv_m , I)$ and the optimum dual parameters can be obtained by solving the following $M$ independent dual problems
\begin{eqnarray}
\sup_{ \omegav_m }~&& -\frac{1}{2} \muv_m^\top \muv_m + \sum_{n \in \mathcal{I}^m_{\textrm{tr}}} \omega_{mn} \\
\forall n \in \mathcal{I}^m_{\textrm{tr}}, ~ \st:~&& 0 \leq \omega_{mn} \leq C,  \nonumber
\end{eqnarray}
where the constraints are from the conjugate function $g_0^\ast$ in Lemma~\ref{lemma:conjugateMTiLSVM}. These dual problems (or their primal forms) can be efficiently solved with a binary SVM solver, such as SVM-light or LibSVM.

\section{Experiments}\label{sec:experiment}

We present empirical results for both classification and multi-task learning. Our results appear to demonstrate the merits inherited from both Bayesian nonparametrics and large-margin learning.

\subsection{Multi-way Classification}

We evaluate the infinite latent SVM (iLSVM) for classification on the real TRECVID2003 and Flickr image datasets, which have been extensively evaluated in the context of learning finite latent feature models~\citep{Chen:nips10}. TRECVID2003 consists of 1078 video key-frames that belong to 5 categories, including {\it Airplane scene}, {\it Basketball scene}, {\it Weather news}, {\it Baseball scene}, and {\it Hockey scene}. Each data example has two types of features -- 1894-dimension binary vector of text features and 165-dimension HSV color histogram. The Flickr image dataset consists of 3411 natural scene images about 13 types of animals, including {\it squirrel, cow, cat, zebra, tiger, lion, elephant, whales, rabbit, snake, antlers, hawk and wolf}, downloaded from the Flickr website\footnote{http://www.flickr.com/}. Also, each example has two types of features, including 500-dimension SIFT bag-of-words and 634-dimension real-valued features (e.g., color histogram, edge direction histogram, and block-wise color moments). Here, we consider the real-valued features only by defining Gaussian likelihood distributions for $\xv$; and we define the discriminant function using latent features only as in Eq.~(\ref{eq:latent-func-ilsvm}). We follow the same training/testing splits as in~\citep{Chen:nips10}.

\begin{table}[t]
\vskip 0.1in
\begin{center}
\begin{tabular}{|c|cc|cc|}
\hline
{} &  \multicolumn{2}{c|}{TRECVID2003} & \multicolumn{2}{c|}{Flickr} \\
Model & Accuracy & F1 score & Accuracy & F1 score \\
\hline
EFH+SVM & 0.565 $\pm$ 0.0 & 0.427 $\pm$ 0.0 & 0.476 $\pm$ 0.0 & 0.461 $\pm$ 0.0 \\
MMH     & {\bf 0.566} $\pm$ 0.0 & 0.430 $\pm$ 0.0 & {\bf 0.538} $\pm$ 0.0 & {\bf 0.512} $\pm$ 0.0 \\
\hline
IBP+SVM & 0.553 $\pm$ 0.013 & 0.397 $\pm$ 0.030 & 0.500 $\pm$ 0.004 & 0.477 $\pm$ 0.009\\
iLSVM   & 0.563 $\pm$ 0.010 & {\bf 0.448} $\pm$ 0.011 & 0.533 $\pm$ 0.005 & 0.510 $\pm$ 0.010 \\
\hline
\end{tabular}
\end{center}
\caption{ Classification accuracy and F1 scores on the TRECVID2003 and Flickr image datasets (Note: MMH and EFH have zero std because of their deterministic initialization).}
\label{table:Classification}\vspace{-.1cm}
\end{table}

\begin{figure*}
\begin{center}
\includegraphics[width=0.85\columnwidth,height=0.3\columnwidth]{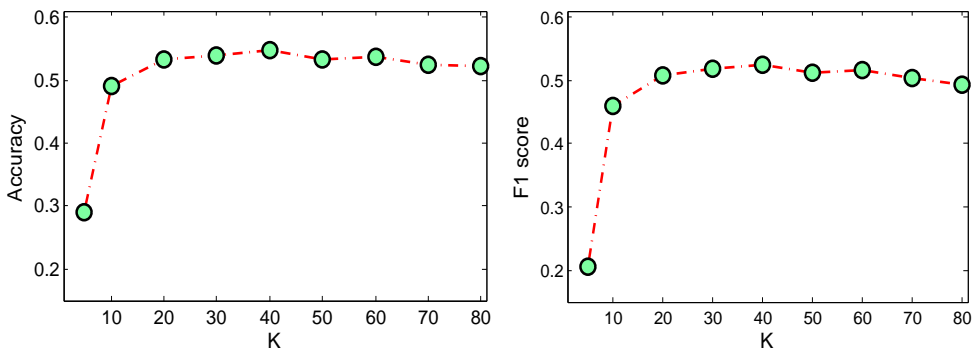}
\caption{Accuracy and F1 score of MMH on the Flickr dataset with different numbers of latent features.}
\label{fig:FlickrK}\vspace{-.2cm}
\end{center}
\end{figure*}

\begin{figure*}[t]
\begin{center}
\includegraphics[width=0.9\columnwidth,height=0.45\columnwidth]{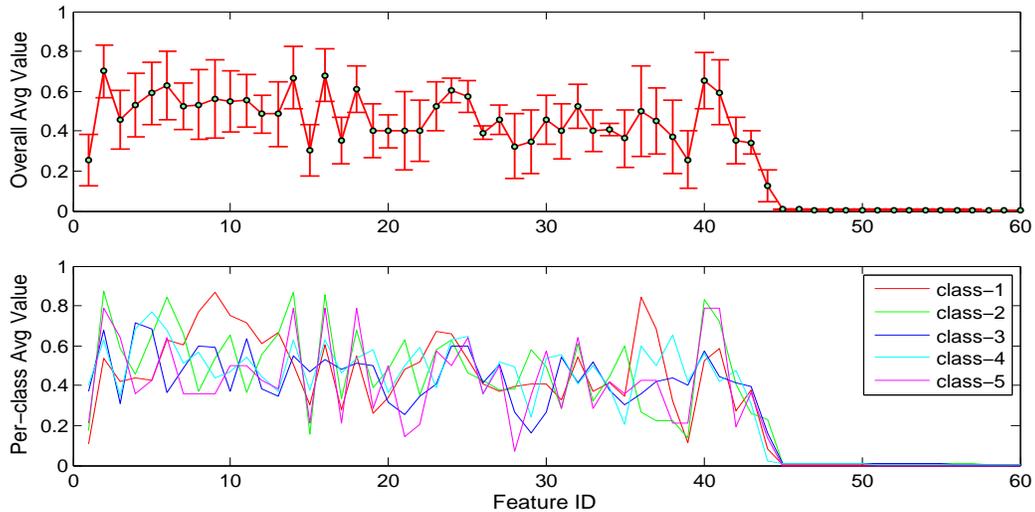}
\caption{(Up) the overall average values of the latent features with standard deviation over different classes; and (Bottom) the per-class average values of latent features learned by iLSVM on the TRECVID dataset.}\vspace{-.4cm}
\label{fig:TrecPattern}
\end{center}
\end{figure*}

\begin{figure*}
\begin{center}
\includegraphics[width=0.9\columnwidth,height=0.23\columnwidth]{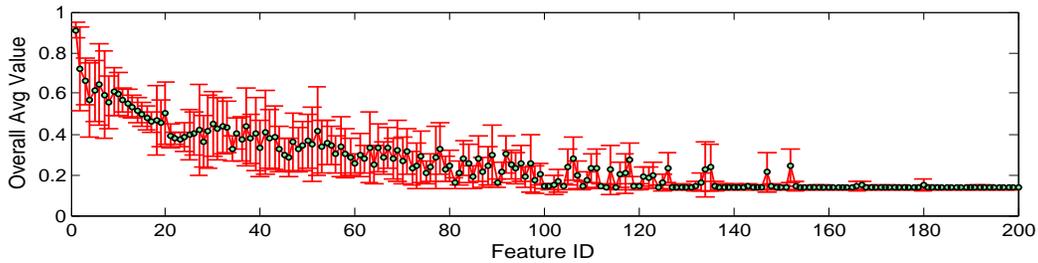}
\caption{The overall average values of the latent features with standard deviation over different classes on the Flickr dataset.}\vspace{-.4cm}
\label{fig:FlickrZ}
\end{center}
\end{figure*}

\begin{figure*}[t]
\begin{center}
\setlength{\tabcolsep}{2.9pt}
\scalebox{0.75}{
 \begin{tabular}{ c c c c c c }
 \hline \hline
{\multirow{5}{*}{F1}} & {\multirow{5}{*}{\includegraphics[width=.18\columnwidth,height=.15\columnwidth]{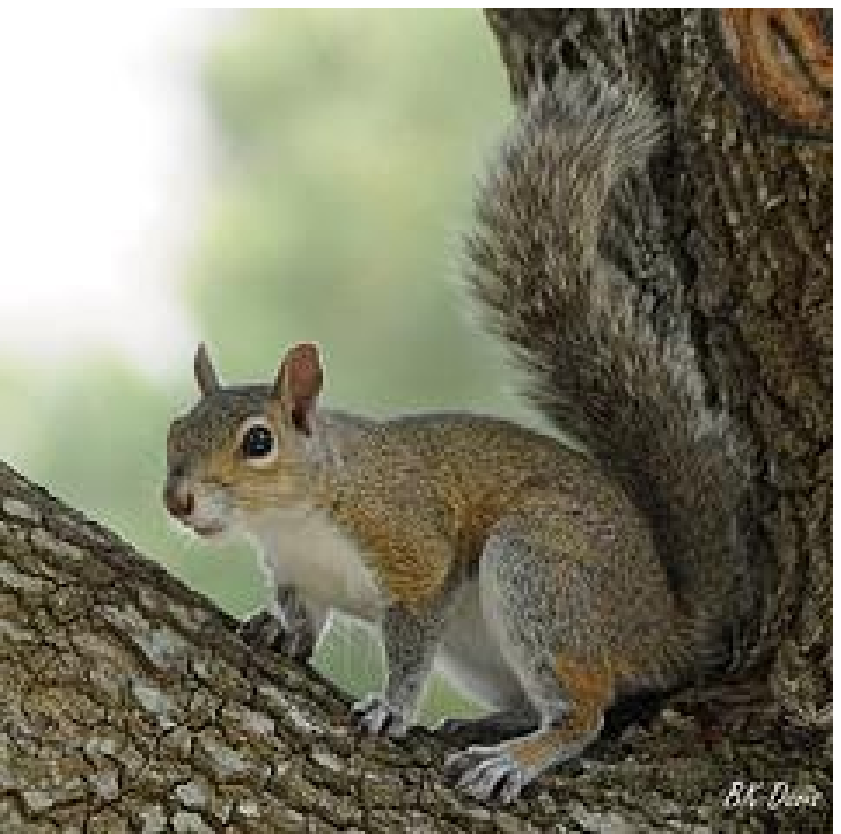}}} & {\multirow{5}{*}{\includegraphics[width=.18\columnwidth,height=.15\columnwidth]{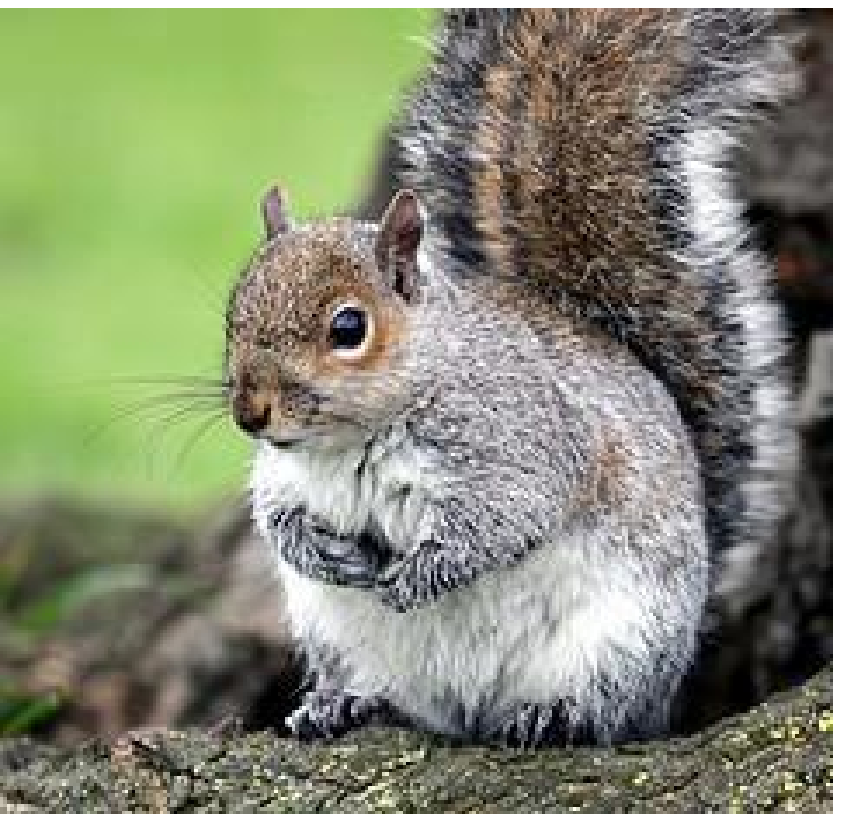}}} & {\multirow{5}{*}{\includegraphics[width=.18\columnwidth,height=.15\columnwidth]{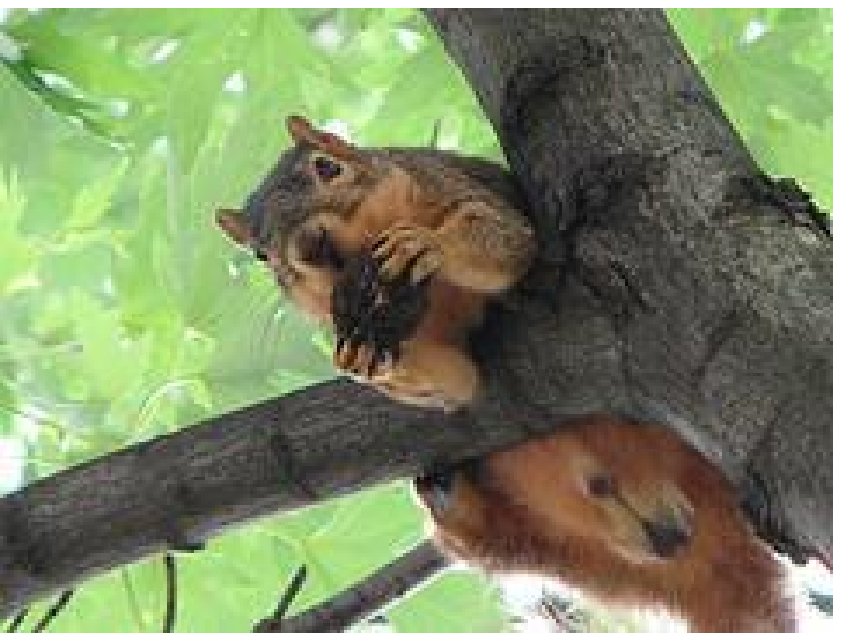}}} & {\multirow{5}{*}{\includegraphics[width=.18\columnwidth,height=.15\columnwidth]{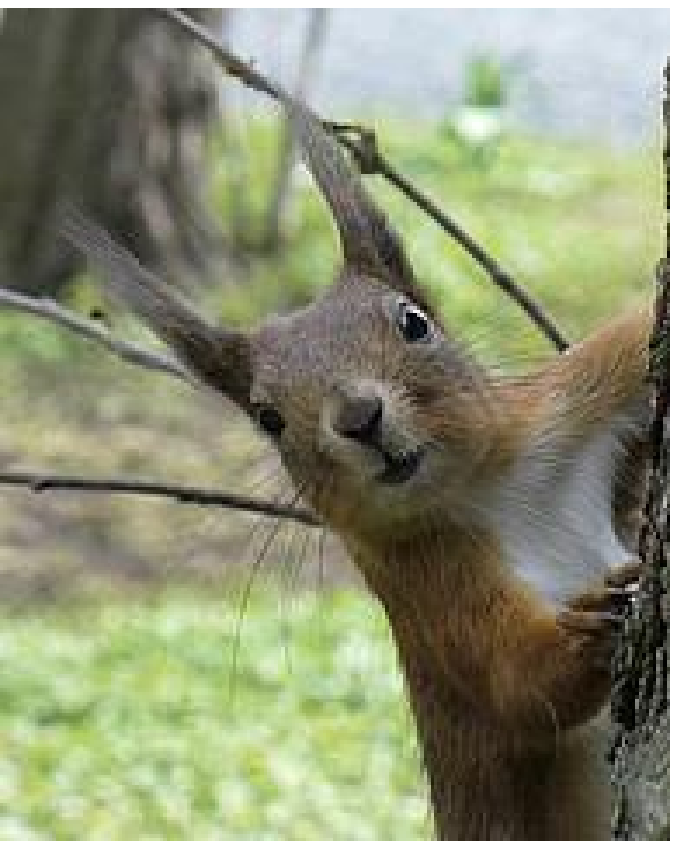}}} & {\multirow{5}{*}{\includegraphics[width=.18\columnwidth,height=.15\columnwidth]{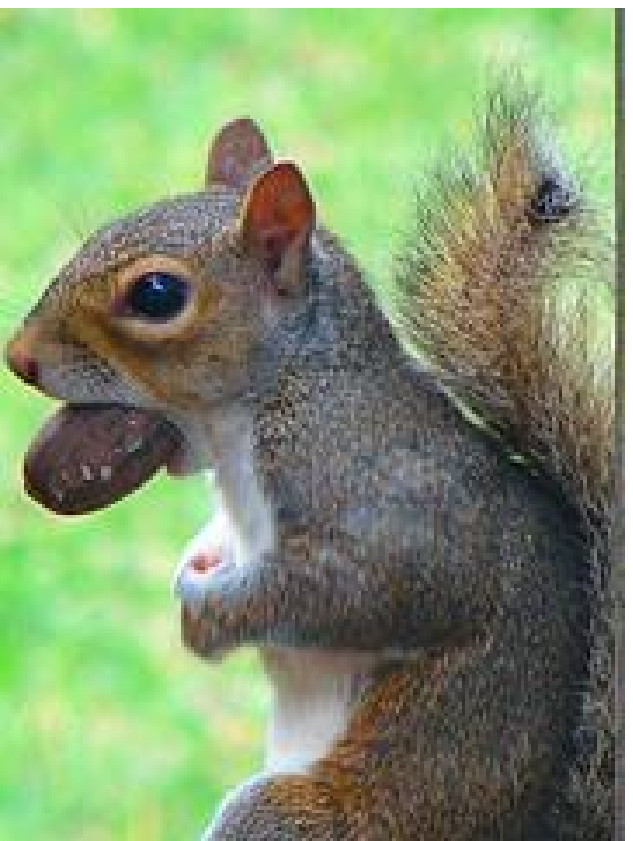}}} \\
        &&&&& \\
        &&&&& \\
        &&&&& \\
        &&&&& \\
         \hline
{\multirow{5}{*}{F2}} & {\multirow{5}{*}{\includegraphics[width=.18\columnwidth,height=.15\columnwidth]{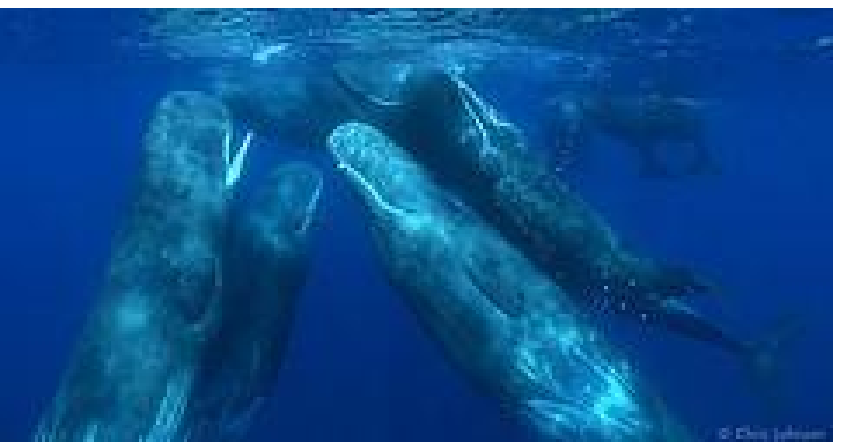}}} & {\multirow{5}{*}{\includegraphics[width=.18\columnwidth,height=.15\columnwidth]{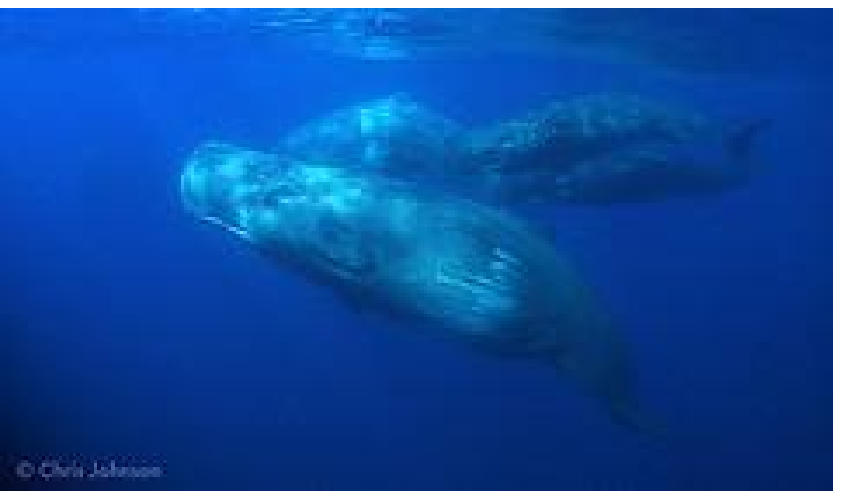}}} & {\multirow{5}{*}{\includegraphics[width=.18\columnwidth,height=.15\columnwidth]{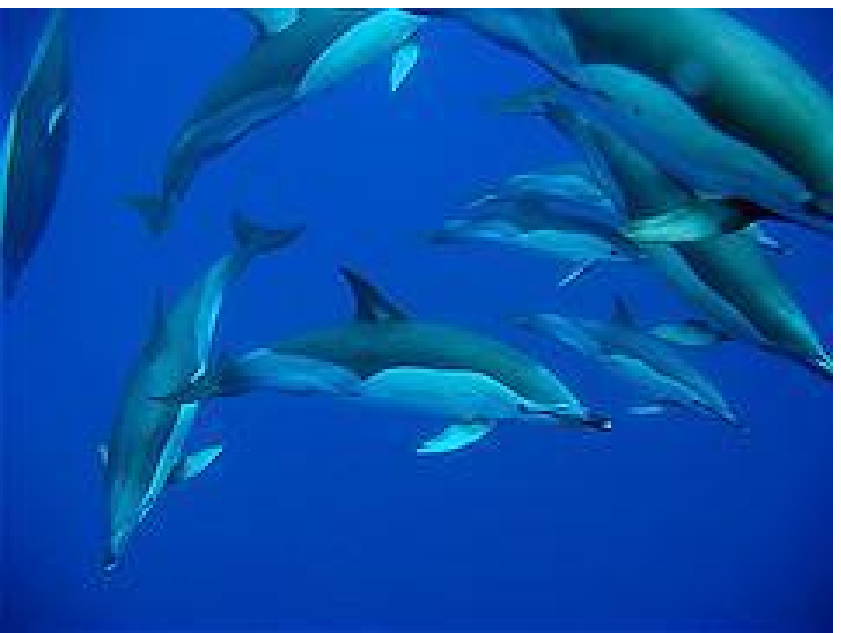}}} & {\multirow{5}{*}{\includegraphics[width=.18\columnwidth,height=.15\columnwidth]{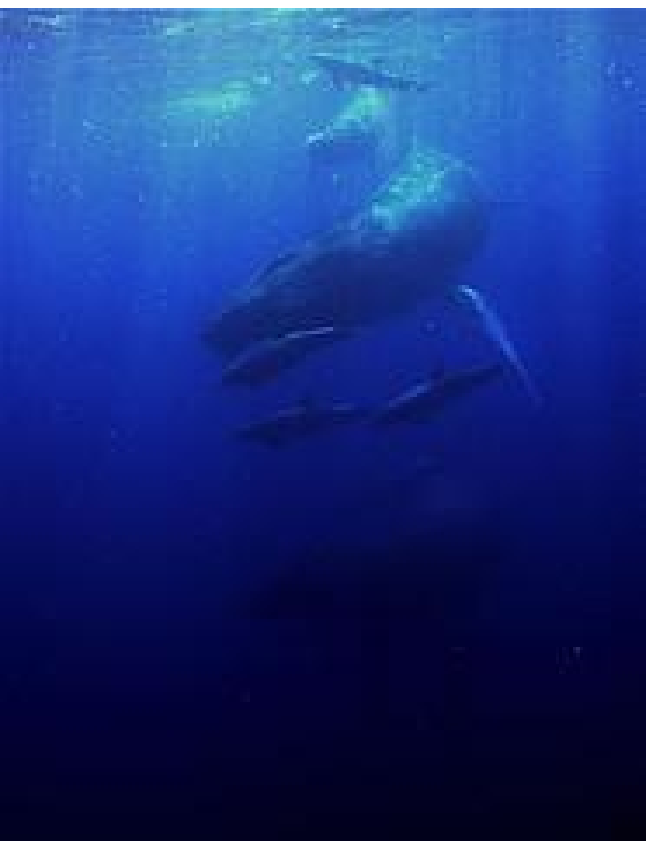}}} & {\multirow{5}{*}{\includegraphics[width=.18\columnwidth,height=.15\columnwidth]{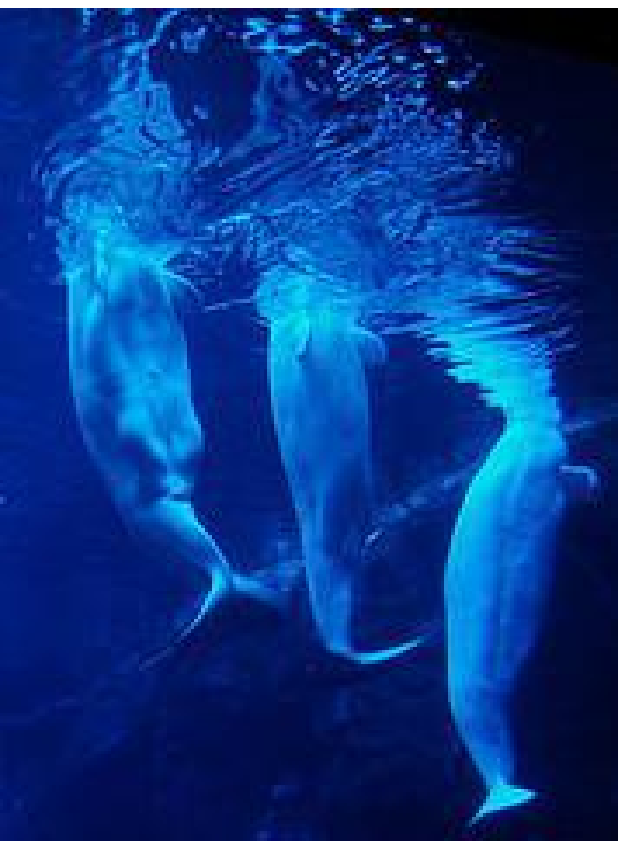}}} \\
        &&&&& \\
        &&&&& \\
        &&&&& \\
        &&&&& \\
         \hline
{\multirow{5}{*}{F3}} & {\multirow{5}{*}{\includegraphics[width=.18\columnwidth,height=.15\columnwidth]{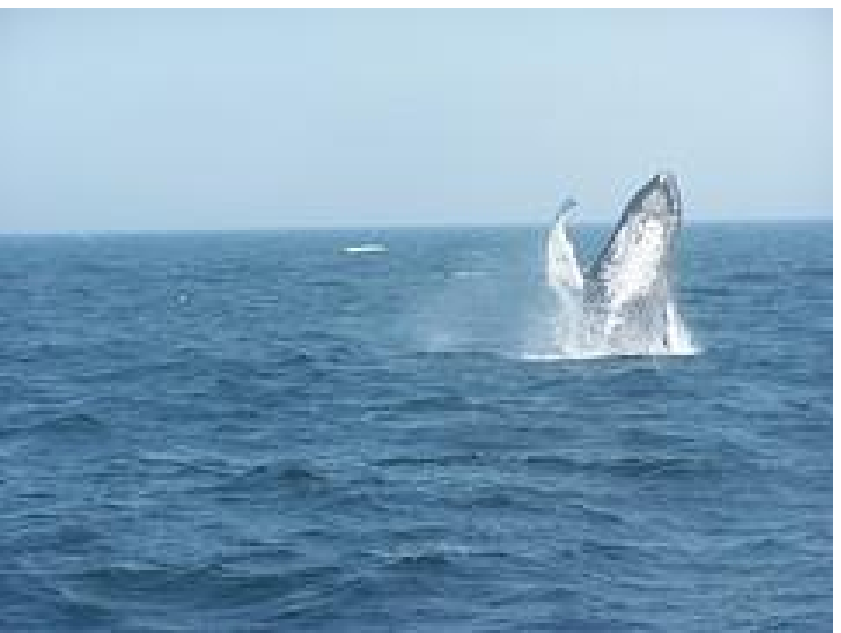}}} & {\multirow{5}{*}{\includegraphics[width=.18\columnwidth,height=.15\columnwidth]{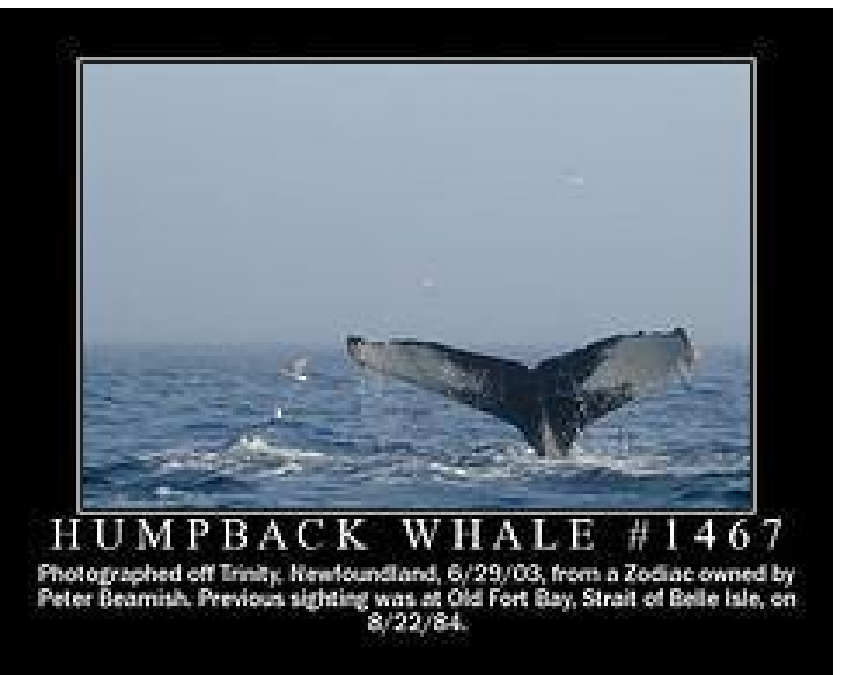}}} & {\multirow{5}{*}{\includegraphics[width=.18\columnwidth,height=.15\columnwidth]{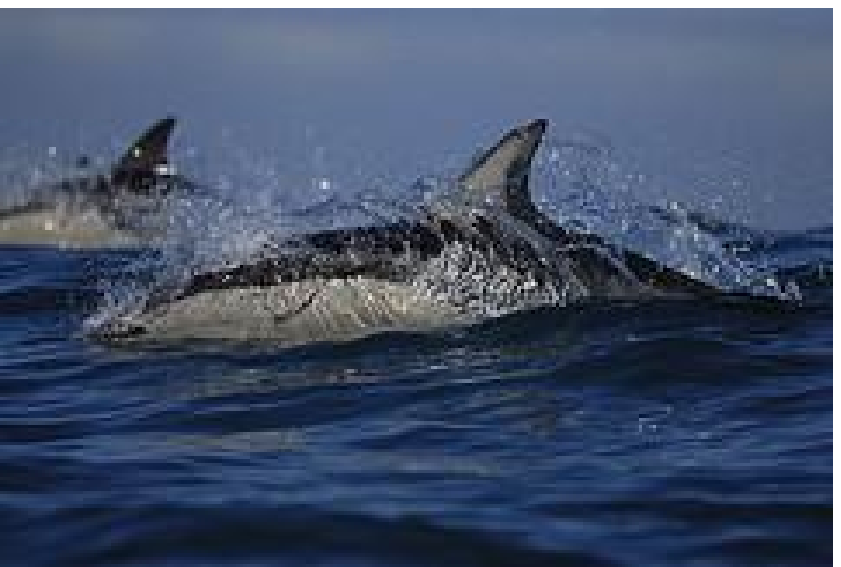}}} & {\multirow{5}{*}{\includegraphics[width=.18\columnwidth,height=.15\columnwidth]{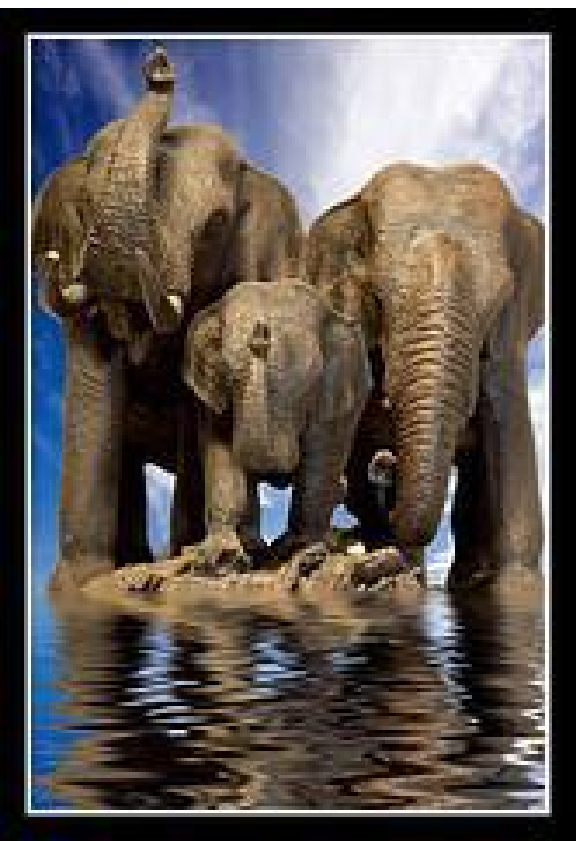}}} & {\multirow{5}{*}{\includegraphics[width=.18\columnwidth,height=.15\columnwidth]{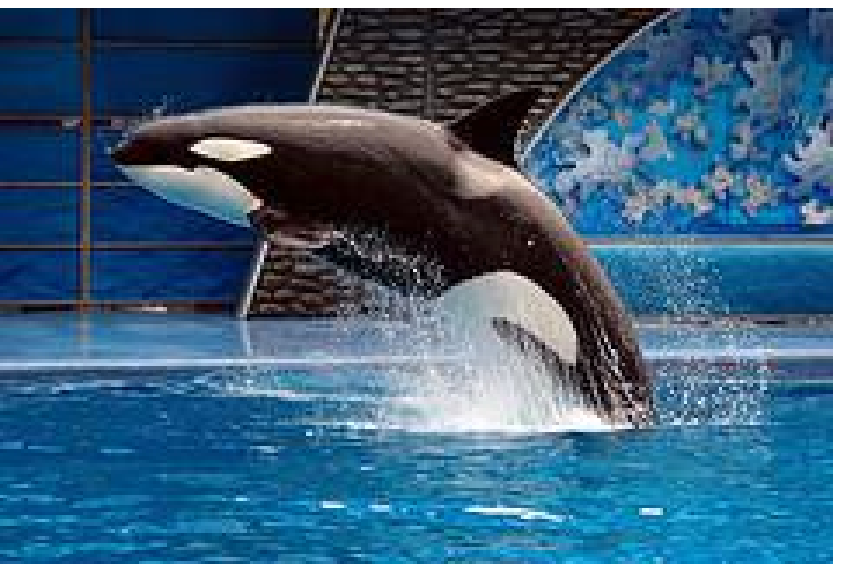}}} \\
        &&&&& \\
        &&&&& \\
        &&&&& \\
        &&&&& \\
         \hline
{\multirow{5}{*}{F4}} & {\multirow{5}{*}{\includegraphics[width=.18\columnwidth,height=.15\columnwidth]{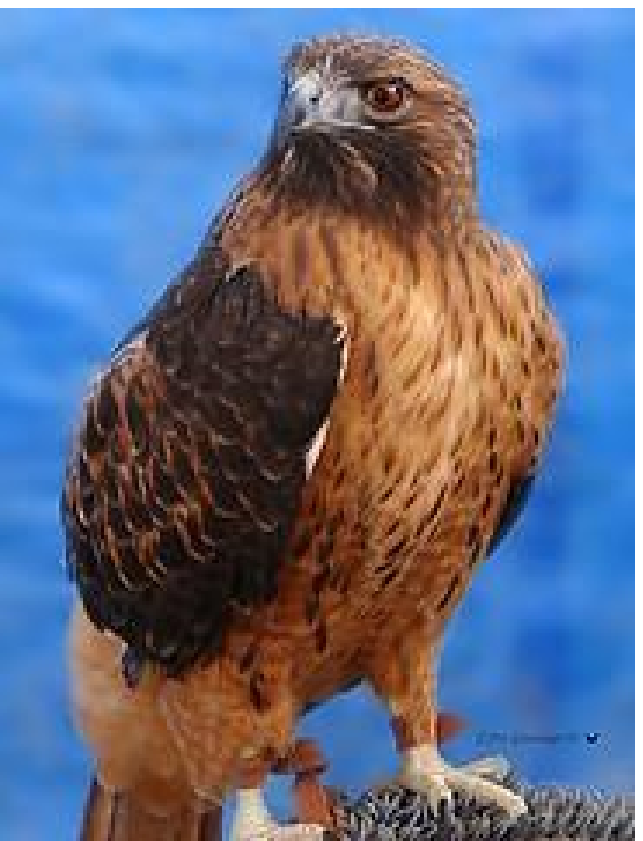}}} & {\multirow{5}{*}{\includegraphics[width=.18\columnwidth,height=.15\columnwidth]{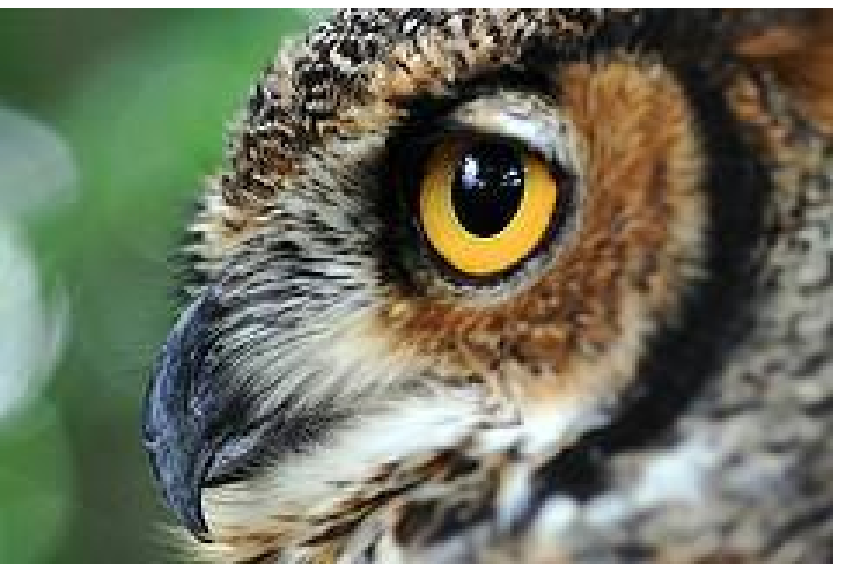}}} & {\multirow{5}{*}{\includegraphics[width=.18\columnwidth,height=.15\columnwidth]{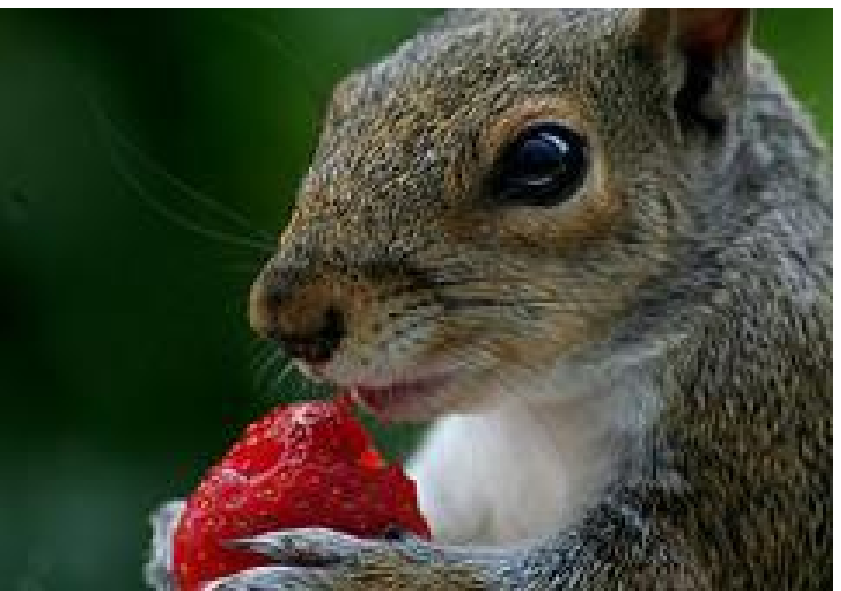}}} & {\multirow{5}{*}{\includegraphics[width=.18\columnwidth,height=.15\columnwidth]{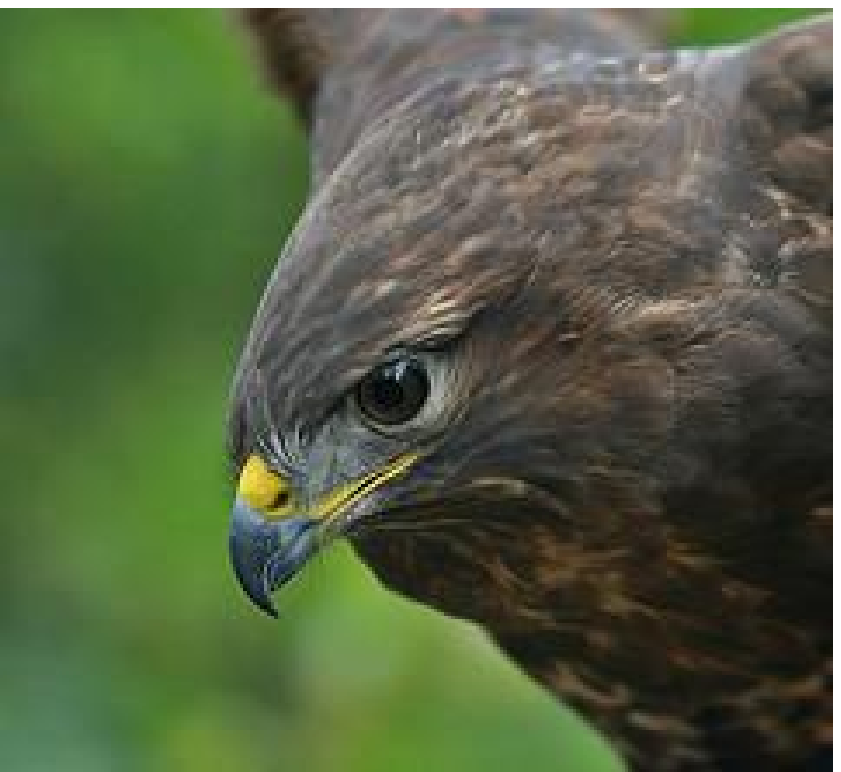}}} & {\multirow{5}{*}{\includegraphics[width=.18\columnwidth,height=.15\columnwidth]{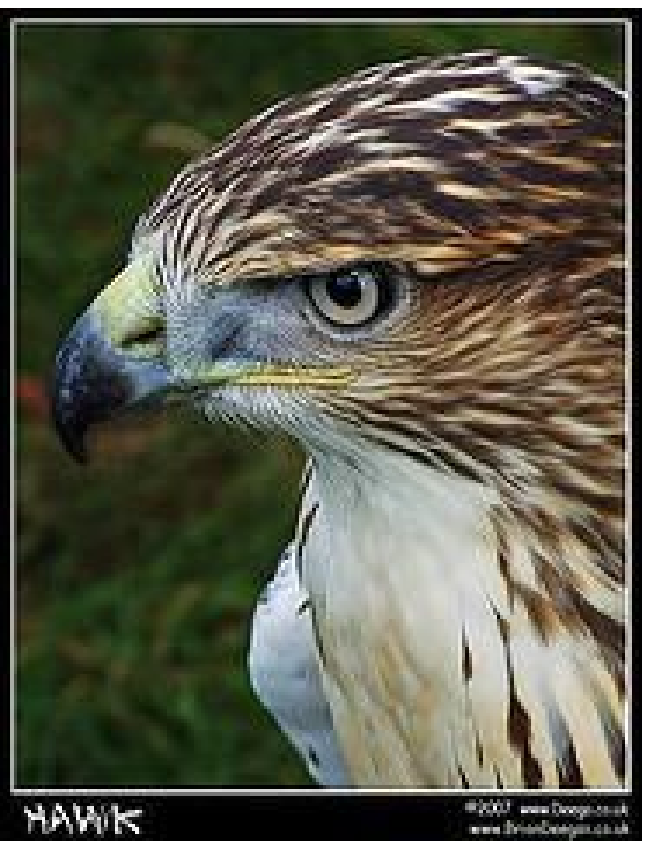}}} \\
        &&&&& \\
        &&&&& \\
        &&&&& \\
        &&&&& \\
        \hline
{\multirow{5}{*}{F5}} & {\multirow{5}{*}{\includegraphics[width=.18\columnwidth,height=.15\columnwidth]{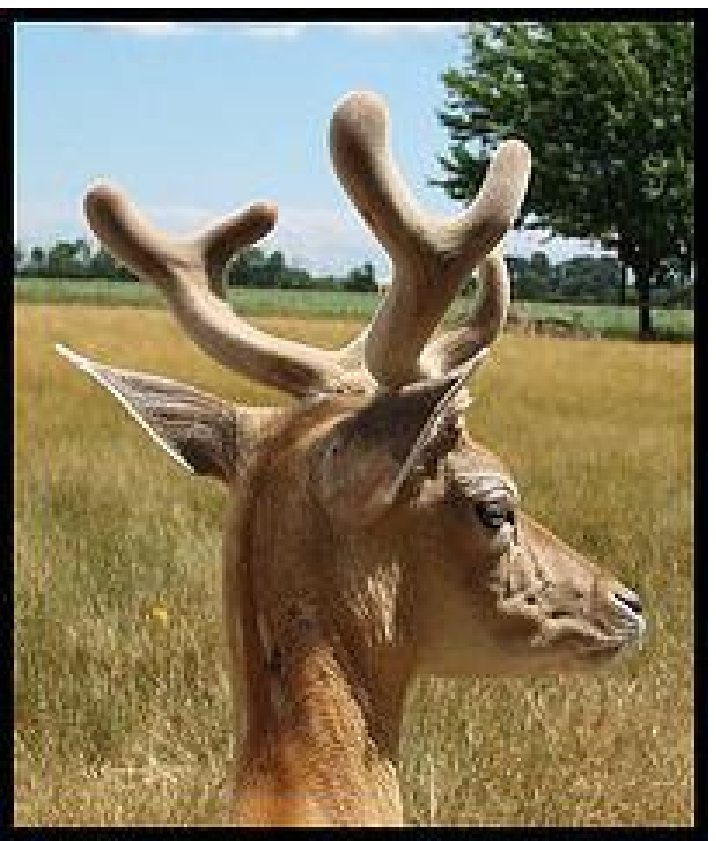}}} & {\multirow{5}{*}{\includegraphics[width=.18\columnwidth,height=.15\columnwidth]{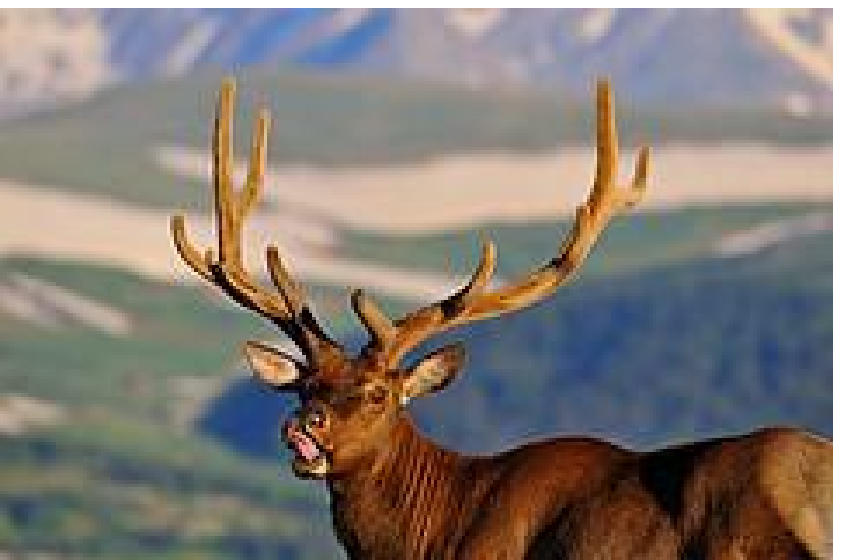}}} & {\multirow{5}{*}{\includegraphics[width=.18\columnwidth,height=.15\columnwidth]{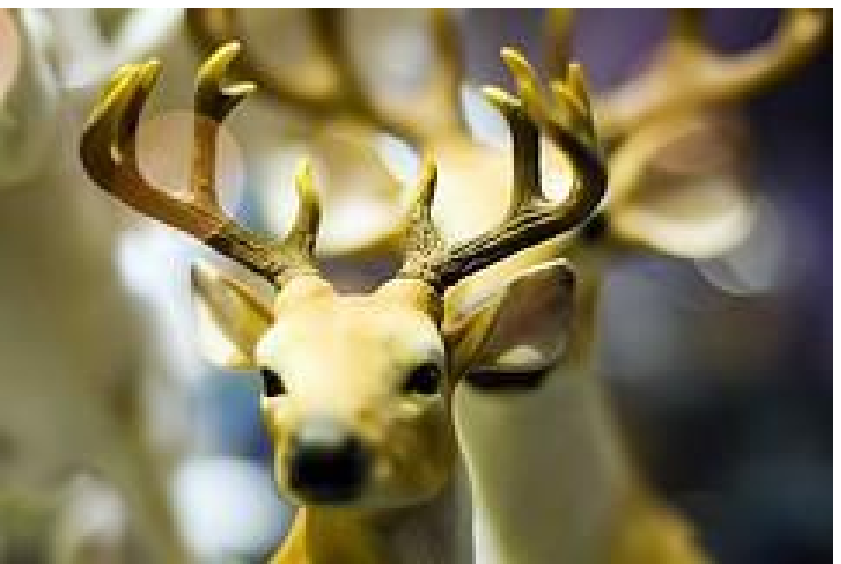}}} & {\multirow{5}{*}{\includegraphics[width=.18\columnwidth,height=.15\columnwidth]{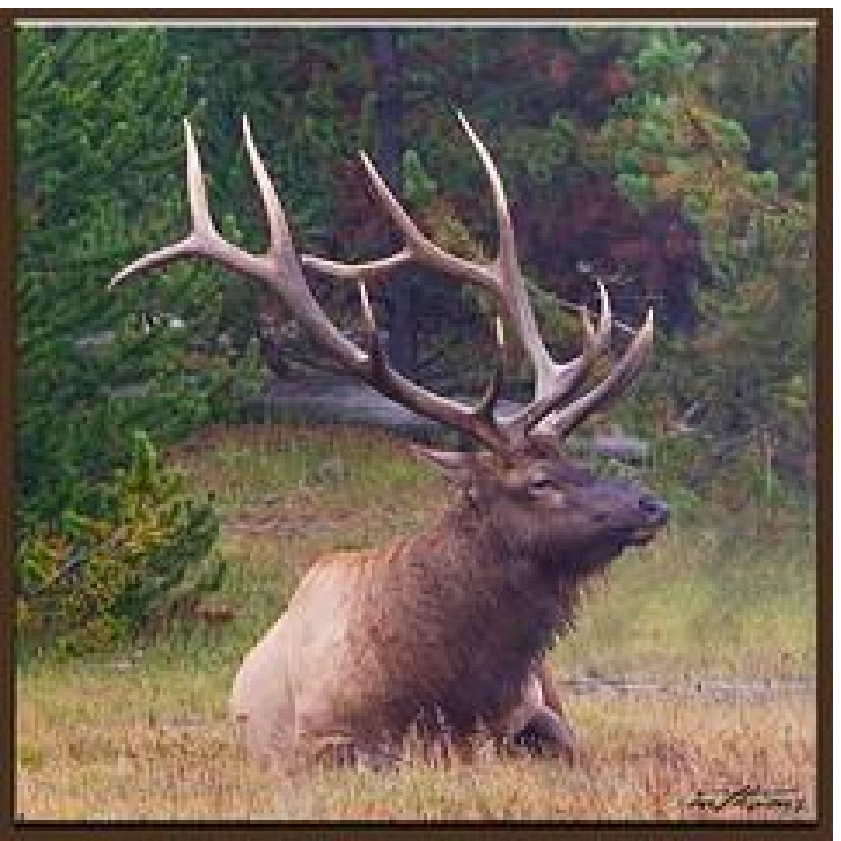}}} & {\multirow{5}{*}{\includegraphics[width=.18\columnwidth,height=.15\columnwidth]{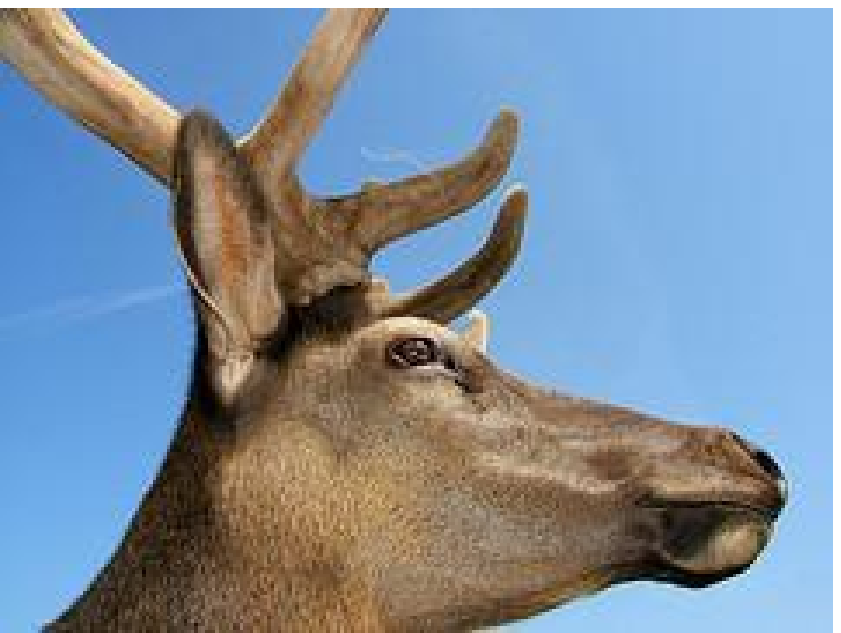}}} \\
        &&&&& \\
        &&&&& \\
        &&&&& \\
        &&&&& \\
         \hline
{\multirow{5}{*}{F6}} & {\multirow{5}{*}{\includegraphics[width=.18\columnwidth,height=.15\columnwidth]{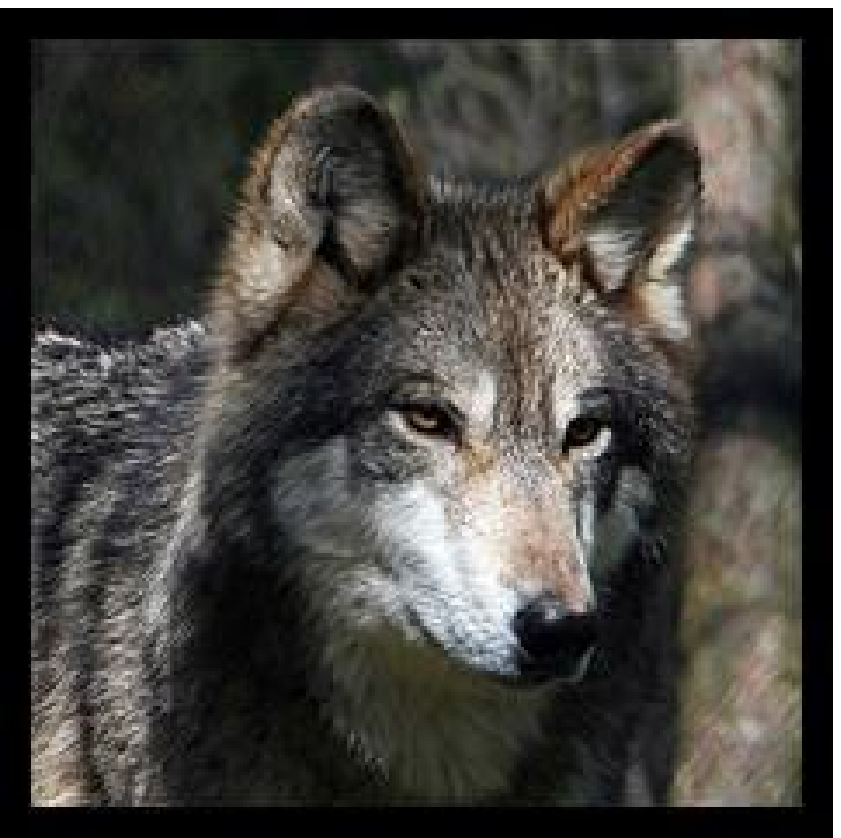}}} & {\multirow{5}{*}{\includegraphics[width=.18\columnwidth,height=.15\columnwidth]{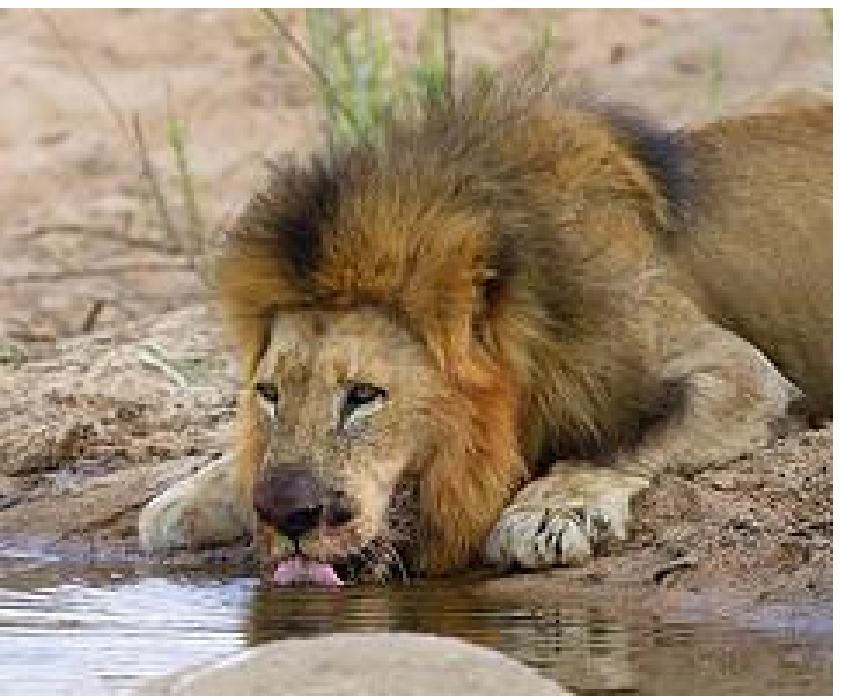}}} & {\multirow{5}{*}{\includegraphics[width=.18\columnwidth,height=.15\columnwidth]{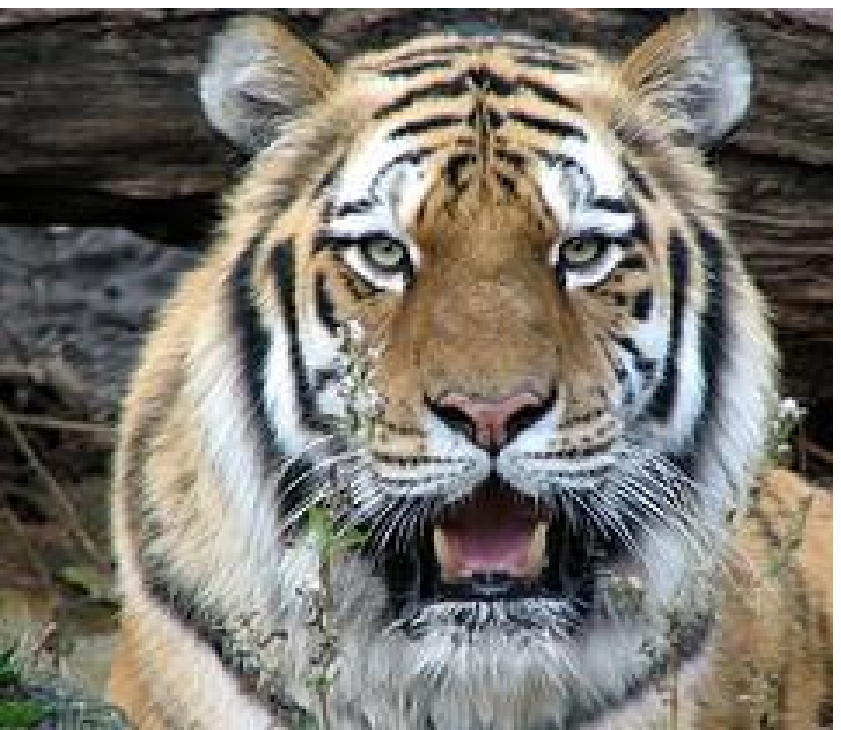}}} & {\multirow{5}{*}{\includegraphics[width=.18\columnwidth,height=.15\columnwidth]{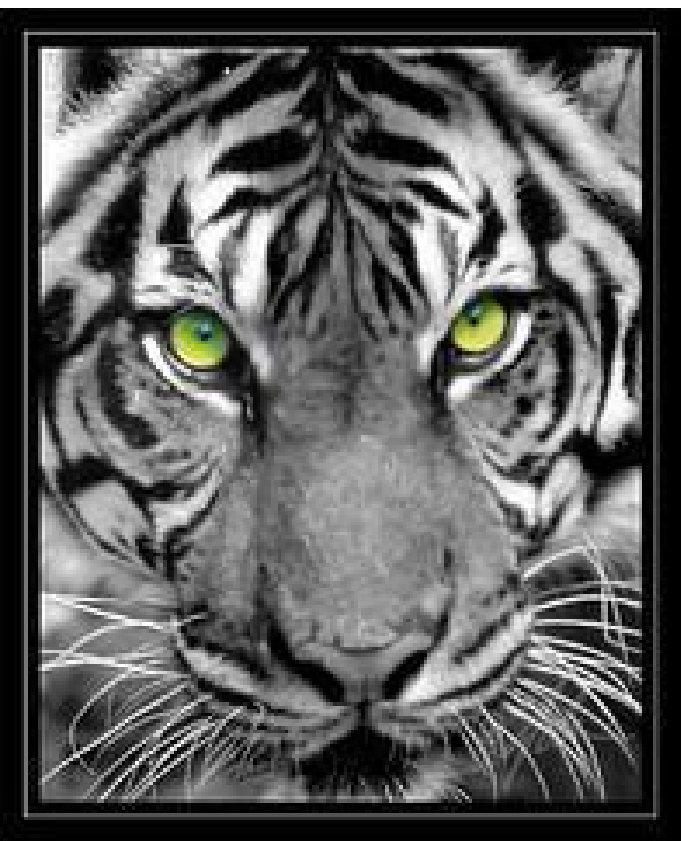}}} & {\multirow{5}{*}{\includegraphics[width=.18\columnwidth,height=.15\columnwidth]{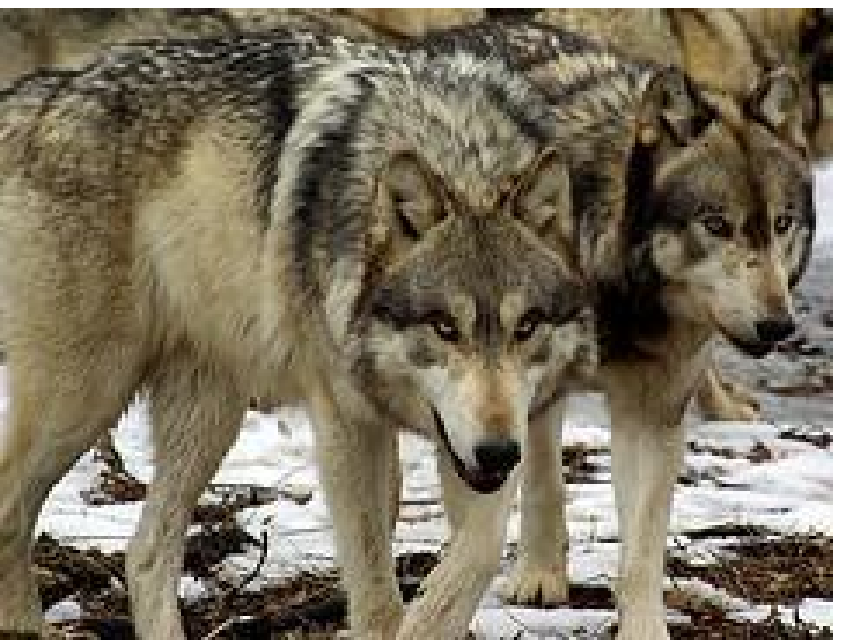}}} \\
        &&&&& \\
        &&&&& \\
        &&&&& \\
        &&&&& \\
        \hline \hline
\end{tabular}}
\caption{Six example features discovered iLSVM on the Flickr animal dataset. For each feature, we show 5 top-ranked images.} \label{fig:flickr_topic}
\end{center}\vspace{-1cm}
\end{figure*}

We compare iLSVM with the large-margin Harmonium (MMH)~\citep{Chen:nips10}, which was shown to outperform many other latent feature models, and two decoupled approaches -- {\it EFH+SVM} and {\it IBP+SVM}. EFH+SVM uses the exponential family Harmonium (EFH)~\citep{Welling:04} to discover latent features and then learns a multi-way SVM classifier. IBP+SVM is similar, but uses an IBP factor analysis model~\citep{Griffiths:tr05} to discover latent features. To initialize the learning algorithms for these models, we found that using the SVD factors of the input feature matrix as the initial weights for MMH and EFH can produce better results. Here, we also use the SVD factors as the initial mean of weights in the likelihood models for iLSVM.
Both MMH and EFH+SVM are finite models and they need to pre-specify the dimensionality of latent features. We report their results on classification accuracy and F1 score (i.e., the average F1 score over all possible classes)~\citep{Zhu:iSVM11} achieved with the best dimensionality in Table~\ref{table:Classification}. Figure~\ref{fig:FlickrK} illustrates the performance change of MMH when using different number of latent features, from which we can see that $K=40$ produces the best performance and either increasing or decreasing $K$ could make the performance worse. For iLSVM and IBP+SVM, we use the mean-field inference method and present the average performance with 5 randomly initialized runs (Please see Appendix D.2 for the algorithm and initialization details). We perform 5-fold cross-validation on training data to select hyperparameters, e.g., $\alpha$ and $C$ (we use the same procedure for MT-iLSVM). We can see that iLSVM can achieve comparable performance with the nearly optimal MMH, without needing to pre-specify the latent feature dimension\footnote{We set the truncation level to 300, which is large enough.}, and is much better than the decoupled approaches (i.e., IBP+SVM and EFH+SVM). For the two stage methods, we don't have a clear winner -- IBP+SVM performs a bit worse than EFH+SVM on the TRECVID dataset, while it outperforms EFH+SVM on the flickr dataset. The reason for the difference may be due to the initialization or different properties of the data.

It is also interesting to examine the discovered latent features. Figure~\ref{fig:TrecPattern} shows the overall average values of latent features and the per-class average feature values of iLSVM in one run on the TRECVID dataset. We can see that on average only about 45 features are active for the TRECVID dataset. For the overall average, we also present the standard deviation over the 5 categories. A larger deviation means that the corresponding feature is more discriminative when predicting different categories. For example, feature 26 and feature 34 are generally less discriminative than many other features, such as feature 1 and feature 30. Figure~\ref{fig:FlickrZ} shows the overall average feature values together with standard deviation on the Flickr dataset. We omitted the per-class average because that figure is too crowded with 13 categories. We can that as $k$ increases, the probability that feature $k$ is active decreases. The reason for the features with stable values (i.e., standard deviations are extremely small) is due to our initialization strategy (each feature has $0.5$ probability to be active). Initializing $\psi_{dk}$ as being exponentially decreasing (e.g., like the constructing process of $\piv$) leads to a faster decay and many features will be inactive. To examine the semantics\footnote{The interpretation of latent features depends heavily on the input data.} of each feature, Figure~\ref{fig:flickr_topic} presents some example features discovered on the Flickr animal dataset. For each feature, we present 5 top-ranked images which have large values on this particular feature. We can see that most of the features are semantically interpretable. For instance, feature F1 is about squirrel; feature F2 is about ocean animal, which is whales in the Flickr dataset; and feature F4 is about hawk. We can also see that some features are about different aspects of the same category. For example, feature F2 and feature F3 are both about whales, but with different background.

\subsection{Multi-task Learning}

Now, we evaluate the multi-task infinite latent SVM (MT-iLSVM) on several well-studied real datasets.

\subsubsection{Description of the Data}

{\bf Scene and Yeast Data}: These datasets are from the UCI repository, and each data example has multiple labels. As in~\citep{HalDaume:10}, we treat the multi-label classification as a multi-task learning problem, where each label assignment is treated as a binary classification task. The Yeast dataset consists of 1500 training and 917 test examples, each having 103 features, and the number of labels (or tasks) per example is 14. The Scene dataset consists 1211 training and 1196 test examples, each having 294 features, and the number of labels (or tasks) per example for this dataset is 6.


{\bf School Data}: This dataset comes from the Inner London Education Authority and has been used to study the effectiveness of schools. It consists of examination records of 15,362 students from 139 secondary schools in years 1985, 1986 and 1987. The dataset is publicly available and has been extensively evaluated in various multi-task learning methods~\citep{Bakker:JMLR03,Bonilla:MTGP08,Zhang:uai10}, where each task is defined as predicting the exam scores of students belonging to a specific school based on four student-dependent features (year of the exam, gender, VR band and ethnic group) and four school-dependent features (percentage of students eligible
for free school meals, percentage of students in VR band 1, school gender and school denomination). In order to compare with the above methods, we follow the same setup described in~\citep{Argyriou:nips07,Bakker:JMLR03} and similarly we create dummy variables for those features that are categorical forming a total of 19 student-dependent features and 8 school-dependent features. We use the same 10 random splits\footnote{Available at: http://ttic.uchicago.edu/$\sim$argyriou/code/index.html} of the data, so that $75\%$ of the examples from each school (task) belong to the training set and $25\%$ to the test set. On average, the training set includes about 80 students per school and the test set about 30 students per school.

\begin{table}[t]
\begin{center}
\begin{tabular}{|c|c|ccc|}
\hline
Dataset & Model & Acc & F1-Micro & F1-Macro \\
\hline
\multirow{5}{*}{Yeast} & YaXue  & 0.5106 & 0.3897  & 0.4022 \\
{} & Piyushrai-1 & 0.5212 & 0.3631  & 0.3901  \\
{} & Piyushrai-2 & 0.5424 & 0.3946  & 0.4112  \\
{} & MT-IBP+SVM  & 0.5475 $\pm$ 0.005 & 0.3910 $\pm$ 0.006 & 0.4345 $\pm$ 0.007  \\
{} & MT-iLSVM & {\bf 0.5792} $\pm$ 0.003 & {\bf 0.4258} $\pm$ 0.005 & {\bf 0.4742} $\pm$ 0.008  \\
\hline
\multirow{5}{*}{Scene} & YaXue       & 0.7765 & 0.2669 & 0.2816 \\
{} & Piyushrai-1  & 0.7756 & 0.3153 & 0.3242 \\
{} & Piyushrai-2  & 0.7911 & 0.3214 & 0.3226 \\
{} & MT-IBP+SVM  &  0.8590 $\pm$ 0.002 & 0.4880 $\pm$ 0.012 & 0.5147 $\pm$ 0.018 \\
{} & MT-iLSVM    & {\bf 0.8752} $\pm$ 0.004 & {\bf 0.5834} $\pm$ 0.026 & {\bf 0.6148} $\pm$ 0.020 \\
\hline
\end{tabular}
\end{center}
\caption{ Multi-label classification performance on Scene and Yeast datasets.}\label{table:MultiLabel}
\end{table}

\subsubsection{Results}

{\bf Scene and Yeast Data}: We compare with the closely related nonparametric Bayesian methods, including kernel stick-breaking (YaXue)~\citep{XueYa:icml07} and the basic and augmented infinite predictor subspace models (i.e., Piyushrai-1 and Piyushrai-2)~\citep{HalDaume:10}. These nonparametric Bayesian models were shown to outperform the independent Bayesian logistic regression and a single-task pooling approach~\citep{HalDaume:10}. We also compare with a decoupled method {\it MT-IBP+SVM}\footnote{This decoupled approach is in fact an one-iteration MT-iLSVM, where we first infer the shared latent matrix $\Zv$ and then learn an SVM classifier for each task.} that uses an IBP factor analysis model to find shared latent features among multiple tasks and then builds separate SVM classifiers for different tasks. For MT-iLSVM and MT-IBP+SVM, we use the mean-field inference method in Sec~\ref{sec:inference} and report the average performance with 5 randomly initialized runs (See Appendix D.1 for initialization details). For comparison with~\citep{HalDaume:10,XueYa:icml07}, we use the overall classification accuracy, F1-Macro and F1-Micro as performance measures. Table~\ref{table:MultiLabel} shows the results. On both datasets, MT-iLSVM needs less than 50 latent features on average. We can see that the large-margin MT-iLSVM performs much better than other nonparametric Bayesian methods and MT-IBP+SVM, which separates the inference of latent features from learning the classifiers.

\begin{figure*}[t]
\begin{center}
\includegraphics[width=0.75\columnwidth,height=0.35\columnwidth]{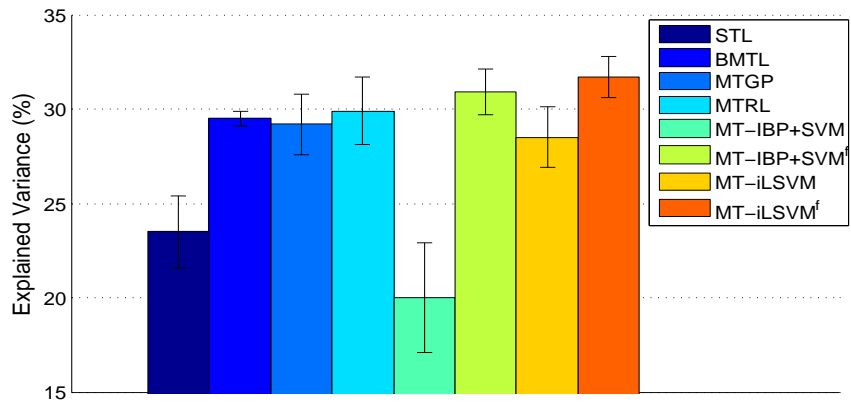}
\caption{Percentage of explained variance by various models on the School dataset.}
\label{fig:MultiTask}
\end{center}
\end{figure*}

\begin{figure*}[t]
\begin{center}
\subfigure[Yeast]{\includegraphics[width=0.43\columnwidth,height=0.33\columnwidth]{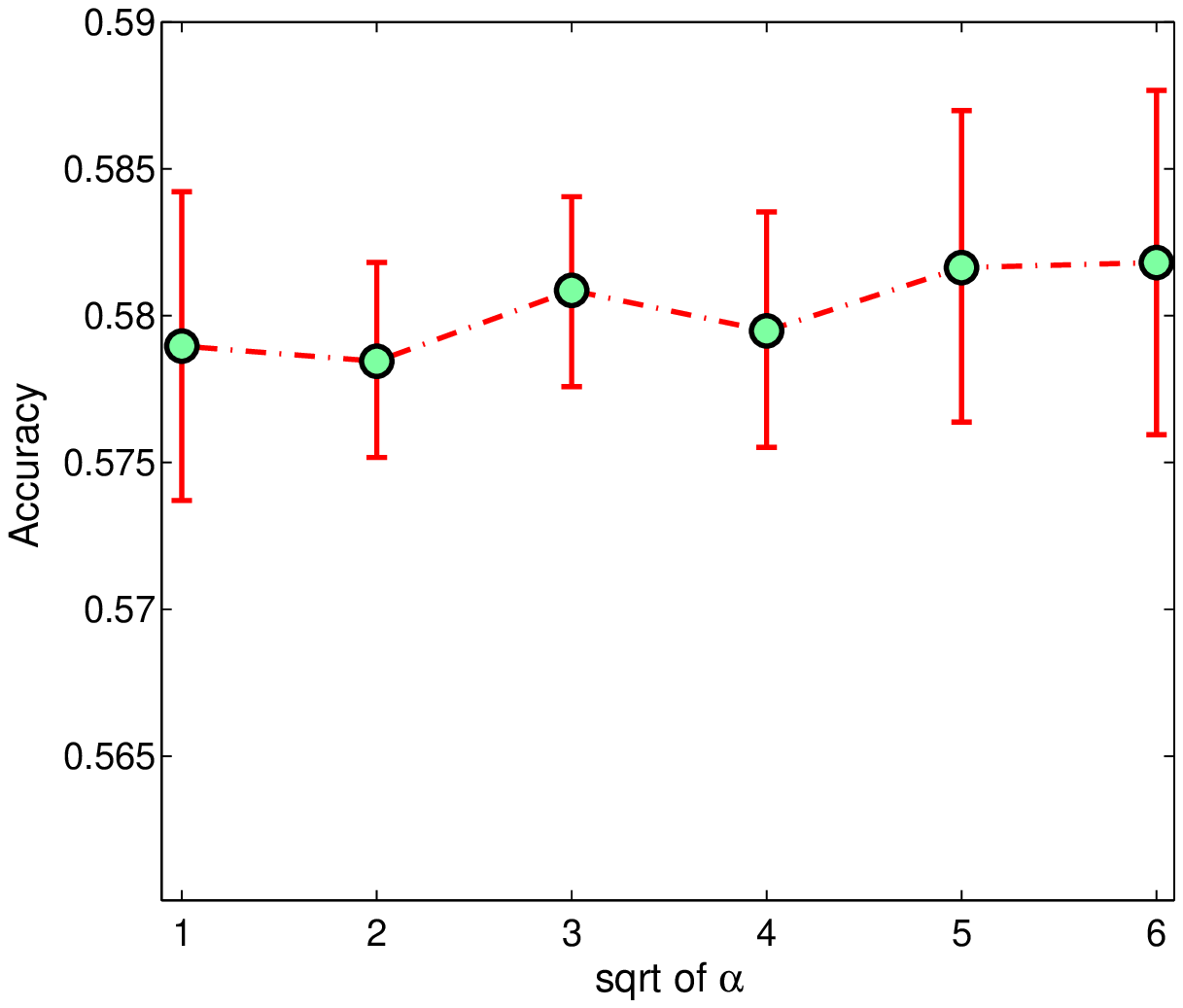}}\hfill
\subfigure[Yeast]{\includegraphics[width=0.43\columnwidth,height=0.33\columnwidth]{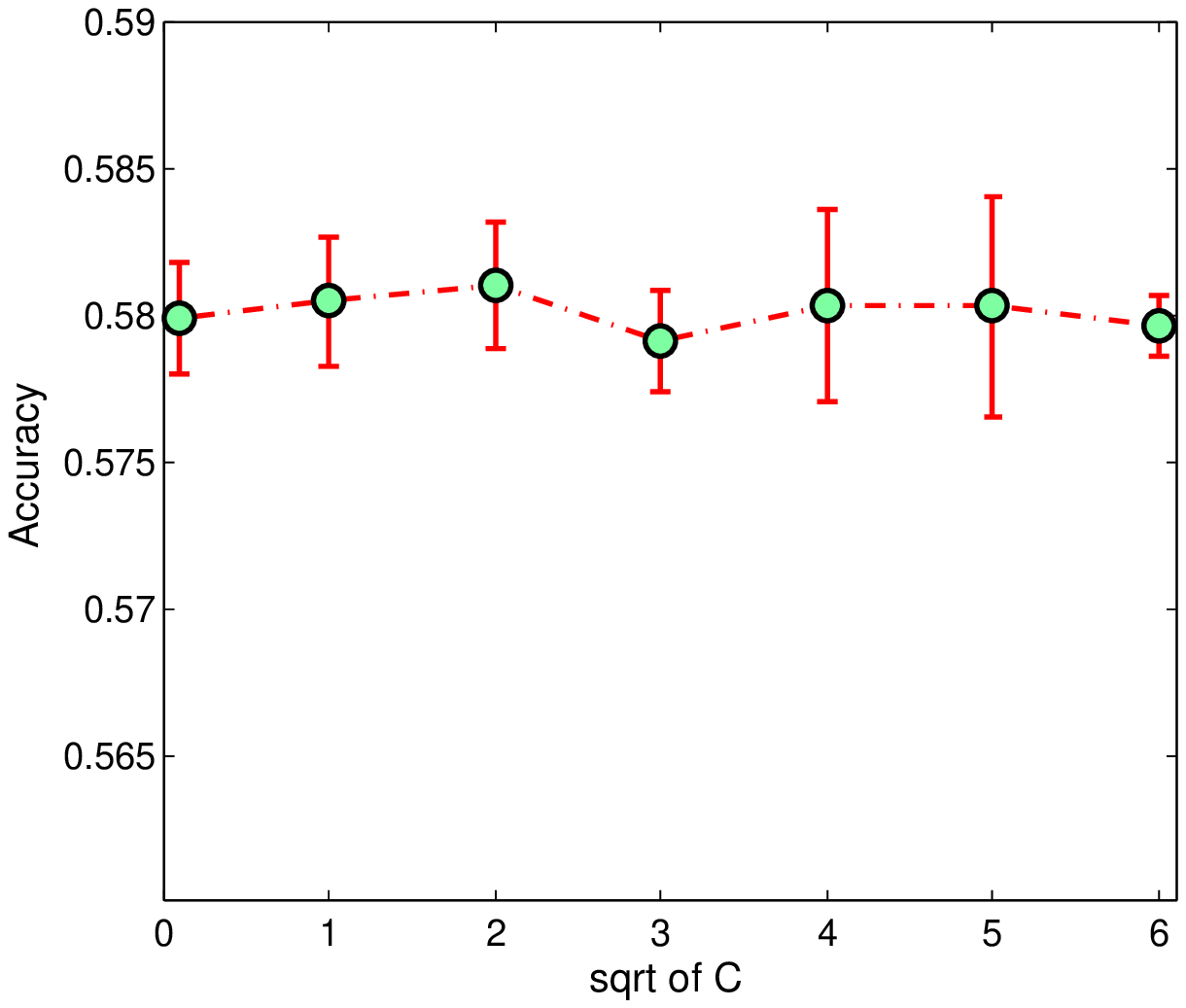}}\hfill
\subfigure[School]{\includegraphics[width=0.43\columnwidth,height=0.33\columnwidth]{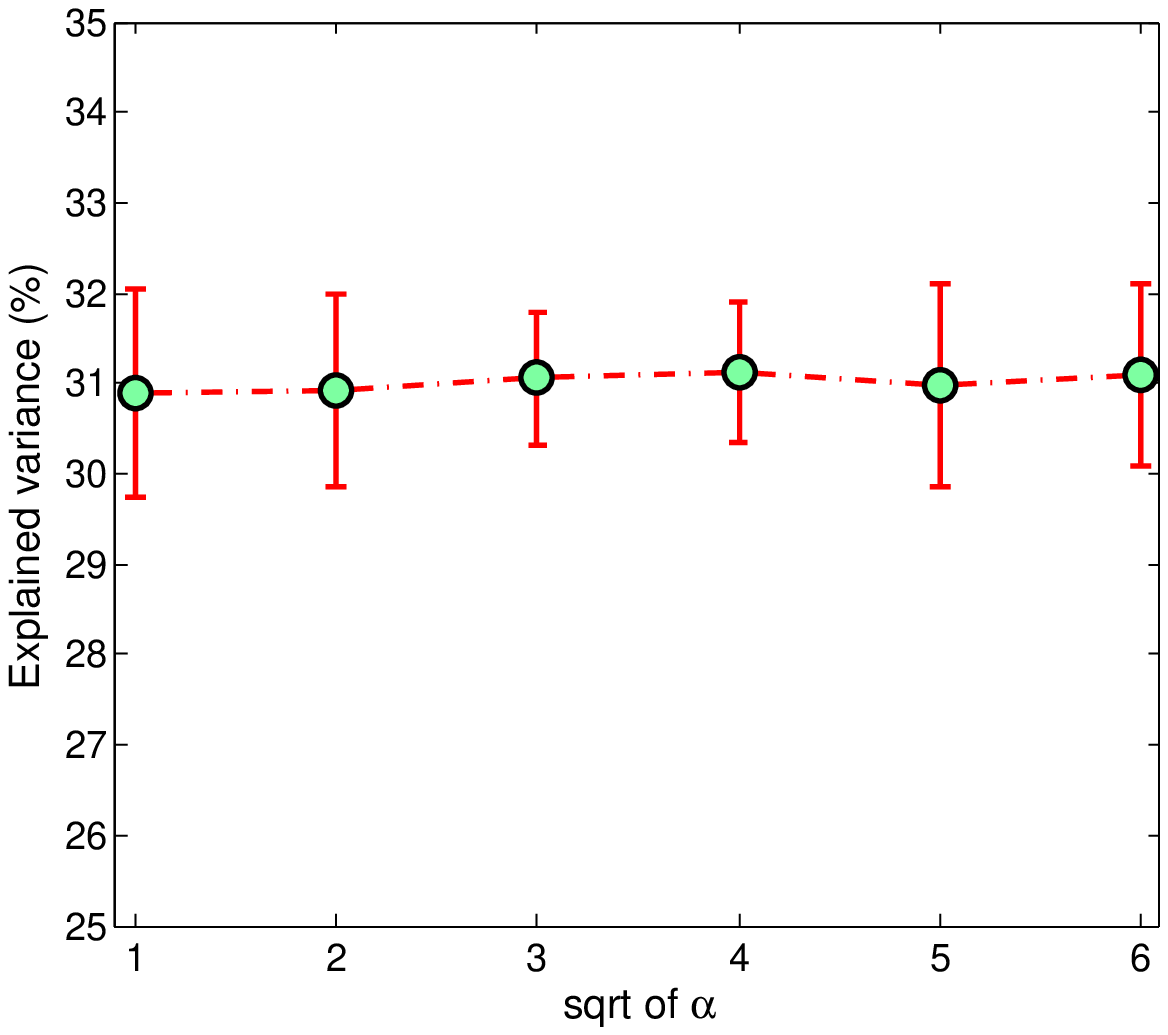}}\hfill
\subfigure[School]{\includegraphics[width=0.43\columnwidth,height=0.33\columnwidth]{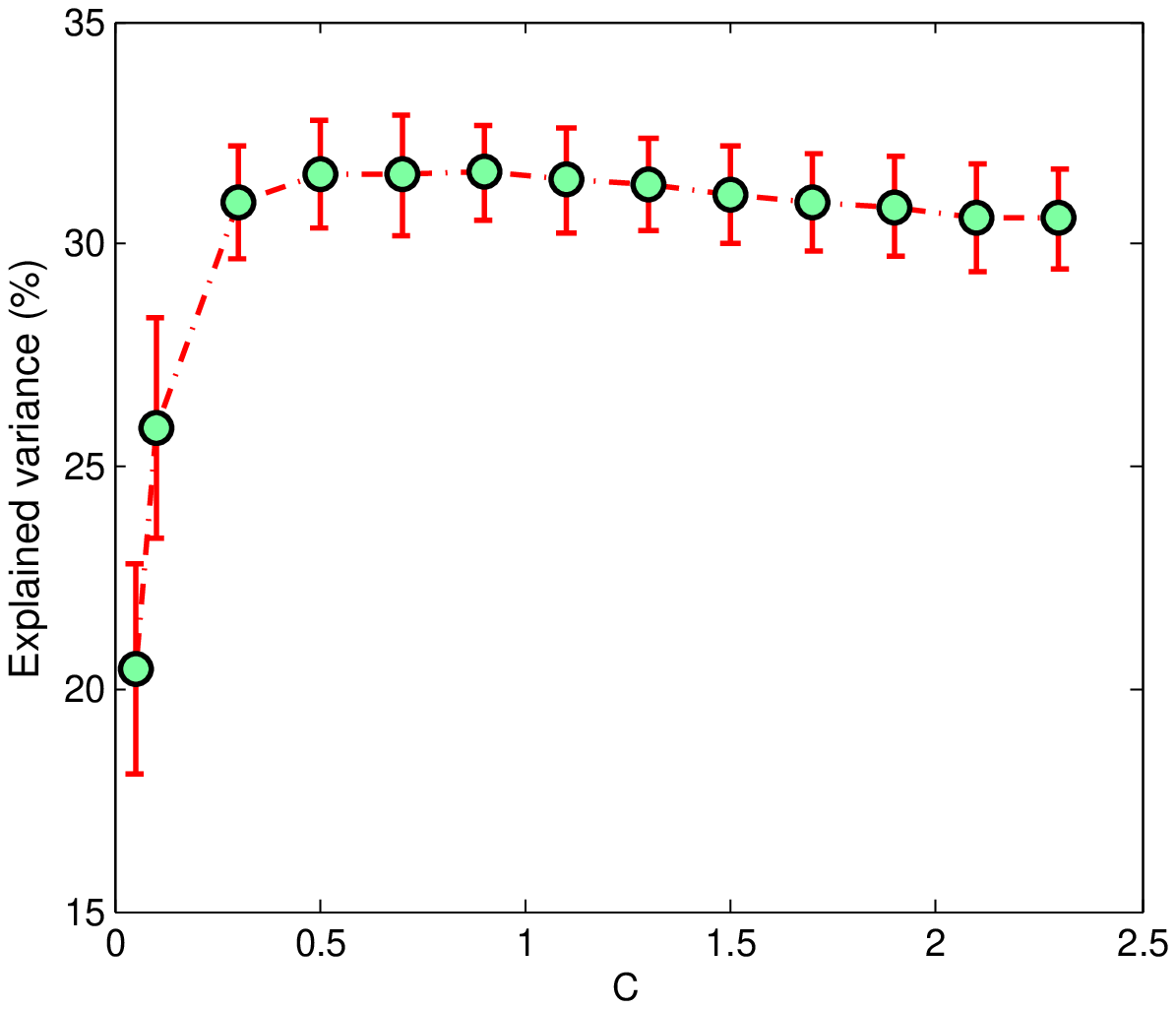}}
\caption{Sensitivity study of MT-iLSVM: (a) classification accuracy with different $\alpha$ on Yeast data; (b) classification accuracy with different $C$ on Yeast data; (c) percentage of explained variance with different $\alpha$ on School data; and (d) percentage of explained variance with different $C$ on School data.}
\label{fig:sensitivity}\vspace{-.4cm}
\end{center}
\end{figure*}

{\bf School Data}: We use the percentage of explained variance~\citep{Bakker:JMLR03} as the measure of the regression performance, which is defined as the total variance of the data minus the sum-squared error on the test set as a percentage of the total variance. Since we use the same settings, we can compare with the state-of-the-art results of
\begin{enumerate}[(1)]
\item Bayesian multi-task learning (BMTL)~\citep{Bakker:JMLR03};
\item Multi-task Gaussian processes (MTGP)~\citep{Bonilla:MTGP08};
\item Convex multi-task relationship learning (MTRL)~\citep{Zhang:uai10};
\end{enumerate}
and single-task learning (STL) as reported in~\citep{Bonilla:MTGP08,Zhang:uai10}. For MT-iLSVM and MT-IBP+SVM, we also report the results achieved by using both the latent features (i.e., $\Zv^\top \xv$) and the original input features $\xv$ through vector concatenation, and we denote the corresponding methods by {\it MT-iLSVM$^f$} and {\it MT-IBP+SVM$^f$}, respectively.
On average the multi-task latent SVM (i.e., MT-iLSVM) needs about 50 latent features to get sufficiently good and robust performance. From the results in Figure~\ref{fig:MultiTask}, we can see that the MT-iLSVM achieves better results than the existing methods that have been tested in previous studies. Again, the joint MT-iLSVM performs much better than the decoupled method MT-IBP+SVM, which separates the latent feature inference from the training of large-margin classifiers. Finally, using both latent features and the original input features can boost the performance slightly for MT-iLSVM, while much more significantly for the decoupled MT-IBP+SVM.

\subsection{Sensitivity Analysis}

\begin{figure*}[t]
\begin{center}
\includegraphics[width=0.88\columnwidth,height=0.4\columnwidth]{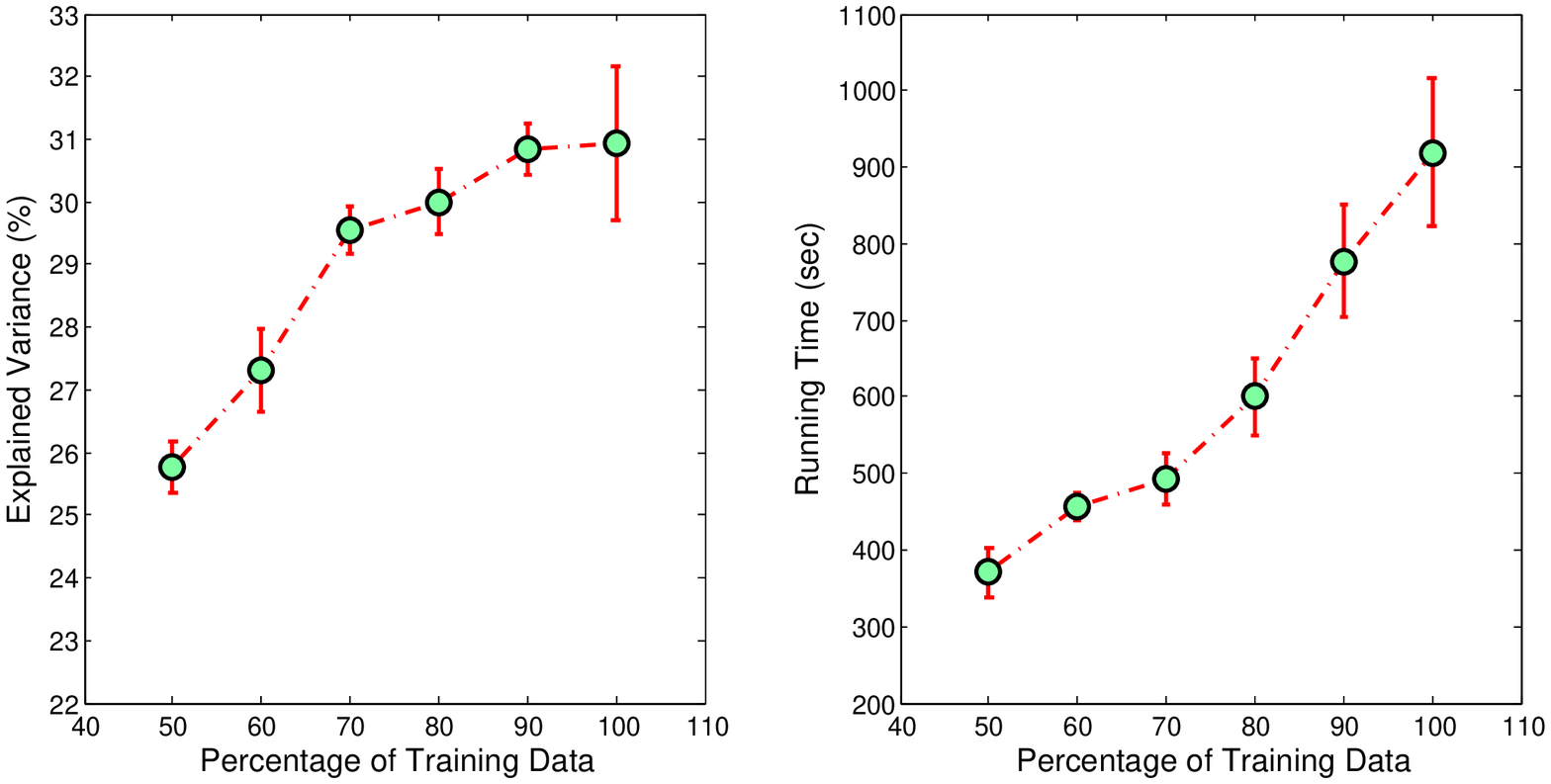}
\caption{Percentage of explained variance and running time by MT-iLSVM with various training sizes.}
\label{fig:Time}
\end{center}
\end{figure*}

Figure~\ref{fig:sensitivity} shows how the performance of MT-iLSVM changes against the hyper-parameter $\alpha$ and regularization constant $C$ on the Yeast and School datasets. We can see that on the Yeast dataset, MT-iLSVM is insensitive to both $\alpha$ and $C$. For the School dataset, MT-iLSVM is very insensitive the $\alpha$, and it is stable when $C$ is set between 0.3 and 1.


Figure~\ref{fig:Time} shows how the training size affects the performance and running time of MT-iLSVM on the School dataset. We use the first $b\%$ ($b=50,60,70,80,90,100$) of the training data in each of the 10 random splits as training set and use the corresponding test data as test set. We can see that as training size increases, the performance and running time generally increase; and MT-iLSVM achieves the state-of-art performance when using about $70\%$ training data. From the running time, we can also see that MT-iLSVM is generally quite efficient by using mean-field inference.

Finally, we investigate how the performance of MT-iLSVM changes against the hyperparameters $\sigma_{m0}^2$ and $\lambda_{mn}^2$. We initially set $\sigma_{m0}^2=1$ and compute $\lambda_{mn}^2$ from observed data. If we further estimate them by maximizing the objective function, the performance does not change much ($\pm 0.3 \%$ for average explained variance on the School dataset). We have similar observations for iLSVM.

\section{Conclusions and Discussions}

We present regularized Bayesian inference (RegBayes), a computational framework to perform post-data posterior inference with a rich set of regularization/constraints on the desired post-data posterior distributions. RegBayes is formulated as a information-theoretical optimization problem, and it is applicable to both directed and undirected graphical models. We present a general theorem to characterize the solution of RegBayes, when the posterior regularization is induced from a linear operator (e.g., expectation). Furthermore, we particularly concentrate on developing two large-margin nonparametric Bayesian models under the RegBayes framework to learn predictive latent features for classification and multi-task learning, by exploring the large-margin principle to define posterior constraints. Both models allow the latent dimension to be automatically resolved from the data. The empirical results on several real datasets appear to demonstrate that our methods inherit the merits from both Bayesian nonparametrics and large-margin learning.


RegBayes offers a flexible framework for considering posterior regularization in performing parametric or nonparametric Bayesian inference. For future work, we plan to study other posterior regularization beyond the large-margin constraints, such as posterior constraints defined on manifold structures~\citep{Huh:KDD10} and those represented in the form of first-order logic, and investigate how posterior regularization can be used in other interesting nonparametric Bayesian models~\citep{Beal:iHMM07,YWTeh:jasa06,Blei:icml10} in different contexts, such as link prediction~\citep{Miller:nips09} for social network analysis and low-rank matrix factorization for collaborative prediction. Some of our preliminary results~\citep{Xu:nips12,Zhu:icml12,Zhu:RegBayes-icml14} have shown great promise. It is interesting to investigate more carefully along this direction. Moreover, as we have stated, RegBayes can be developed for undirected MRFs. But the inference would be even harder. We plan to do a systematic investigation along this direction too. We have some preliminary results presented in~\citep{Chen:iEFH11}, but there is a lot of room to further improve. Finally, regularized Bayesian inference in general leads to a highly nontrivial inference problem. Although the general solution can be derived with convex analysis theory, it is normally intractable to infer them directly. Therefore, approximate inference techniques such as the truncated mean-field approximation have to be used. For the current truncated inference methods, one key limit is to pre-specify the truncation level. A too conservative truncation level could lead to a waste of computing resources. So, it is important to develop inference algorithms that could adaptively determine the number of latent features, such as Monte Carlo methods. We have some preliminary progress along this direction as reported in the work~\citep{Zhu:nips12,Zhu:icml13}. It is interesting to extend these techniques to deal with other challenging nonparametric Bayesian models.

\section*{Acknowledgements}

We thank the anonymous reviewers and the editors for many helpful comments to improve the manuscript. NC and JZ are supported by National Key Foundation R\&D Projects (No.s 2013CB329403, 2012CB316301), National Natural Science Foundation of China (Nos. 61322308, 61332007, 61305066), Tsinghua University Initiative Scientific Research Program (No. 20121088071), and the China Postdoctoral Science Foundation Grant (No. 2013T60117) to NC. EX is supported by AFOSR FA95501010247, ONR N000140910758, NSF Career DBI-0546594 and an Alfred P. Sloan Research Fellowship.

\appendix

\section*{Appendix A: Generalization Beyond Bayesian Networks}

Standard Bayesian inference and the proposed regularized Bayesian inference implicitly make the assumption that the model can be graphically drawn as a Bayesian network as illustrated in Figure~\ref{fig:BayesGraph1}\footnote{The structure within $\model$ can be arbitrary, either a directed, undirected or hybrid chain graph.}. Here, we consider a more general formulation which could cover both directed and undirected latent variable models, such as the well-studied Boltzmann machines~\citep{Murray:UAI04,Welling:04}, as well as the case where a model could have some unknown parameters (e.g., hyper-parameters) and need an estimation procedure, such as maximum likelihood estimation (MLE), besides posterior inference. The latter is also known as empirical Bayesian methods, which are frequently employed by practitioners.

{\bf Extension 1: Empirical Bayesian Inference with Unknown Parameters}: As illustrated in Figure~\ref{fig:BayesGraph2}, in some cases we need to perform the empirical Bayesian inference in the presence of unknown parameters. For instance, in a linear-Gaussian Bayesian model, we may choose to estimate its covariance matrix using MLE; and in a latent Dirichlet allocation (LDA)~\citep{Blei:03} model, we may choose to estimate the unknown topical dictionary, although in principle we can treat these parameters as random variables and perform full Bayesian inference. In such cases, we need some mechanisms to estimate the unknown parameters when doing Bayesian inference. Let $\Theta$ be model parameters. We can formulate empirical Bayesian inference as solving\footnote{The objective can be derived using variational techniques. It is in fact a variational upper bound of the negative log-likelihood.}
\begin{eqnarray}\label{eq:MLEBayes}
\inf_{\Theta, q(\model)} &&  \KL(q(\model) \Vert \pi(\model)) - \int_{\ms} \log p(\mathcal{D}|\model,\Theta) q(\model) d\model  \\
\mathrm{s.t.:} && q(\model) \in \mathcal{P}_{\textrm{prob}}.\nonumber
\end{eqnarray}
Although the problem is convex over $q(\model)$ for any fixed $\Theta$, it is not jointly convex in general. A natural algorithm to solve this problem is the well-known EM procedure~\citep{Dempster:77}, which converges to a local optimum. Specifically, we have the following result.
\begin{lemma}\label{lemma:MLEDuality}
For problem~(\ref{eq:MLEBayes}), the optimum solution of $q(\model)$ is equivalent to the posterior distribution by Bayes' theorem for any $\Theta$; and the optimum $\Theta^\ast$ is the MLE
\begin{eqnarray}
\Theta^\ast = \argmax_\Theta \log p(\mathcal{D}|\Theta). \nonumber
\end{eqnarray}
\end{lemma}
\begin{proof}
According to the variational formulation of Bayes' rule in Eq.~(\ref{eq:BasicMaxEnt}), we get that the optimum solution is $q(\model) = p(\model|\data, \Theta)$ for any $\Theta$. Substituting the optimum solution of $q$ into the objective, we get the optimization problem of $\Theta$.
\end{proof}

\begin{figure*}
\begin{center}
{\hfill\subfigure[]{\includegraphics[height=0.25\columnwidth]{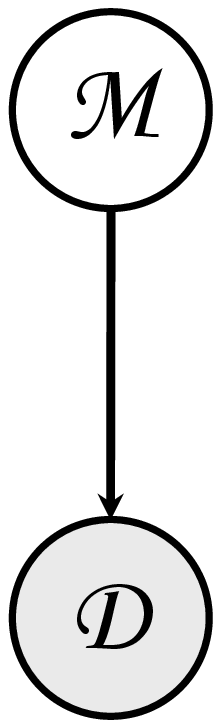}\label{fig:BayesGraph1}}\hfill
\subfigure[]{\includegraphics[height=.25\columnwidth]{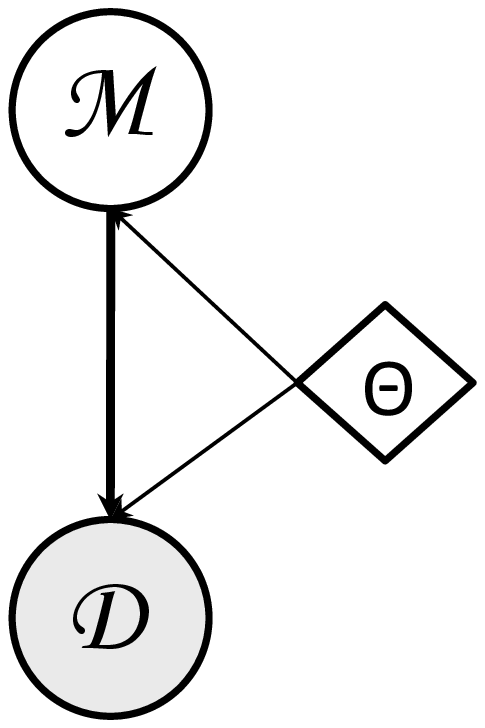}\label{fig:BayesGraph2}}\hfill
\subfigure[]{\includegraphics[height=.25\columnwidth]{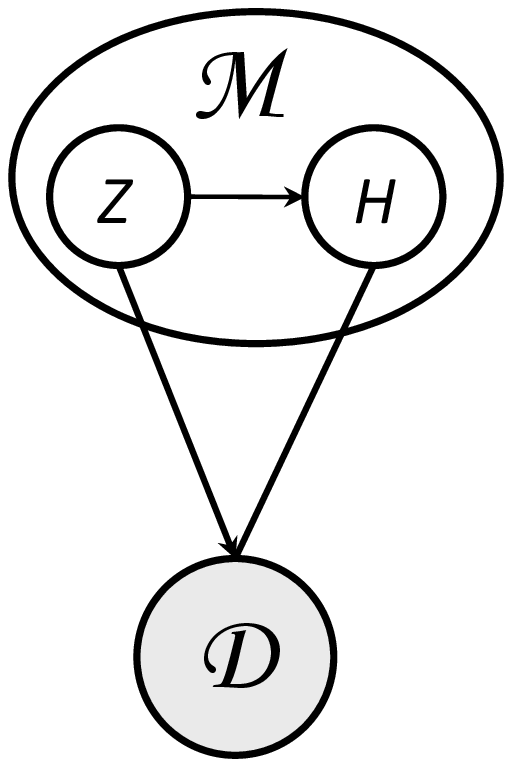}\label{fig:BayesGraph3}}\hfill}
\caption{Illustration graphs for three different types of models that involve Bayesian inference: (a) a Bayesian generative model; (b) a Bayesian generative model with unknown parameters $\Theta$; and (c) a chain graph model.}
\end{center}
\end{figure*}

{\bf Extension 2: Chain Graph}: In the above cases, we have assumed that the observed data are generated by some model in a directed causal sense. This assumption holds in directed latent variable models. However, in many cases, we may choose alternative formulations to define the joint distribution of a model and the observed data. Figure~\ref{fig:BayesGraph3} illustrates one such scenario, where the model $\model$ consists of two subsets of random variables. One subset $H$ is connected to the observed data via an undirected graph and the other subset $Z$ is connected to the observed data and $H$ using directed edges. This graph is known as a chain graph. Due to the Markov properties of chain graph~\citep{Frydenberg:90}, we know that the joint distribution has the factorization form as
\begin{eqnarray}\label{eq:chainGraph}
p(\model, \data) = p(Z) p(H, \data | Z),
\end{eqnarray}
where $p(H,\data|Z)$ is a Markov random field (MRF).
One concrete example of such a hybrid chain model is the Bayesian Boltzman machines~\citep{Murray:UAI04}, which treat the parameters of a Boltzmann machine as random variables and perform Bayesian inference with MCMC sampling methods.

The insights that RegBayes covers undirected or chain graph latent variable models come from the observation that the objective $\mathcal{L}(q(\model))$ of problem~(\ref{eq:BasicMaxEnt}) is in fact an KL-divergence, namely, we can show that
\begin{eqnarray}\label{eq:JointKL}
\mathcal{L}(q(\model)) = \KL( q(\model) \Vert p(\model, \data) ),
\end{eqnarray}
where $p(\model, \data)$ is the joint distribution. Note that when $\data$ is given, the distribution $p(\model, \data)$ is non-normalized for $\model$; and we have abused the KL notation for non-normalized distributions in Eq.~(\ref{eq:JointKL}), but with the same formula. For directed Bayesian networks~\citep{Zhu:iLSVM11}, we naturally have $p(\model, \data) = \pi(\model) p(\data|\model)$. For the undirected MRF models, we have $\model=\{Z, H\}$ and again we can define the joint distribution as in Eq.~(\ref{eq:chainGraph}).

Putting the above two extensions of Bayesian inference together, the regularized Bayesian inference with estimating unknown model parameters can be generally formulated as
\setlength\arraycolsep{1pt}\begin{eqnarray}\label{eq:constraindBayes_param}
\inf_{\Theta, q(\model), \xiv} && \mathcal{L}( \Theta, q(\model) ) + U(\xiv)~~~\mathrm{or}~~~ \inf_{\Theta, q(\model)}~  \mathcal{L}( \Theta, q(\model) ) + g(Eq(\model)) \\
\mathrm{s.t.:} && q(\model) \in \mathcal{P}_{\textrm{post}}(\xiv) ~~~~~~~~~~~~~~~~ \mathrm{s.t.:}~~  q(\model) \in \mathcal{P}_{\textrm{prob}}, \nonumber
\end{eqnarray}
where $\mathcal{L}(\Theta, q(\model))$ is the objective function of problem~(\ref{eq:MLEBayes}). These two formulations are equivalent. We will call the former a {\it constrained} formulation and call the latter an {\it unconstrained} formulation by ignoring the standard normalization constraints, which are easy to deal with.

\section*{Appendix B: MedLDA---A RegBayes Model with Finite Latent Features}\label{sec:MedLDA}

This section presents a new interpretation of MedLDA (maximum entropy discrimination latent Dirichlet allocation)~\citep{Zhu:MedLDA09} under the framework of regularized Bayesian inference. MedLDA is a max-margin supervised topic model, an extension of latent Dirichlet allocation (LDA)~\citep{Blei:03} for supervised learning tasks. In MedLDA, each data example is projected to a point in a finite dimensional latent space, of which each feature corresponds to a topic, i.e., a unigram distribution over the terms in a vocabulary. MedLDA represents each data as a probability distribution over the features, which results in a conservation constraint (i.e., the more a data expresses on one feature, the less it can express others) ~\citep{Griffiths:tr05}. The infinite latent feature models discussed in Section~\ref{sec:ilsvm} do not have such a constraint.

Without loss of generality, we consider the MedLDA regression model as an example (classification model is similar), whose graphical structure is shown in Figure~\ref{fig:MedLDA}. We assume that all data examples have the same length $V$ for notation simplicity. Each document is associated with a response variable $Y$, which is observed in the training phase but unobserved in testing. We will use $y$ to denote an instance value of $Y$. Let $K$ be the number of topics or the dimensionality of the latent topic space. MedLDA builds an LDA model to describe the observed words. The generating process of LDA is that each document $n$ has a mixing proportion $\thetav_n \sim \textrm{Dirichlet}(\alphav)$; each word $w_{nm}$ is associated with a topic $z_{nm} \sim \thetav_n$, which indexes the topic that generates the word, i.e., $w_{nm} \sim \betav_{z_{nm}}$. Define $\bar{Z}_n = \frac{1}{V} \sum_{m=1}^V Z_{nm}$ as the average topic assignment for document $n$. Let $\Theta = \{\alphav, \betav, \delta^2\}$ denote the unknown model parameters and $\data = \{y_n, w_{nm}\}$ be the training set.
MedLDA was defined as solving a regularized MLE problem with expectation constraints
\setlength\arraycolsep{2pt}\begin{eqnarray}
\inf_{\Theta, \xiv,
\xiv^\ast}~ & & -\log p(\{y_n, w_{nm}\} | \Theta) + C \sum_{n=1}^N (\xi_n + \xi_n^\ast)  \\
\mathrm{s.t.}~\forall n: & & \left\{ \begin{array}{rcl}
    y_n - \ep_p\lbrack \etav^\top \bar{Z}_n \rbrack & \leq & \epsilon + \xi_n \\
    - y_n + \ep_p\lbrack \etav^\top \bar{Z}_n \rbrack & \leq & \epsilon + \xi_n^\ast \\
    \xi_n,~\xi_n^\ast & \geq & 0
    \end{array} \right. \nonumber
\end{eqnarray}
The posterior constraints are imposed following the large-margin principle and they correspond to a quality measure of the prediction results on training data. In fact, it is easy to show that minimizing $U(\xiv, \xiv^\ast) = C \sum_{n=1}^N (\xi_n + \xi_n^\ast)$ under the above constraints is equivalent to minimizing an $\epsilon$-insensitive loss~(Smola and Sch$\ddot{o}$lkopf, 2003)
\begin{eqnarray}
\mathcal{R}_\epsilon\Big( p(\{\thetav_n, z_{nm}, \etav\} | \data, \Theta ) \Big) = C \sum_{n=1}^N \max(0, |y_n - \ep_p[\etav^\top \bar{Z}_n]| - \epsilon).
\end{eqnarray}
of the expected linear prediction rule $\hat{y}_n = \ep_p\lbrack \etav^\top \bar{Z}_n \rbrack$.

\begin{figure*}
\begin{center}
\includegraphics[height=0.22\columnwidth]{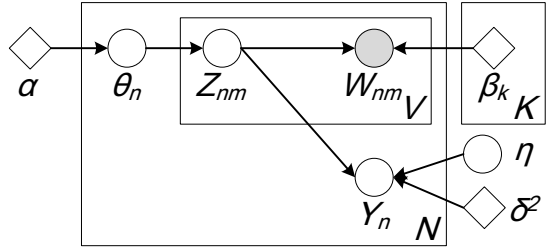}
\caption{Graphical structure of MedLDA.}\label{fig:MedLDA}\vspace{-.2cm}
\end{center}
\end{figure*}


To practically learn an MedLDA model, since the above problem is intractable, variational methods were used by introducing an auxiliary distribution $q(\{\thetav_n, z_{nm}, \etav\} | \Theta)$~\footnote{We have explicitly written the condition on model parameters.} to approximate the true posterior $p(\{\thetav_n, z_{nm}, \etav\} | \data, \Theta )$, replacing the negative data likelihood with its upper bound $\mathcal{L}\big(q(\{\thetav_n, z_{nm}, \etav\} | \Theta )\big)$, and replacing $p$ by $q$ in the constraints. The variational MedLDA regression model is
\setlength\arraycolsep{2pt}\begin{eqnarray}
\inf_{q, \Theta, \xiv,
\xiv^\ast}~ & & \mathcal{L}\Big(q(\{\thetav_n, z_{nm}, \etav\} | \Theta )\Big) + C \sum_{n=1}^N (\xi_n + \xi_n^\ast)  \\
\mathrm{s.t.}~\forall n: & & \left\{ \begin{array}{rcl}
    y_n - \ep_q\lbrack \etav^\top \bar{Z}_n \rbrack & \leq & \epsilon + \xi_n \\
    - y_n + \ep_q\lbrack \etav^\top \bar{Z}_n \rbrack & \leq & \epsilon + \xi_n^\ast \\
    \xi_n,~\xi_n^\ast & \geq & 0
    \end{array} \right. \nonumber
\end{eqnarray}
where $\mathcal{L}\big(q(\{\thetav_n, z_{nm}, \etav\} | \Theta )\big) = - \ep_q\big[ \log p(\{\thetav_n, z_{nm}, \etav\}, \data | \Theta) \big\rbrack - \mathcal{H}\big(q(\{\thetav_n, z_{nm}, \etav\} | \Theta )\big)$ is a variational upper-bound of the negative data log-likelihood. The upper bound is tight if no restricting constraints are made on the variational distribution $q$. In practice, additional assumptions (e.g., mean-field) can be made on $q$ to derive a practical approximate algorithm.

Based on the previous discussions on the extensions of RegBayes and the duality in Lemma~\ref{lemma:MLEDuality}, we can reformulate the MedLDA regression model as an example of RegBayes. Specifically, for the MedLDA regression model, we have $\model = \{\thetav_n, z_{nm}, \etav\}$. According to Eq.~(\ref{eq:JointKL}), we can easily show that
\begin{eqnarray}
\mathcal{L}\Big(q(\{\thetav_n, z_{nm}, \etav\} | \Theta )\Big) &=& \KL\Big( q(\{\thetav_n, z_{nm}, \etav \}|\Theta) \Vert p(\{\thetav_n, z_{nm}, \etav\}, \{w_{nm}, y_n\} | \Theta) \Big) \nonumber \\
&=& \mathcal{L}_B\Big( \Theta, q(\model|\Theta) \Big). \nonumber
\end{eqnarray}
Then, the MedLDA problem is a RegBayes model in Eq.~(\ref{eq:constraindBayes_param}) with
\setlength\arraycolsep{1pt} \begin{eqnarray}\label{eq:MedLDAConstraints}
\mathcal{P}^{\textrm{MedLDA}}_{\mathrm{post}}(\Theta, \xiv, \xiv^\ast) \defEq \left\{ q(\{\thetav_n, z_{nm}, \etav\} | \Theta)~ \begin{array}{|crcl}
\forall n: & y_n - \ep_q\lbrack \etav^\top \bar{Z}_n \rbrack & \leq & \epsilon + \xi_n \\
   {} & - y_n + \ep_q\lbrack \etav^\top \bar{Z}_n \rbrack & \leq & \epsilon + \xi_n^\ast \\
   {} & \xi_n,~\xi_n^\ast & \geq & 0
\end{array}\right\}.
\end{eqnarray}

For the MedLDA problem, we can use Lagrangian methods to solve the constrained formulation. Alternatively, we can also use the convex duality theorem to solve the equivalent unconstrained form. For the variational MedLDA, the $\epsilon$-insensitive loss is $\mathcal{R}_\epsilon(q(\{\theta_n, z_{nm}, \etav\} | \Theta))$. Its conjugate can be derived using the results of Lemma~\ref{lemma:ConjugateMaxFunc2}. Specifically, we have the following result, whose proof is deferred to Appendix C.6.
\begin{lemma}[Conjugate of MedLDA]\label{lemma:conjugateMedLDA}
For the variational MedLDA problem, we have
\begin{eqnarray}
 \inf_{\Theta, q(\{\thetav_n, z_{nm}, \etav\} | \Theta) \in \mathcal{P}_{\mathrm{prob}}} && \mathcal{L}(q(\{\thetav_n, z_{nm}, \etav\} | \Theta), \Theta) + \mathcal{R}_\epsilon(q(\{\thetav_n, z_{nm}, \etav\} | \Theta)) \\
 =~~~~~~~~ \sup_{\omegav} ~~~~~~~~~&& - \log Z^\prime(\omegav, \Theta^\ast) - \sum_n g_2^\ast( \omegav_n;  -y_n+\epsilon, y_n+\epsilon ) , \nonumber
\end{eqnarray}
where $\omegav_n = (\omega_n, \omega_n^\prime)$. Moreover, The optimum distribution is the posterior distribution
\begin{eqnarray}\label{eq:qOptMedLDA}
\hat{q}(\{\thetav_n, z_{nm}, \etav\} | \Theta^\ast) = \frac{1}{Z^\prime(\hat{\omegav}, \Theta^\ast | \data)}p(\{\thetav_n, z_{nm}, \etav\}, \data | \Theta^\ast)\exp\Big\{ \sum_{n}(\hat{\omega}_n - \hat{\omega}_n^\prime) \etav^\top \bar{z}_n  \Big\},
\end{eqnarray}
where $Z^\prime(\hat{\omegav}, \Theta | \data)$ is the normalization factor and the optimum parameters are
\begin{eqnarray}
\Theta^\ast = \argmax_\Theta \log p(\data | \Theta).
\end{eqnarray}
\end{lemma}

Note that although in general, either the primal or the dual problem is hard to solve exactly, the above conjugate results are still useful when developing approximate inference algorithms. For instance, we can impose additional mean-field assumptions on $q$ in the primal formulation and iteratively solve for each factor; and in this process convex conjugates are useful to deal with the large-margin constraints~\citep{Zhu:MedLDA09}. Alternatively, we can apply approximate methods (e.g., MCMC sampling) to infer the $q$ based on its solution in Eq.~(\ref{eq:qOptMedLDA}), and iteratively solves for the dual parameters $\omegav$ using approximate statistics~\citep{schofield2006fitting}. We will discuss more on this when presenting
the inference algorithms for iLSVM and MT-iLSVM.

In the above discussions, we have treated the topics $\betav$ as fixed unknown parameters. A fully Bayesian formulation would treat $\betav$ as
random variables, e.g., with a Dirichlet prior~\citep{Blei:03,Griffiths:04}. Under the RegBayes interpretation, we can easily do such an extension of MedLDA, simply by
moving $\betav$ from $\Theta$ to $\model$.

\section*{Appendix C: Proof of the Theorems and Lemmas}

\subsection*{Appendix C.1: Proof of Theorem~\ref{lemma:RegBayes}}
\begin{proof}
The adjoint of the linear operator $E$ is given by $\left< Ex, \phi \right> = \left< E^\ast \phi, x \right>$. In this theorem, $E$ is the expectation with respect to $q$. Thus, we have
\begin{eqnarray}
\left< Eq, \phi \right> &=& \left< \int q(\model) \psiv(\model,\data) \ud \mu(\model), \phi \right>  \nonumber \\
              &=& \int q(\model) \left< \psiv(\model,\data), \phi \right>  \ud \mu(\model)  \nonumber \\
              &=& (E^\ast \phi)(q),
\end{eqnarray}
where $E^\ast \phi = \left< \phi, \psi(.) \right>$.

By definition, we have $\KL(q(\model) \Vert p(\model, \data)) = \KL(q(\model) \Vert p(\model|\data)) + c$, where $c = -\log p(\data)$ is a constant. Let $f(q(\model))$ denote the KL-divergence $\KL(q(\model) \Vert p(\model | \data))$. 
The following proof is similar to the proof of the Fenchel duality theorem~\citep{Borwein:05}. Let $t$ and $d$ denote the primal value and the dual value, respectively. By Lemma 4.3.1~\citep{Borwein:05}, under appropriate regularity conditions, there is a $\hat{\phi}$ such that
\begin{eqnarray}
t \leq \left[ f(q) - \left< \hat{\phi}, Eq \right> \right] + \left[ g(\phi) + \left< \hat{\phi}, \phi \right> \right] + c. \nonumber
\end{eqnarray}
For any $\mu$, setting $\phi = Eq + \mu$ in the above inequality, we have
\begin{eqnarray}
t & \leq & f(q) + g(E q + \mu) + \left< \hat{\phi}, \mu \right>  + c \nonumber \\
   &=& \left\{ f(q) - \left< E^\ast \hat{\phi}, q\right> \right\} + \left\{ g(Eq + \mu) - \left< -\hat{\phi}, Eq + \mu\right> \right\}  + c.  \nonumber
\end{eqnarray}
Taking the infimum over all points $\mu$, we have
\begin{eqnarray}
t  \leq   \left\{ f(q) - \left< E^\ast \hat{\phi}, q\right> \right\} - g^\ast( - \hat{\phi} )  + c. \nonumber
\end{eqnarray}
Then, taking the infimum over all points $q \in \mathcal{P}_{\textrm{prob}}$, we have
\begin{eqnarray}\label{eq:infimum2}
t  &\leq&   \inf_{q \in \mathcal{P}_{\textrm{prob}}} \left\{ f(q) - \left< E^\ast \hat{\phi}, q\right> \right\} - g^\ast( - \hat{\phi} )  + c \nonumber \\
   &=& -f^\ast( E^\ast \hat{\phi}) - g^\ast(- \hat{\phi})  + c \nonumber \\
   &\leq& d,
\end{eqnarray}
where
\begin{eqnarray}
f^\ast(E^\ast \phi) &=& \log \int p(\model | \data) \exp\left( \left< \phi, \psi(\model,\data) \right> \right) \ud \mu(\model)  \nonumber
\end{eqnarray}
is the convex conjugate of the KL-divergence.

Since $d \leq t$ due to the Fenchel weak duality theorem~\citep{Borwein:05} (Theorem 4.4.2), we have the strong duality that $t = d$, and $\hat{\phi}$ attains the supremum in the dual problem.
During the deviation of the infimum in Eq.~(\ref{eq:infimum2}), we get the optimum solution of $q$:
\begin{eqnarray}
\hat{q}_{\hat{\phi}}(\model) &\propto& p(\model|\data) \exp\left( \left< \hat{\phi}, \psiv(\model; \data) \right> \right) \nonumber \\
                &=& p(\model, \data) \exp\left( \left< \hat{\phi}, \psiv(\model; \data) \right> - \Lambda_{\hat{\phi}} \right). \nonumber
\end{eqnarray}
Absorbing the constant $c$ into $f^\ast$, we get the dual objective of Theorem~\ref{lemma:RegBayes}.
\end{proof}

\subsection*{Appendix C.2: Proof of Lemma~\ref{proposition:ConjugateBinaryFunc}}
\begin{proof}
By definition, $g_0^\ast(\mu) = \sup_{x \in \mathbb{R}}(x\mu - C\max(0, x))$. We consider two cases. First, if $\mu < 0$, we have
$$g_0^\ast(\mu) \geq \sup_{x < 0}(x \mu - C\max(0,x)) = \sup_{x < 0} x \mu = \infty.$$ Therefore, we have $g_0^\ast(\mu) = \infty$ if $\mu < 0$.
Second, if $\mu \geq 0$, we have
$$g_0^\ast(\mu) = \sup_{x \geq 0}(x \mu - C x) = \indicator( \mu \leq C ).$$
Putting the above results together, we prove the claim.
\end{proof}

\subsection*{Appendix C.3: Proof of Lemma~\ref{proposition:ConjugateMaxFunc}}\label{sec:Appendix1}

\begin{proof}
The proof has a similar structure as the proof of Lemma~\ref{proposition:ConjugateBinaryFunc}. By definition, we have
\begin{eqnarray}
g_1^\ast(\muv) = \sup_{\xv}\Big\{ \muv^\top \xv - g_1(\xv) \Big\} = \sup_{\xv}\Big\{ \sum_j \mu_j x_j - \max(x_1, \cdots, x_L) \Big\}. \nonumber
\end{eqnarray}
We first show that $\forall i,~\mu_i \geq 0$ in order to have finite $g_1^\ast$ values. Suppose that $\exists j,~\mu_j < 0$. Then, we define
\begin{eqnarray}
\mathcal{G}_j = \{\xv \in \mathbb{R}^L: x_j < 0 \},~\textrm{and}~\mathcal{G}_j^o = \{\xv \in \mathcal{G}_j: x_i = 0,~\textrm{if}~i \neq j \}.
\end{eqnarray}
Since $\mathcal{G}_j^o \subset \mathcal{G}_j \subset \mathbb{R}^L$, we have $$g_1^\ast(\muv) \geq \sup_{\xv \in \mathcal{G}_j} \{ \muv^\top \xv - g_1(\xv) \} \geq \sup_{\xv \in \mathcal{G}_j^o} \{ \muv^\top \xv - g_1(\xv) \} = \sup_{x_j \in \mathbb{R}_{-}}\{ x_j \mu_j - 0 \} = \infty.$$
Therefore, $g_1^\ast(\muv) = \infty$ if $\exists j,~\mu_j < 0$.

Now, we consider the second case, where $\forall i, \mu_i \geq 0$. We can easily show that $$\forall \xv \in \mathbb{R}^L,~\muv^\top \xv - g_1(\xv) \leq \sum_i \mu_i \max(\xv) - g_1(\xv).$$
Therefore $$g_1^\ast(\muv) \leq \sup_{\xv \in \mathbb{R}^L}\Big\{ (\sum_i \mu_i - C) \max(\xv) \Big\} = \indicator\Big( \sum_i \mu_i = C \Big).$$
Moreover, let $\mathcal{G}^1 = \{\xv \in \mathbb{R}^L:~ \xv = x \ev,~x \in \mathbb{R} \}$, where $\ev$ is a vector with every element being $1$. Then, we have
$$g_1^\ast(\muv) \geq \sup_{\xv \in \mathcal{G}^1}\{ \muv^\top \xv - g_1(\xv) \} = \sup_{x \in \mathbb{R}} \Big\{ (\sum_i \mu_i - C) x \Big\} = \indicator\Big(\sum_i \mu_i = C \Big).$$

Putting the above results together proves the claim.
\end{proof}

\subsection*{Appendix C.4: Proof of Lemma~\ref{lemma:ConjugateMaxFunc2}}

\begin{proof}
By definition, the conjugate is
\begin{eqnarray}
g_2^\ast(\mu) && = \sup_{x \in \mathbb{R}}\Big\{ \mu x - C\max(0, |x - y| - \epsilon) \Big\}. \nonumber \\
&& = - \inf_{x \in \mathbb{R}}\Big\{ - \mu x + C\max(0, |x-y| - \epsilon) \Big\}. \nonumber \\
&& = - \inf_{x \in \mathbb{R}; t \geq 0; t \geq |x-y|-\epsilon} \Big\{ - \mu x + C t \Big\} \nonumber \\
&& = - \sup_{\alpha,\beta \geq 0} \Big\{  \inf_{x,t \in \mathbb{R}} \big\{ - \mu x + C t - \alpha(t - |x-y| + \epsilon) - \beta t \big\} \Big\} \nonumber \\
&& = - \sup_{\alpha,\beta \geq 0} \Big\{ \inf_{x \in \mathbb{R}} \big\{ - \mu x + \alpha|x-y| \big\} + \inf_{ t \in \mathbb{R}}\big\{ C t - \alpha t - \beta t \big\} - \alpha \epsilon \Big\} \nonumber
\end{eqnarray}
For the second infimum, it is easy to show that $$\inf_{ t \in \mathbb{R}}\big\{ C t - \alpha t - \beta t \big\} = -\indicator( \alpha + \beta = C ).$$
For the first infimum, we can show that $$\inf_{x \in \mathbb{R}} \big\{ - \mu x + \alpha|x-y| \big\} = -\mu y + \inf_{x^\prime \in \mathbb{R}} \big\{ - \mu x^\prime + \alpha|x^\prime| \big\} = - \mu y - \indicator( |\mu| \leq \alpha ).$$

Thus, we have \begin{eqnarray}
g_2^\ast(\mu) && = - \sup_{\alpha,\beta \geq 0} \Big\{ -\mu y - \alpha \epsilon  - \indicator( |\mu| \leq \alpha )  - \indicator( \alpha + \beta = C ) \Big\} \nonumber \\
&& = -( -\mu y - \epsilon |\mu| - \indicator(|\mu| \leq C) ) \nonumber \\
&& = \mu y + \epsilon |\mu| + \indicator(|\mu| \leq C), \nonumber
\end{eqnarray}
where the second equality holds by setting $\alpha = |\mu|$, under the condition that $\epsilon$ is positive; the condition $|\mu| \leq C$ is induced from the conditions $\alpha + \beta = C$ and $\beta \geq 0$.
\end{proof}

\subsection*{Appendix C.5: Proof of Lemma~\ref{lemma:conjugateiLSVM}}\label{sec:AppendixiLSVM}

\begin{proof}
By definition, we have $g(Eq) \defEq \mathcal{R}_h^c\big(q(\Zv, \etav, \Wv)\big)  = \sum_n g_1(\ell_n^\Delta - Eq(n))$. Let $\muv_n = Eq(n)$. We have the conjugate
\begin{eqnarray}
g^\ast(\omegav) && = \sup_{\muv}\Big\{ \omegav^\top \muv - \sum_n g_1( \ell_n^\Delta - \muv_n) \Big\} \nonumber \\
                               && = \sum_n \sup_{\muv_n}\Big\{ \omegav_n^\top \muv_n - g_1(\ell_n^\Delta - \muv_n) \Big\} \nonumber \\
                               && = \sum_n \sup_{\nuv_n}\Big\{ \omegav_n^\top ( \ell_n^\Delta - \nuv_n ) - g_1(\nuv_n) \Big\} \nonumber \\
                               && = \sum_n \Big( \omegav_n^\top \ell_n^\Delta + g_1^\ast( - \omegav_n ) \Big). \nonumber
\end{eqnarray}
Thus, $$g^\ast(-\omegav) = \sum_n \big(- \omegav_n^\top \ell_n^\Delta + g_1^\ast( \omegav_n ) \big).$$
Using the results of Theorem~\ref{lemma:RegBayes} proves the claim.
\end{proof}

\subsection*{Appendix C.6: Proof of Lemma~\ref{lemma:conjugateMTiLSVM}}\label{sec:Appendix4}
\begin{proof}
Similar structure as the proof of Lemma~\ref{lemma:conjugateiLSVM}. In this case, the linear expectation operator is $E:~\mathcal{P}_{\textrm{prob}} \to \mathbb{R}^{\sum_m |\mathcal{I}_{\textrm{tr}}^m|}$ and the element of $Eq$ evaluated at the $n$th example for task $m$ is
\begin{eqnarray}
Eq(n, m) \defEq  y_{mn} \ep_{q(\Zv, \etav)}[ \Zv \etav_m ]^\top \xv_{mn} = \ep_{q(\Zv, \etav)}[ y_{mn} (\Zv \etav_m )^\top \xv_{mn}].
\end{eqnarray}
Then, let $g_0:~\mathbb{R} \to \mathbb{R}$ be a function defined in Lemma~\ref{proposition:ConjugateBinaryFunc}. We have $$g(Eq) \defEq \mathcal{R}_h^{MT}\Big(q(\Zv, \etav, \Wv)\Big) = \sum_{m, n \in \mathcal{I}^m_{\textrm{tr}}} g_0\Big( 1 - Eq(n, m) \Big).$$
Let $\muv = Eq$. By definition, the conjugate is
\begin{eqnarray}
g^\ast(\omegav)  && = \sup_{\muv}\Big\{ \omegav^\top \muv - \sum_{m, n \in \mathcal{I}^m_{\textrm{tr}}} g_0(1 - \mu_{mn}) \Big\} \nonumber \\
                     && = \sum_{m, n \in \mathcal{I}^m_{\textrm{tr}}} \sup_{\mu_{mn}}\Big\{ \omega_{mn} \mu_{mn} - g_0(1 - \mu_{mn}) \Big\} \nonumber \\
                               && = \sum_{m, n \in \mathcal{I}^m_{\textrm{tr}}} \sup_{\nu_n^m}\Big\{ \omega_{mn} (1 - \nu_{mn}) - g_0(\nu_{mn}) \Big\} \nonumber \\
                               && = \sum_{m, n \in \mathcal{I}^m_{\textrm{tr}}} \Big(\omega_{mn} + g_0^\ast( -\omega_{mn} ) \Big). \nonumber
\end{eqnarray}
Thus, $$g^\ast(-\omegav) = \sum_{m, n \in \mathcal{I}^m_{\textrm{tr}}} \Big(-\omega_{mn} + g_0^\ast( \omega_{mn} ) \Big).$$
By the results in Theorem~\ref{lemma:RegBayes} and Lemma~\ref{proposition:ConjugateBinaryFunc}, we can derive the conjugate of the problem~(\ref{eq:MTconstrained}).
\end{proof}

\subsection*{Appendix C.7: Proof of Lemma~\ref{lemma:conjugateMedLDA}} 

\begin{proof}
Similar structure as the proof of Lemma~\ref{lemma:conjugateiLSVM}. In this case, the linear expectation operator is $E:~\mathcal{P}_{\textrm{prob}} \to \mathbb{R}^{N}$ and the elements of $Eq$ evaluated at the $n$th example is 
\begin{eqnarray}
\mu_n = \ep_{q(\{\theta_n, z_{nm}, \etav\} | \Theta)}[ \etav^\top \bar{z}_n ].
\end{eqnarray}
Then, using the $g_2$ function defined in Lemma~\ref{lemma:ConjugateMaxFunc2}, we have $$g(Eq) \defEq \mathcal{R}_\epsilon (q(\{\theta_n, z_{nm}, \etav\} | \Theta)) = \sum_n g_2\Big(\mu_n; y_n, \epsilon \Big).$$ Therefore $g^\ast(\omegav) = \sum_{n} g_2^\ast(\omega_n; y_n, \epsilon )$ and $g^\ast(-\omegav) = \sum_{n} g_2^\ast(-\omega_n; y_n, \epsilon ).$ By the results in Theorem~\ref{lemma:RegBayes} and Lemma~\ref{proposition:ConjugateBinaryFunc}, we can derive the conjugate and the optimum solution of $\hat{q}$. The optimum solution of $\Theta$ is due to Lemma~\ref{lemma:MLEDuality}. Note that the constraints are not directly dependent on $\Theta$.
\end{proof}

\section*{Appendix D: Inference Algorithms for Infinite Latent SVMs}

\subsection*{Appendix D.1: Inference for MT-iLSVM}\label{sec:Appendix2}

In this section, we provide the derivation of the inference algorithm for MT-iLSVM, which is outlined in Algorithm~\ref{alg:algMT-iLSVM} and detailed below.

For MT-iLSVM, the model $\model$ consists of all the latent variables $(\nuv, \Wv, \Zv, \etav)$. Let $L_{mn}(q) \defEq \ep_q[ \log p(\xv_{mn}|\Zv, \wv_{mn}, \lambda_{mn}^2) ]$ be the expected data likelihood. Then, under the truncated mean-field assumption~(\ref{eq:MF-MTiLSVM}), we have
\setlength\arraycolsep{1pt}\begin{eqnarray}
L_{mn}(q) &&= -\frac{ \xv_{mn}^\top \xv_{mn} - 2\xv_{mn}^\top \ep_q[\Zv \wv_{mn}] + \ep_q[\wv_{mn}^\top \Uv \wv_{mn}] }{2 \lambda_{mn}^2} - \frac{D  \log( 2\pi \lambda_{mn}^2) }{2}, \nonumber
\end{eqnarray}
where $\xv_{mn}^\top \ep_q[\Zv \wv_{mn}] = \sum_{k} \xv_{mn}^\top \psiv_{.k}$; $\psiv_{.k} \defEq (\psi_{1k} \cdots \psi_{Dk})^\top$ is the $k$th column of $\psiv = \ep_q[\Zv]$;
\begin{eqnarray}
\ep_q[ \wv_{mn}^\top \Uv \wv_{mn}] = 2 \sum_{ j < k} \phi_{mn}^j\phi_{mn}^k \Uv_{jk} + \sum_k \Uv_{kk}( K \sigma_{mn}^2 + \Phi_{mn}^\top \Phi_{mn} ); \nonumber
\end{eqnarray}
and $\Uv \defEq \ep_q[\Zv^\top \Zv]$ is a $K\times K$ matrix, whose element is\begin{eqnarray}
\Uv_{ij} = \left\{\begin{array}{ll}
            \sum_d \psi_{di}, & \textrm{if}~ i=j \\
            \sum_d \psi_{di}\psi_{dj}, & \textrm{otherwise}.
            \end{array}\right. \nonumber
\end{eqnarray}
For the KL-divergence term, we have $ \KL(q(\model) \Vert \pi(\model) ) = \KL(q(\nuv) \Vert \pi(\nuv)) + \KL( q(\Wv) \Vert \pi(\Wv) ) + \ep_{q(\nuv)}[\KL(q(\Zv) \Vert \pi(\Zv | \nuv))] + \KL( q(\etav) \Vert \pi(\etav) ) $, where the individual terms are
\setlength\arraycolsep{1pt}\begin{eqnarray}
\KL(q(\nuv) \Vert \pi(\nuv)) = && \sum_{k=1}^K \Big((\gamma_{k1}-\alpha)(\varphi(\gamma_{k1})-\varphi(\gamma_{k1}+\gamma_{k2})) + (\gamma_{k2}-1) (\varphi(\gamma_{k2})-\varphi(\gamma_{k1}+\gamma_{k2})) \nonumber \\
&& - \log \frac{\Gamma(\gamma_{k1})\Gamma(\gamma_{k2})}{\Gamma(\gamma_{k1}+\gamma_{k2})}\Big) - K \log \alpha, \nonumber\\
\ep_{q(\nuv)}[ \KL(q(\Zv) \Vert \pi(\Zv | \nuv)) ] = && \sum_{dk} \Big( - \psi_{dk} \sum_{j=1}^k \ep_q[\log \nu_j] - (1-\psi_{dk})\ep_q[ \log(1 - \prod_{j=1}^k \nu_j) ] \nonumber \\
&& + \psi_{dk}\log \psi_{dk} + (1-\psi_{dk})\log(1-\psi_{dk}) \Big) \nonumber \\
\KL(q(\Wv) \Vert \pi(\Wv)) = && \sum_{mn} \Big(\frac{K\sigma_{mn}^2 + \Phi_{mn}^\top\Phi_{mn}}{2\sigma_{m0}^2} - \frac{K(1 + \log \frac{\sigma_{mn}^2}{\sigma_{m0}^2})}{2}\Big). \nonumber
\end{eqnarray}
where $\varphi(\cdot)$ is the digamma function and $\ep_q[ \log v_j ] = \varphi(\gamma_{j1}) - \varphi(\gamma_{j1} + \gamma_{j2})$. For $\KL(q(\etav)\Vert \pi(\etav))$, we do not need to write it explicitly, as we shall see. Finally, the effective discriminant function is $$f_m(\xv_{mn}; q(\Zv, \etav)) = \ep_q[\etav_m]^\top \psiv^\top \xv_{mn} = \sum_{k=1}^K \ep_q[\eta_{mk}] \psiv_{.k}^\top \xv_{mn}.$$
All the above terms can be easily computed, except the term $\ep_q[ \log(1 - \prod_{j=1}^k \nu_j) ]$. Here, we adopt the multivariate lower bound~\citep{Doshi-Velez:09}
\begin{eqnarray}
\ep_q[ \log(1 - \prod_{j=1}^k \nu_j) ] && \geq \sum_{m=1}^k  q_{km} \varphi(\gamma_{m2}) + \sum_{m=1}^{k-1}( \sum_{n=m+1}^k  q_{kn}) \varphi(\gamma_{m1}) \nonumber \\
                                        && - \sum_{m=1}^k(\sum_{n=m}^k  q_{kn})\varphi(\gamma_{m1}+\gamma_{m2}) + \mathcal{H}(q_{k.}), \nonumber
\end{eqnarray}
where the variational parameters $q_{k.}=(q_{k1} \cdots q_{kk})^\top$ belong to the $k$-simplex, and $\mathcal{H}(q_{k.})$ is the entropy of $q_{k.}$. The tightest lower bound is achieved by setting $q_{k.}$ to be the optimum value
\begin{eqnarray}\label{eq:varQ}
q_{km} = \frac{1}{Z_k} \exp\Big( \varphi(\gamma_{m2}) + \sum_{n=1}^{m-1} \varphi(\gamma_{n1}) - \sum_{n=1}^m \varphi(\gamma_{n1}+\gamma_{n2}) \Big),
\end{eqnarray}
where $Z_k$ is a normalization factor to make $q_{k.}$ be a distribution. We denote the tightest lower bound by $\mathcal{L}_k^\nu$. Replacing the term $\ep_q[ \log(1 - \! \prod_{j=1}^k \nu_j) ]$ with its lower bound $\mathcal{L}_k^\nu$, we can have an upper bound of $\KL(q(\model)\Vert \pi(\model))$ and we denote this upper bound by $\mathcal{L}(q)$.

\begin{algorithm}[tb]
   \caption{\footnotesize Inference Algorithm of MT-iLSVM} \label{alg:algMT-iLSVM}
\begin{algorithmic}[1]
   \STATE {\bfseries Input:} data $\mathcal{D} = \lbrace (\xv_{mn}, y_{mn}) \rbrace_{m,n\in\mathcal{I}_{\textrm{tr}}^m} \cup \lbrace \xv_{mn} \rbrace_{m, n\in\mathcal{I}_{\textrm{tst}}^m}$, constants $\alpha$ and $C$
   \STATE {\bfseries Output:} distributions $q(\nuv)$, $q(\Zv)$, $q(\Wv)$, $q(\etav)$ and hyper-parameters $\sigma_{m0}^2$ and $\lambda_{mn}^2$
   \STATE Initialize $\gamma_{k1}=\alpha$, $\gamma_{k2}=1$, $\psi_{dk}=0.5+\epsilon$, where $\epsilon \sim \mathcal{N}(0, 0.001)$, $\Phi_{mn} = 0$, $\sigma_{mn}^2=\sigma_{m0}^2=1$, $\muv_m = 0$, $\lambda_{mn}^2$ is computed from $\mathcal{D}$.
   \REPEAT
        \REPEAT
            \STATE update $(\gamma_{k1}, \gamma_{k2})$ using Eq. (\ref{eq:updateNu}), $\forall 1 \leq k \leq K$;
            \STATE update $\phi_{mn}^k$ and $\sigma_{mn}^2$ using Eq. (\ref{eq:updateW}), $\forall m, \forall n, \forall 1 \leq k \leq K$;
            \STATE update $\psi_{dk}$ using Eq. (\ref{eq:updateZ}), $\forall 1 \leq d \leq D, \forall 1 \leq k \leq K$;
        \UNTIL relative change of $L$ is less than $\tau$ (e.g., $1e^{-3}$) or iteration number is $T$ (e.g., 10)
       \FOR{$m=1$ {\bfseries to} $M$}
        \STATE solve the dual problem~(\ref{eq:binarySVM}) using a binary SVM learner.
       \ENDFOR
       \STATE update the hyper-parameters $\sigma_{m0}^2$ using Eq.~(\ref{eq:MT-Hyper1}) and $\lambda_{mn}^2$ using Eq.~(\ref{eq:MT-Hyper2}). ({\it Optional})
   \UNTIL relative change of $L$ is less than $\tau^\prime$ (e.g., $1e^{-4}$) or iteration number is $T^\prime$ (e.g., 20)
\end{algorithmic}
\end{algorithm}

With the above terms and the upper bound $\mathcal{L}(q)$, we can implement the general procedure outlined in Algorithm~\ref{alg:coord_descent_regress} to solve the MT-iLSVM problem.
Specifically, the inference procedure iteratively solves the following steps, as summarized in Algorithm~\ref{alg:algMT-iLSVM}:

{\bf Infer $q(\nuv)$, $q(\Zv)$ and $q(\Wv)$:} For $q(\Wv)$, since both the prior $\pi(\Wv)$ and $q(\Wv)$ are Gaussian, we can easily derive the update rules, similar as in Gaussian mixture models
\begin{eqnarray}\label{eq:updateW}
\phi_{mn}^k && = \frac{\sum_d x_{mn}^d \psi_{dk} - \sum_{j \neq k} \phi_{mn}^j \Uv_{kj} }{\lambda_{mn}^2} \Big(\frac{1}{\sigma_{m0}^2} + \frac{\sum_d \psi_{dk}}{\lambda_{mn}^2} \Big)^{-1} \\
\sigma_{mn}^2 &&= \Big(\frac{1}{\sigma_{m0}^2} + \frac{1}{ K} \sum_k \frac{\Uv_{kk}}{\lambda_{mn}^2} \Big)^{-1} \nonumber
\end{eqnarray}
For $q(\nuv)$, we have the update rules similar as in~\citep{Doshi-Velez:09}, that is,
\begin{eqnarray} \label{eq:updateNu}
\gamma_{k1} = && \alpha + \sum_{m=k}^K\sum_{d=1}^D \psi_{dm} + \sum_{m=k+1}^K(D - \sum_{d=1}^D\psi_{dm})(\sum_{i=k+1}^m q_{mi}) \\
\gamma_{k2} = && 1 + \sum_{m=k}^K(D - \sum_{d=1}^D \psi_{dm}) q_{mk}. \nonumber
\end{eqnarray}
For $q(\Zv)$, we have the mean-field update equation as
\begin{eqnarray}\label{eq:updateZ}
\psi_{dk} = \frac{1}{1 + e^{-\vartheta_{dk}}},
\end{eqnarray}
where
\begin{eqnarray}
\vartheta_{dk} = && \sum_{j=1}^k \ep_q[ \log v_j ] - \mathcal{L}_k^\nu - \sum_{mn} \frac{1}{2\lambda_{mn}^2}\Big( (K \sigma_{mn}^2 + (\phi_{mn}^k)^2)  \nonumber \\
&&  - 2 x_{mn}^d \phi_{mn}^k + 2 \sum_{j\neq k} \phi_{mn}^j\phi_{mn}^k \psi_{dj} \Big) + \sum_{m, n\in \mathcal{I}_{\textrm{tr}}^m} y_{mn} \ep_q[\eta_{mk}] x_{mn}^d. \nonumber
\end{eqnarray}
{\bf Infer $q(\etav)$ and solve for $\omegav$:} By the convex duality theory, we have the solution 
\begin{eqnarray}
q(\etav) && \propto \pi(\etav) \exp\Big\lbrace \sum_{ m,n \in \mathcal{I}^m_{\textrm{tr}}} y_{mn} \omega_{mn} \etav_{m}^\top \psiv^\top \xv_{mn} \Big\rbrace \nonumber \\
&& = \prod_{m=1}^M \pi(\etav_m) \exp\Big\lbrace \etav_{m}^\top \big( \sum_{n \in \mathcal{I}^m_{\textrm{tr}}} y_{mn} \omega_{mn} \psiv^\top \xv_{mn} \big) \Big\rbrace. \nonumber
\end{eqnarray}
Therefore, we can see that although we did not assume $q(\etav)$ is factorized, we can get the induced factorization form $q(\etav) = \prod_m q(\etav_m)$, where
\begin{eqnarray}
q(\eta_m) \propto \pi(\etav_m) \exp\Big\lbrace \etav_{m}^\top \big( \sum_{n \in \mathcal{I}^m_{\textrm{tr}}}  y_{mn} \omega_{mn} \psiv^\top \xv_{mn} \big) \Big\rbrace. \nonumber
\end{eqnarray}
Here, we assume $\pi(\etav_m)$ is standard normal. Then, we have $q(\etav_m) = \mathcal{N}(\etav_m|\muv_m , I)$, where
\begin{eqnarray}
\muv_m = \sum_{n \in \mathcal{I}^m_{\textrm{tr}}} y_{mn} \omega_{mn} \psiv^\top \xv_{mn}. \nonumber
\end{eqnarray}
The optimum dual parameters can be obtained by solving the following $M$ independent dual problems
\begin{eqnarray}\label{eq:binarySVM}
\sup_{ \omegav_m }~ -\frac{1}{2} \muv_m^\top \muv_m + \sum_{n \in \mathcal{I}^m_{\textrm{tr}}} \omega_{mn} ~~~~~ \st.:~ 0 \leq \omega_{mn} \leq C, \forall n \in \mathcal{I}^m_{\textrm{tr}},
\end{eqnarray}
which (and its primal form) can be efficiently solved with a binary SVM solver, such as SVM-light.

As we have stated, the hyperparameters $\sigma_0^2$ and $\lambda_{mn}^2$ can be set a priori or estimated from the data. The empirical estimation can be easily done with closed form solutions by optimizing the RegBayes objective with all the variational terms fixed. For MT-iLSVM, we have
\begin{eqnarray}
\sigma_{m0}^2 &=& \frac{\sum_{n=1}^{N_m} (K \sigma_{mn}^2 + \Phi_{mn}^\top \Phi_{mn})}{K N_m} \label{eq:MT-Hyper1} \\
\lambda_{mn}^2 &=& \frac{\xv_{mn}^\top \xv_{mn} - 2\xv_{mn}^\top \ep_q[\Zv \wv_{mn}] + \ep_q[\wv_{mn}^\top \Uv \wv_{mn}]}{D}. \label{eq:MT-Hyper2}
\end{eqnarray}

\subsection*{Appendix D.2: Inference for Infinite Latent SVM}\label{sec:Appendix3}

In this section, we develop the inference algorithm for iLSVM based on the stick-breaking construction of the IBP prior. The algorithm is outlined in Algorithm~\ref{alg:algiLSVM}.

Similar as in the inference for MT-iLSVM, we make the additional constraint about the feasible distribution $$q(\nuv, \Wv, \Zv, \etav) = q(\etav) q(\Wv|\Phi, \Sigma) \prod_n \Big( \prod_{k=1}^K q(z_{nk}|\psi_{nk}) \Big) \prod_{k=1}^K q(\nu_k | \gammav_k),$$where $K$ is the truncation level; $q(\Wv|\Phi, \Sigma)= \prod_k \mathcal{N}(\Wv_{.k}|\Phi_{.k}, \sigma_k^2 I)$; $q(z_{nk}|\phi_{nk}) = \Ber(\phi_{nk})$; and $q(\nu_k | \gammav_k) = \B(\gamma_{k1}, \gamma_{k2})$. Then, we solve the unconstrained problem using convex duality with dual parameters being $\omegav$. Let $L_n(q) \defEq \ep_q[ \log p(\xv_n | \zv_n, \Wv) ]$. We have
\begin{eqnarray}
L_n(q) &&= -\frac{ \xv_n^\top \xv_n - 2\xv_n^\top \Phi \ep_q[\zv_n]^\top + \ep_q[ \zv_n \Av \zv_n^\top ] }{2 \sigma_{n0}^2} - \frac{D  \log( 2\pi \sigma_{n0}^2) }{2},
\end{eqnarray}
where $\Av \defEq \ep_q[\Wv^\top \Wv]$ is a $K\times K$ matrix; $\xv_n^\top \Phi \ep_q[\zv_n]^\top = 2 \sum_k \psi_{nk} (\xv_n^\top \Phi_{.k})$; and $$\ep_q[\zv_n \Av \zv_n^\top] = 2 \sum_{ j < k} \psi_{nj}\psi_{nk} \Av_{jk} + \sum_k \psi_{nk}( D \sigma_k^2 + \Av_{kk} ).$$
The effective discriminant function is $f(y, \xv_n) = \sum_k \ep_q[\eta_y^k] \psi_{nk}$.
Again, for computational tractability, we need the lower bound $\mathcal{L}_k^\nu$ of the term $\ep_q[ \log (1 - \prod_{j=1}^k v_j) ]$. Using this lower bound, we can get an upper bound of the KL-divergence term. Then, the inference procedure iteratively solves the following steps:

{\bf Infer $q(\nuv)$, $q(\Zv)$ and $q(\Wv)$:} For $q(\Wv)$, we have the update rules
\begin{eqnarray}\label{eq:updateW2}
\Phi_{.k} && = \sum_n \frac{\psi_{nk}}{\sigma_{n0}^2}\Big(\xv_n - \sum_{j\neq k} \psi_{nj} \Phi_{.j}\Big) \Big(1 + \sum_n \frac{\psi_{nk}}{\sigma_{n0}^2} \Big)^{-1} \\
\sigma_k^2 &&= \Big(1 + \sum_n \frac{\psi_{nk}}{\sigma_{n0}^2} \Big)^{-1}. \nonumber
\end{eqnarray}
For $q(\nuv)$, we have the update rules similar as in~\citep{Doshi-Velez:09}, that is,
\begin{eqnarray} \label{eq:updateNu2}
\gamma_{k1} = && \alpha + \sum_{m=k}^K\sum_{n=1}^N \psi_{nm} + \sum_{m=k+1}^K(N - \sum_{n=1}^N\psi_{nm})(\sum_{i=k+1}^m q_{mi}) \\
\gamma_{k2} = && 1 + \sum_{m=k}^K(N - \sum_{n=1}^N \psi_{nm}) q_{mk}, \nonumber
\end{eqnarray}
where $q_{.k}$ is computed in the same way as in Eq. (\ref{eq:varQ}).
For $q(\Zv)$, the mean-field update equation for $\psi$ is\\[-1.2cm]

\begin{eqnarray}\label{eq:updateZ2}
\psi_{nk} = \frac{1}{1 + e^{-\vartheta_{nk}}},
\end{eqnarray}\\[-.9cm]

where\\[-1.2cm]

\begin{eqnarray}
\vartheta_{nk} && = \sum_{j=1}^k \ep_q[ \log v_j ] - \mathcal{L}_k^\nu(q) - \frac{1}{2\sigma_{n0}^2}(D\sigma_k^2 + \Phi_{.k}^\top \Phi_{.k}) \nonumber \\
&&  + \frac{1}{\sigma_{n0}^2} \Phi_{.k}^\top \Big( \xv_n - \sum_{j\neq k} \psi_{nj} \Phi_{.j} \Big) + \sum_y \omega_n^y \ep_q[\eta_{y_n}^k - \eta_y^k]. \nonumber
\end{eqnarray}\\[-.9cm]

For testing data, $\vartheta_{nk}$ does not have the last term because of the absence of large-margin constraints.

{\bf Infer $q(\etav)$ and solve for $\omegav$:} By the convex duality theory, we have
\begin{eqnarray}
q(\etav) && \propto \pi(\etav) \exp\Big\lbrace \etav^\top ( \sum_{n \in \mathcal{I}_{\textrm{tr}}} \sum_y \omega_n^y \ep_q[ \gv(y_n, \xv_n, \zv_n) - \gv(y, \xv_n, \zv_n) ]) \Big\rbrace. \nonumber
\end{eqnarray}
For the standard normal prior $\pi(\etav)$, we have that $q(\etav)$ is also normal, with mean $$\muv = \sum_{n \in \iTrain} \sum_y \omega_d^y \ep_q[ \gv(y_n, \xv_n, \zv_n) - \gv(y, \xv_n, \zv_n) ]$$ and identity covariance matrix. The dual problem is
\begin{eqnarray}\label{eq:multiwaySVM}
\sup_{ \omegav }~ -\frac{1}{2} \muv^\top \muv + \sum_{n \in \iTrain} \sum_y \omega_n^y ~~~~~ \st.:~ \omega_n^y \geq 0, ~ \sum_{y} \omega_n^y = C, \forall n \in \iTrain,
\end{eqnarray}
which (and its primal form) can be efficiently solved with a multi-class SVM solver.

\begin{algorithm}[tb]
   \caption{\footnotesize Inference Algorithm of iLSVM} \label{alg:algiLSVM}
\begin{algorithmic}[1]
   \STATE {\bfseries Input:} data $\mathcal{D} = \lbrace (\xv_{n}, y_{n}) \rbrace_{n\in \iTrain} \cup \lbrace \xv_{n} \rbrace_{n \in \iTest}$, constants $\alpha$ and $C$
   \STATE {\bfseries Output:} distributions $q(\nuv)$, $q(\Zv)$, $q(\Wv)$, $q(\etav)$ and hyper-parameters $\sigma_0^2$ and $\sigma_{n0}^2$
   \STATE Initialize $\gamma_{k1}=\alpha$, $\gamma_{k2}=1$, $\psi_{nk}=0.5+\epsilon$, where $\epsilon \sim \mathcal{N}(0, 0.001)$, $\Phi_{.k} = 0$, $\sigma_k^2 = \sigma_0^2 = 1$, $\muv = 0$, $\sigma_{n0}^2$ is computed from $\mathcal{D}$.
   \REPEAT
        \REPEAT
            \STATE update $(\gamma_{k1}, \gamma_{k2})$ using Eq. (\ref{eq:updateNu2}), $\forall 1 \leq k \leq K$;
            \STATE update $\Phi_{.k}$ and $\sigma_k^2$ using Eq. (\ref{eq:updateW2}), $\forall 1 \leq k \leq K$;
            \STATE update $\psi_{nk}$ using Eq. (\ref{eq:updateZ2}), $\forall n \in \iTrain, \forall 1 \leq k \leq K$;
            \STATE update $\psi_{nk}$ using Eq. (\ref{eq:updateZ2}), but $\vartheta_{nk}$ doesn't have the last term, $\forall n \in \iTest, \forall 1 \leq k \leq K$;
        \UNTIL relative change of $L$ is less than $\tau$ (e.g., $1e^{-3}$) or iteration number is $T$ (e.g., 10)
       \STATE solve the dual problem~(\ref{eq:multiwaySVM}) (or its primal form) using a multi-class SVM learner.
       \STATE update the hyper-parameters $\sigma_0^2$ using Eq. (\ref{eq:C-Hyper1}) and $\sigma_{n0}^2$ using Eq. (\ref{eq:C-Hyper2}). ({\it Optional})
   \UNTIL relative change of $L$ is less than $\tau^\prime$ (e.g., $1e^{-4}$) or iteration number is $T^\prime$ (e.g., 20)
\end{algorithmic}
\end{algorithm}

Similar as in MT-iLSVM, the hyperparameters $\sigma_0^2$ and $\sigma_{n0}^2$ can be set a priori or estimated from the data. The empirical estimation can be easily done with closed form solutions. For iLSVM, we have
\begin{eqnarray}
\sigma_{0}^2 &=& \frac{\sum_{k=1}^{K} (D \sigma_{k}^2 + \Phi_{.k}^\top \Phi_{k})}{K D} \label{eq:C-Hyper1} \\
\sigma_{n0}^2 &=& \frac{ \xv_n^\top \xv_n - 2\xv_n^\top \Phi \ep_p[\zv_n]^\top + \ep_q[ \zv_n \Av \zv_n^\top ] }{D}. \label{eq:C-Hyper2}
\end{eqnarray}

\bibliographystyle{plain}
\bibliography{med}

\end{document}